\documentclass[twoside]{article}

\usepackage[accepted]{aistats2025}

\usepackage[round]{natbib}

\usepackage{hyperref}       %
\usepackage{url}            %
\usepackage{booktabs}       %
\usepackage{amsfonts}       %
\usepackage{nicefrac}       %
\usepackage{microtype}      %
\usepackage[dvipsnames]{xcolor}
\usepackage{subcaption}
\usepackage{tikz}
\usetikzlibrary{shapes.geometric, arrows}

\tikzstyle{startstop} = [rectangle, rounded corners, minimum width=3.5cm, minimum height=1cm, text width=3cm, text centered, draw=black, fill=red!30]
\tikzstyle{process} = [rectangle, rounded corners, minimum height=1cm, text width=3.5cm, text centered, draw=black, fill=green!30]
\tikzstyle{decision} = [trapezium, trapezium stretches=true, shape border rotate=180, text width=2cm, minimum height=1.1cm, text centered, draw=black, fill=yellow!30]
\tikzstyle{arrow} = [thick,->,>=stealth]

\definecolor{myblue}{rgb}{0.39215686274509803, 0.5843137254901961, 0.9294117647058824}
\definecolor{myorange}{rgb}{1.0, 0.6470588235294118, 0.0}
\definecolor{mycoral}{rgb}{0.9411764705882353, 0.5019607843137255, 0.5019607843137255}
\definecolor{mylightgreen}{rgb}{0.596078431372549, 0.984313725490196, 0.596078431372549}
\definecolor{mydarkgreen}{rgb}{0.1803921568627451, 0.5450980392156862, 0.3411764705882353}
\definecolor{mylightorchid}{rgb}{0.8470588235294118, 0.7490196078431373, 0.8470588235294118}
\definecolor{mydarkorchid}{rgb}{0.7294117647058823, 0.3333333333333333, 0.8274509803921568}
\definecolor{mythistle}{rgb}{0.8470588235294118, 0.7490196078431373, 0.8470588235294118}
\definecolor{mymediumorchid}{rgb}{0.7294117647058823, 0.3333333333333333, 0.8274509803921568}
\definecolor{mypalegreen}{rgb}{0.596078431372549, 0.984313725490196, 0.596078431372549}
\definecolor{myseagreen}{rgb}{0.1803921568627451, 0.5450980392156862, 0.3411764705882353}

\usepackage[noend]{algorithmic}

\usepackage{algorithm}

\usepackage{amsmath}
\usepackage{amssymb}
\usepackage{mathtools}
\usepackage{amsthm}
\usepackage[capitalize,noabbrev]{cleveref}
\usepackage{wasysym}
\usepackage{physics}
\usepackage{siunitx}
\usepackage{afterpage}

\theoremstyle{plain}
\newtheorem{theorem}{Theorem}[section]
\newtheorem{proposition}[theorem]{Proposition}

\theoremstyle{definition}
\newtheorem{definition}[theorem]{Definition}
\newtheorem{assumption}[theorem]{Assumption}
\theoremstyle{remark}
\newtheorem{remark}[theorem]{Remark}
\newtheorem{example}[theorem]{Example}

\DeclareMathOperator*{\cov}{Cov}
\DeclareSIUnit{\nothing}{\relax}

\newcommand{\ud}{\mathrm{d}}

\usepackage{enumitem}   
\usepackage{wrapfig}
\usepackage{float}
\usepackage{url}
\usepackage{upgreek}
\usepackage{tikz-cd}

\newcommand{\R}{\mathbb R}
\newcommand{\X}{\R^d}
\newcommand{\E}{\mathbb E}

\newcommand{\cN}{\mathcal N}

\newcommand{\cX}{\mathcal X}

\newcommand{\ps}[1]{\langle #1 \rangle}
\newcommand{\cF}{\mathcal{F}}
\newcommand{\cG}{\mathcal{G}}

\newcommand{\cP}{\mathcal{P}}
\newcommand{\KL}{\mathop{\mathrm{KL}}\nolimits}
\newcommand{\FD}{\mathop{\mathrm{FD}}\nolimits}
\newcommand{\TV}{\mathop{\mathrm{TV}}\nolimits}

\DeclareMathOperator*{\argmin}{argmin}

\newcommand{\pdata}{p_{\text{data}}}
\newcommand{\pref}{p_{\text{ref}}}

\renewcommand{\top}{\intercal}

\usepackage{nicefrac}

\begin{document}

\twocolumn[

\aistatstitle{Implicit Diffusion: Efficient optimization through stochastic sampling}

\aistatsauthor{Pierre Marion \addtocounter{footnote}{-1}\footnotemark \footnote{}\And Anna Korba}
\aistatsaddress{Institute of Mathematics, EPFL\\ Lausanne, Switzerland \And ENSAE CREST\\ IP Paris, France}

\aistatsauthor{Peter Bartlett, Mathieu Blondel, Valentin De Bortoli,}
\aistatsauthor{Arnaud Doucet, Felipe Llinares-Lopez , Courtney Paquette,}
\aistatsauthor{Quentin Berthet\addtocounter{footnote}{-5}\footnote{}}
\aistatsaddress{Google DeepMind}

\runningauthor{Marion, Korba, Bartlett, Blondel, De Bortoli, Doucet, Llinares-Lopez, Paquette, Berthet}]

\footnotetext{Corresponding authors. Address correspondance at \texttt{pierre.marion@epfl.ch} and  \texttt{qberthet@google.com}.}
\addtocounter{footnote}{1}
\footnotetext{Work mostly done while a student researcher at Google DeepMind.}

\begin{abstract}

Sampling and automatic differentiation are both ubiquitous in modern machine learning. At its intersection, differentiating through a sampling operation, with respect to the parameters of the sampling process, is a problem that is both challenging and broadly applicable. We introduce a general framework and a new algorithm for first-order optimization of parameterized stochastic diffusions, performing jointly, in a single loop, optimization and sampling steps. This approach is inspired by recent advances in bilevel optimization and automatic implicit differentiation, leveraging the point of view of sampling as optimization over the space of probability distributions. We provide theoretical and experimental results showcasing the performance of our method.
\end{abstract}

\section{INTRODUCTION}
\label{sec:intro}

Sampling from a target distribution is a ubiquitous task at the heart of various methods in machine learning, optimization, and statistics. Increasingly, sampling algorithms rely on iteratively applying large-scale parameterized functions (e.g. neural networks), as in  denoising diffusion models \citep{ho2020denoising}. 
\begin{figure}[h]
\begin{center}
\includegraphics[width=0.75\columnwidth]{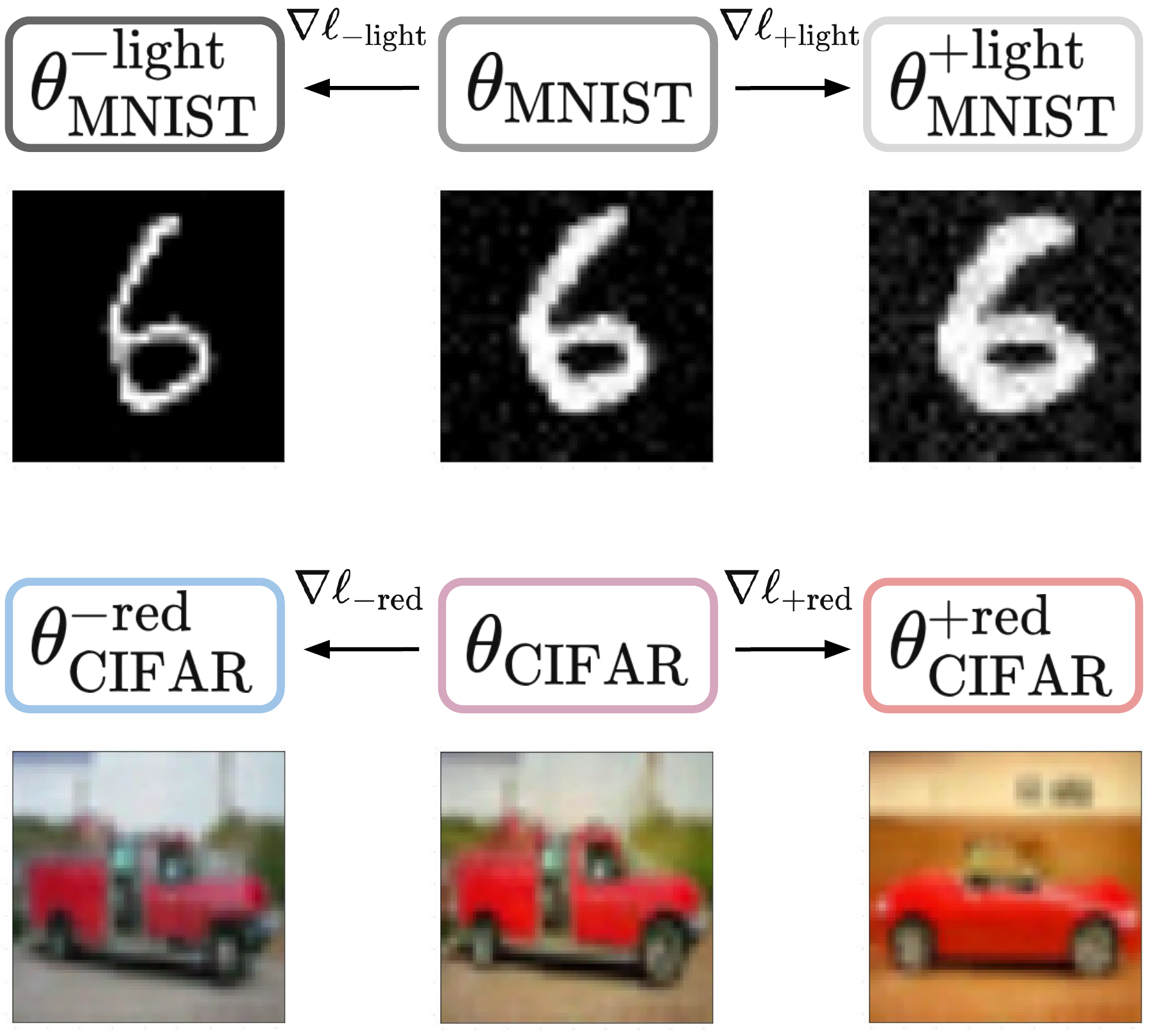}
\caption{Optimizing through sampling with \textbf{Implicit Diffusion} to finetune denoising diffusion models. Reward is brightness for MNIST and red for CIFAR-10.
\vspace{-0.8cm}
\label{fig:intro-examples}}
\end{center}
\end{figure}

This iterative {\em sampling operation} implicitly maps a parameter~$\theta \in \R^p$ to some distribution $\pi^\star(\theta)$ in $\cP$. %

In this work, we focus on optimization problems over these implicitly parameterized distributions. For a space of distributions $\cP$ (e.g. over $\R^d$), and a function $\cF: \cP \to \R$, our main problem is
\begin{equation*}
    \min_{\theta \in \R^p} \ell(\theta):= \min_{\theta \in \R^p}\cF(\pi^\star(\theta))
\end{equation*}
This setting encompasses for instance learning parameterized Langevin diffusions, contrastive learning of energy-based models \citep{gutmann2012noise} or finetuning denoising diffusion models \citep[e.g.,][]{dvijotham2023algorithms,clark2024directly}, as illustrated by Figure \ref{fig:intro-examples}.  Applying first-order optimizers to this problem raises the challenge of computing gradients of functions of the target distribution with respect to the parameter: we have to {\em differentiate through a sampling operation}, where the link between $\theta$ and $\pi^\star(\theta)$ can be {\em implicit} (see, e.g., Figure~\ref{fig:algo}). %

To this aim, we propose to exploit the perspective of {\em sampling as optimization}, where the task of sampling is seen as an optimization problem over the space of probability distributions $\cP$ \citep[see][and references therein]{korba2022sampling}. Typically, approximating a target probability distribution $\pi$ can be cast as the minimization of a dissimilarity functional between probability distributions w.r.t.~$\pi$, that only vanishes at the target. For instance, Langevin diffusion dynamics follow a gradient flow of a Kullback-Leibler (KL) objective w.r.t.~the Wasserstein-2 distance~\citep{jordan1998variational}.

This allows us to draw a link between optimization through stochastic sampling and \textit{bilevel optimization}, which often involves computing derivatives of the solution of a parameterized optimization problem obtained after iterative steps of an algorithm. 
Bilevel optimization is an active area of research with many relevant applications in machine learning, such as hyperparameter optimization \citep{franceschi2018bilevel} or meta-learning \citep{liu2018darts}. In particular, there is a significant effort in the literature for developing tractable and provably efficient algorithms in a large-scale setting \citep{pedregosa2016hyperparameter,chen2021closing,arbel2022amortized,blondel2022efficient,dagreou2022framework}--see Appendix~\ref{apx:additional-related-work} for additional related work. This literature focuses mostly on problems with finite-dimensional variables, in contrast with our work where the solution of the inner problem is a distribution in $\cP$.

These motivating similarities, while useful, are not limiting. We also consider settings where the sampling iterations are not readily interpretable as an optimization algorithm. %
Denoising diffusion cannot directly be formalized as descent dynamics of an objective functional over $\cP$, but its output is determined by a parameter $\theta$ (i.e. weights of the score matching neural networks). 

\paragraph{Main Contributions.} In this work, we introduce the algorithm of \textbf{Implicit Diffusion}, an effective and principled technique for optimizing through a sampling operation. %
More precisely,

\begin{itemize}[topsep=0pt,itemsep=2pt,parsep=2pt,leftmargin=10pt]
    \item[-] We present a general framework describing parameterized sampling algorithms, and introduce Implicit Diffusion optimization, a \textbf{single-loop} optimization algorithm to optimize through sampling.
    \item[-] %
    We provide \textbf{theoretical guarantees} in the continuous and discrete time settings in Section \ref{sec:theory}.
    \item[-] We showcase in Section \ref{sec:experiments} its performance and efficiency in \textbf{experimental settings}. Applications include finetuning denoising diffusions and training energy-based models.
\end{itemize}

To allow for reproducibility, we provide an implementation\footnote{\url{https://github.com/google-deepmind/implicit_diffusion}} of our algorithm in JAX \citep{jax2018github}.

\paragraph{Notation.} For a set $\cX$ (such as $\R^d$), we write $\cP$ for the set of probability distributions on $\cX$, omitting reference to $\cX$. For $f$ a differentiable function on~$\R^d$, we denote by $\nabla f$ its gradient function. %
If $f$ is a differentiable function of $k$ variables, we let $\nabla_i f$ denote its gradient w.r.t.~its $i$-th variable.

\section{PROBLEM PRESENTATION}
\label{sec:problem-presentation}

\subsection{Sampling and optimization perspectives} \label{subsec:sampling-optimization}
The core operation that we consider is sampling by running a stochastic diffusion process that depends on a parameter $\theta \in \R^p$. We consider {\em iterative sampling operators} that map from a {\em parameter space} to a {\em space of probabilities}. We denote by $\pi^\star(\theta) \in \cP$ the outcome distribution of this sampling operator. This parameterized distribution is defined in an {\em implicit} manner: there is not always an explicit way to write down its dependency on~$\theta$. More formally, it is defined as follows.
\begin{definition}[Iterative sampling operators]
\label{def:sampling}
    For a parameter $\theta \in \R^p$, a sequence of parameterized functions $\Sigma_s(\cdot, \theta)$ from $\cP$ to $\cP$ defines a diffusion {\em sampling process}, from $p_0 \in \cP$ iterating
    \begin{equation}
\label{eq:iter-sampling}    
p_{s+1} = \Sigma_s(p_s, \theta)\, .
\end{equation}
The outcome $\pi^\star(\theta) \in \cP$, when $s\to \infty$, or for some $s=T$ defines a {\em sampling operator}
$
\pi^\star: \R^p \to \cP \, .
$ \qed
\end{definition}
We embrace the formalism of stochastic processes as acting on probability distributions. This perspective focuses on the {\em dynamics} of the distribution $(p_s)_{s \ge 0}$, and allows us to more clearly present our optimization problem and algorithms. In practice, however, in all the examples that we consider, this is realized by an iterative process on some random variable~$X_s$ such that $X_s \sim p_s$. 

\begin{example}
Consider the process defined by $X_{s+1} = X_s - 2\delta (X_s-\theta) + \sqrt{2 \delta} B_s$,  where the $B_s$ are i.i.d. standard Gaussian, $X_0 \sim p_0 := \cN(\mu_0, \sigma_0^2)$, and $\delta \in (0, 1)$. This is the discrete-time version of Langevin dynamics for  $V(x, \theta) = 0.5(x - \theta)^2$ (see Section~\ref{subsec:examples}). The dynamics induced on probabilities $p_s = \cN(\mu_s, \sigma_s^2)$ are $\mu_s = \theta + (1 -2\delta)^s(\mu_0 -\theta)$ and $\sigma_s^2 = 1 + (1 -2\delta)^{2s}(\sigma^2_0 -1)$.
The sampling operator for $s \to \infty$ is therefore defined by $\pi^\star:\theta \to \cN(\theta, 1)$.\qed
\end{example}
More generally, we may consider the iterates $X_s$ of the  process defined for noise variables $(B_s)_{s \ge 0}$
\begin{equation}
\label{eq:variable-dym}
X_{s+1} = f_s(X_s, \theta) + B_s\, .
\end{equation}
Applying $f_s(\cdot, \theta)$ to $X_s \sim p_s$ implicitly defines a dynamic $\Sigma_s(\cdot, \theta)$ on the distribution. The dynamics on the variables in~(\ref{eq:variable-dym})
induce dynamics on the distributions described in~(\ref{eq:iter-sampling}).
Note that, in the special case of normalizing flows \citep{kobyzev2019normalizing,papamakarios_2021},
explicit formulas for~$p_s$ can be derived and evaluated.

\begin{remark}
{\em i)} We consider settings with discrete time steps, fitting our focus on algorithms to sample and optimize through sampling. This encompasses in particular the discretization of many continuous-time stochastic processes of interest. Most of our motivations are of this latter type, and we describe these distinctions in our examples (see Section~\ref{subsec:examples}).

{\em ii)} As noted above, these dynamics are often realized by an iterative process on variables $X_s$, or even on an i.i.d. batch of samples $(X^{1}_s, \ldots, X^{n}_s)$. 
When the iterates $\Sigma_s(p_s, \theta)$ are written in our presentation (e.g.~in optimization algorithms in Section \ref{sec:methods}), it is often a shorthand to mean that we have access to samples from $p_s$, or equivalently to an empirical version $\hat p_{s, (n)}$ of the population distribution $p_s$.
Sample versions of our algorithms are described in Appendix~\ref{apx:algos}.

{\em iii)} One of the \textbf{special cases} considered in our analysis are stationary processes with infinite time horizon, where the sampling operation can be interpreted as optimizing over the set of %
distributions 
\begin{equation} \label{eq:inner-sampling}
    \pi^\star(\theta) = \argmin_{p \in \cP} \cG(p, \theta)\, , \quad \text{for some $\cG: \cP \times \R^p \to \R$}\, .
\end{equation}
In this case, the iterative operations in~(\ref{eq:iter-sampling}) can often be directly interpreted as descent steps for the objective $\cG(\cdot, \theta)$. However, our methodology is \textbf{not limited} to this setting: we also consider general sampling schemes with no stationarity and no inner $\cG$, but only a sampling process defined by $\Sigma_s$.
\end{remark}

\paragraph{Optimization objective.} We aim to optimize with respect to $\theta$ the output of the sampling operator, for a function $\cF: \cP \to \R$. In other words, we consider the optimization problem
\begin{equation}
    \label{eq:optim-main}
    \min_{\theta \in \R^p} \ell(\theta) := \min_{\theta \in \R^p} \cF(\pi^\star(\theta))\, .
\end{equation}
This formulation transforms a problem over distributions in $\cP$ to a finite-dimensional problem over $\theta \in \R^p$.  Further, this allows for convenient post-optimization sampling: for some $\theta_{\text{opt}} \in \R^p$ obtained by solving~(\ref{eq:optim-main}), one can sample from $\pi^\star(\theta_{\text{opt}})$. This is the common paradigm in model finetuning. 

\begin{figure*}[!ht]
    \centering
  \includegraphics[width=0.98\textwidth]{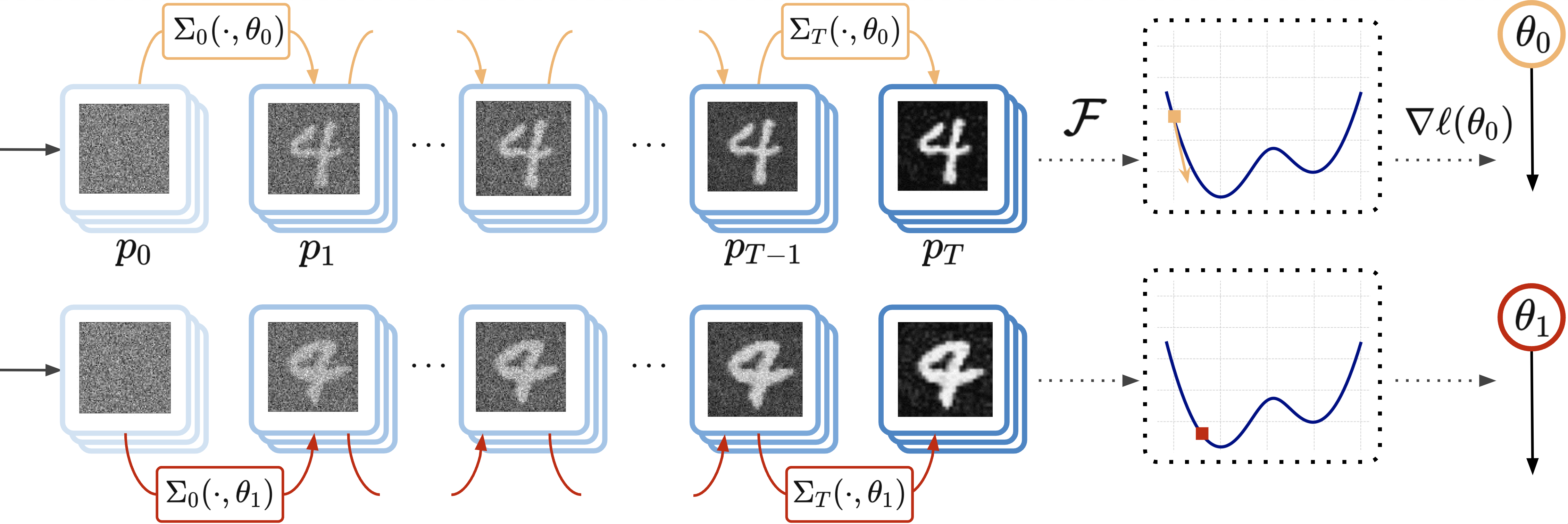}
  \caption{A step of optimization through sampling. For a given parameter $\theta_0$, the sampling process is defined by applying~$\Sigma_s$ for $s \in [T]$, producing $\pi^\star(\theta_0)$. The goal of optimization through sampling is to update $\theta$ to minimize $\ell = \cF \circ \pi^\star$. Here the objective $\cF$ corresponds to having lighter images (on average), which produces thicker digits.}
  \label{fig:algo}
\end{figure*}
\subsection{Examples}   \label{subsec:examples}

\paragraph{Langevin dynamics.}
They are defined by the stochastic differential equation (SDE) \citep{roberts1996exponential}
\begin{equation}
\label{eq:langevin}
\ud X_t = -\nabla_1 V(X_t, \theta) \ud t + \sqrt{2} \ud B_t\, , 
\end{equation}
where $V$ and $\theta \in \R^p$ are such that this SDE has a solution for $t > 0$ that converges in distribution. Here $\pi^\star(\theta)$ is the limiting distribution of $X_t$ when $t \to \infty$, which is the Gibbs distribution
\begin{equation}\label{eq:gibbs}
\pi^\star(\theta)[x] = \exp(-V(x, \theta)) / Z_\theta\, ,
\end{equation}
where $Z_\theta = \int \exp(-V(x, \theta)) \ud x$ is the normalization factor.
To fit our setting of iterative sampling algorithms \eqref{eq:variable-dym}, one can consider the discretization for $\gamma>0$
\[
X_{k+1} = X_k - \gamma \nabla_1 V(X_k, \theta) + \sqrt{2 \gamma} \Delta B_{k+1} \, .
\]
For $\cG(p, \theta) = \KL(p || \pi^\star(\theta))$, the outcome of the sampling operator $\pi^\star(\theta)$ is a minimum of $\cG(\cdot, \theta)$, and the SDE \eqref{eq:langevin} %
implements a gradient flow for $\cG$ in the space of measures, with respect to the Wasserstein-2 distance \citep{jordan1998variational}. 
Two optimization objectives~$\cF$ are of particular interest in this case. First, we may want to maximize some reward~$R: \R^d \to \R$ over our samples, in which case the objective writes $\cF(p) := -\E_{x\sim p}[R(x)]$. Second, to approximate a reference distribution $p_{\text{ref}}$ with sample access, it is possible to take $\cF(p) := \text{KL}(\pref \,|| \, p)$. This case corresponds to training energy-based models \citep{gutmann2012noise}. It is also possible to consider a linear combination of these two objectives.
\paragraph{Denoising diffusion.}   \label{subsec:diffusion-denoising}
It consists in running the SDE for $Y_0 \sim \cN(0, I)$,
\begin{equation}
\label{eq:diffusionback}
\ud Y_t = \{Y_t + 2 s_\theta(Y_t, T-t)\} \ud t + \sqrt{2} \ud B_t \, ,
\end{equation}
where $s_\theta: \R^d \times [0,T] \to \R^d$ is a parameterized %
{\em score function} \citep{hyvarinen2005estimation,vincent2011connection,ho2020denoising}. %
Its aim is to reverse a forward process
$\ud X_t = -X_t \ud t + \sqrt{2} \ud B_t$,  
where we have sample access to $X_0 \sim \pdata \in \cP$. More precisely, denoting by $p_t$ the distribution of $X_t$, if $s_\theta \approx \nabla \log p_t$, then the distribution of $Y_T$ is close to $\pdata$ for large $T$ \citep{anderson1982reverse}, which allows approximate sampling from $\pdata$.
Implementations of $s_\theta$ include U-Nets \citep{ronneberger2015u} or Vision Transformers that split the image into patches \citep{peebles2023scalable}. We present for simplicity an unconditioned model, but conditioning (on class, prompt, etc.) also falls in our framework.

We are interested in optimizing through diffusion sampling and consider $\pi^\star(\theta)$ as the distribution of $Y_T$. A key example is when $\theta_0$ represents the weights of a model $s_{\theta_0}$ that has been pretrained by score matching %
and one wants to finetune the target distribution $\pi^\star(\theta)$, e.g.~in order to increase a reward $R: \R^d \to \R$. 
Figure~\ref{fig:flowchart} situates our contribution within the broader literature on this problem (see details in Appendix \ref{apx:additional-related-work}).
Note that this finetuning step does not require access to $\pdata$.
As for Langevin dynamics, we consider in our algorithms discrete approximations of the process \eqref{eq:diffusionback}.
However in this case, there exists no natural functional $\cG$ minimized by the sampling process.
An alternative to \eqref{eq:diffusionback} producing the same marginal distributions is the ordinary differential equation (ODE)
\begin{equation}
\label{eq:diffusion-ode-back}
Y_0 \sim \cN(0, I) \, , \quad \ud Y_t = \{Y_t + s_\theta(Y_t, T-t)\} \ud t\, .
\end{equation}

\section{METHODS}
\label{sec:methods}

The objective \eqref{eq:optim-main} %
presents several challenges, that we review here. We then introduce an overview of our approach, before detailing our algorithms.

\subsection{Overview}   \label{subsec:overview}

\paragraph{Estimation of gradients through sampling.} Even with samples from $\pi^\star(\theta)$, applying a first-order method to \eqref{eq:optim-main} requires evaluating gradients of $\ell = \cF \circ \pi^\star$. Since there is no closed form for $\ell$ and no explicit computational graph, we consider the following alternative setting to evaluate gradients. %
\begin{definition}[Implicit gradient estimation]    \label{def:implicit-gradient-estimation}
The gradient of $\ell$ can be \textit{implicitly estimated} if $\Sigma_s, \cF$ are such that there exists $\Gamma: \cP \times \R^p \to \R^p$ such that $\nabla \ell(\theta) = \Gamma(\pi^\star(\theta), \theta)$. \qed
\end{definition}

In practice we rarely reach exactly the distribution~$\pi^\star(\theta)$, e.g.~because a finite number of iterations of sampling is performed. Then, if $\hat \pi \approx \pi^\star(\theta)$, the gradient can be approximated by $\hat g = \Gamma(\hat \pi, \theta)$. 
In particular, given access to approximate samples of $\pi^\star(\theta)$, it is possible to compute an estimate of $\nabla \ell(\theta)$, and this is at the heart of our methods--see Appendix~\ref{subsec:gradient-estimation-abstraction} for more details. 
Note that when $\Gamma$ is linear in its first argument, sample access to $\pi^\star(\theta)$ yields unbiased estimates of the gradient. This case has been studied with various approaches \citep[see][and Appendix~\ref{apx:additional-related-work}]{sutton1999policy,fu2012conditional,pflug2012optimization,debortoli2021efficient}.
\begin{figure}[h!]
    \centering
    \includegraphics[width=0.4\textwidth]{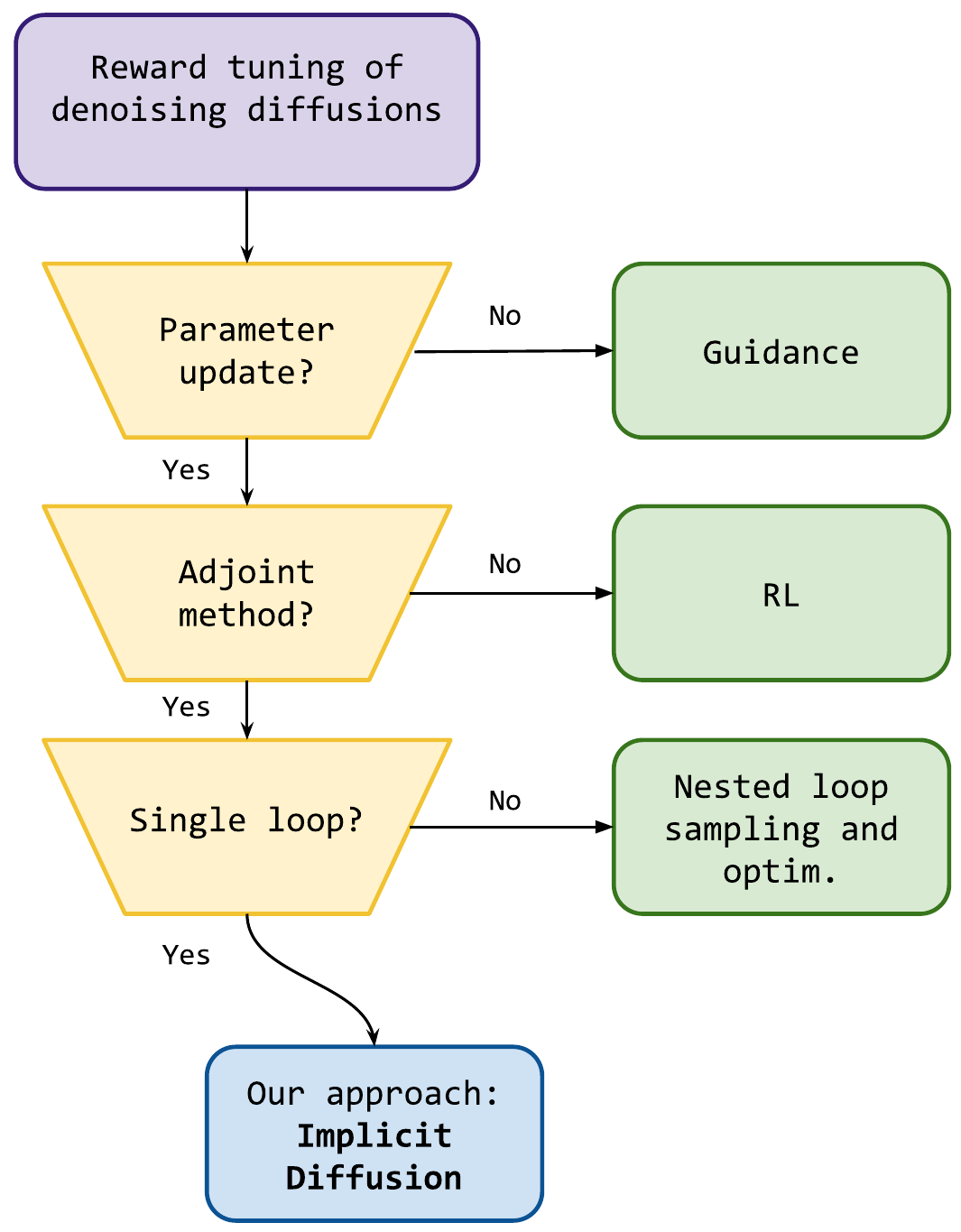}
    \caption{Main approaches for reward tuning of denoising diffusions. References are given in Appendix~\ref{apx:additional-related-work}.
    \label{fig:flowchart}}
\end{figure}
\begin{remark}
Definition \ref{def:implicit-gradient-estimation} is in fact always satisfied by the tautological definition $\Gamma(p,\theta) := \nabla \ell(\theta)$. Rather than the existence of $\Gamma$, key questions are whether $\Gamma$ is easily computable, can be stochastically approximated as explained above, or enjoys theoretical guarantees. We present several sampling problems in Section~\ref{subsubsec:implicit-diff} where this is the case. This point of view should also extend beyond sampling to any setting that involves optimizing over a parameterized distribution, as for instance in Wasserstein Distributionally Robust Optimization \citep{mohajerin2018data}, a method for robust learning. We leave these extensions to future investigations.
\end{remark}

\paragraph{Beyond nested-loop approaches.}
Sampling from $\pi^\star(\theta)$ is usually only feasible via iterations of the sampling process~$\Sigma_s$. The most straightforward method is then a nested loop: at each optimization step $k$, run an inner loop for a large number $T$ of steps of $\Sigma_s$ to produce $\hat \pi_{k} \approx \pi^\star(\theta_k)$, and use it to evaluate a gradient. Algorithm~\ref{alg:naive-nested} formalizes this baseline. 
This approach can be inefficient for two reasons: first, it requires solving the inner sampling problem {\em at each optimization step}. 
Further, nested loops are typically impractical with accelerator-oriented hardware. 
These issues can be partially alleviated by techniques like gradient checkpointing (see references in Appendix~\ref{apx:additional-related-work}).

\begin{algorithm}[H]
\caption{Vanilla nested-loop approach (Baseline)} \label{alg:naive-nested}
\begin{algorithmic}
\INPUT $\theta_0 \in \R^p$, $p_0 \in \cP$
\FOR{$k \in \{0, \dots, K-1\}$ (outer optimization loop)} 
    \STATE $p_k^{(0)} \gets p_0$
    \FOR{$s \in \{0, \dots, T-1\}$ (inner sampling loop)}
        \STATE $p_{k}^{(s+1)} \gets \Sigma_{s}(p_{k}^{(s)}, \theta_k)$
    \ENDFOR
    \STATE $\hat \pi_k \gets p_k^{(T)}$
    \STATE $\theta_{k+1} \gets \theta_k - \eta \Gamma(\hat \pi_k, \theta_k)$ (or another optimizer)
\ENDFOR
\OUTPUT $\theta_K$
\end{algorithmic}
\end{algorithm}

We rather %
follow an approach inspired by methods in bilevel optimization, aiming to {\em jointly} iterate on both the sampling problem (evaluation of $\pi^\star$--the inner problem), and the optimization problem over $\theta \in \R^p$ (the outer objective~$\cF$). We describe these methods in Section~\ref{subsec:implicit-diff} and Algorithms~\ref{alg:implicit-diff} and~\ref{alg:implicit-diff-maxi}. The connection with bilevel optimization is especially seamless when sampling can indeed be cast as an optimization problem over distributions in $\cP$, as in~(\ref{eq:inner-sampling}). However, as noted above, our approach generalizes beyond.

\subsection{Gradient estimation through sampling}
\label{sec:grad-sampling}

We explain how to perform implicit gradient estimation as in Definition \ref{def:implicit-gradient-estimation}, that is, how to derive expressions for the function $\Gamma$, in several cases of interest.

\paragraph{Direct analytical derivation.}
For Langevin dynamics, it is possible to derive analytical expressions for~$\Gamma$ depending on the outer objective $\cF$. 
We illustrate this idea for the two objectives introduced in Section~\ref{subsec:examples}. First, in the case where $\cF_{\text{rew}}(p) = -\E_{x\sim p}[R(x)]$, a computation detailed in Appendix \ref{subsec:apx:joint-optimization-algos} and the use of Definition \ref{def:implicit-gradient-estimation} show that, for $\ell_{\text{rew}}(\theta) = \cF_{\text{rew}}(\pi^\star(\theta))$,
\begin{align}
\begin{split}
\label{eq:gamma-1}
\nabla \ell_{\text{rew}}(\theta) &= \cov\nolimits_{X \sim \pi^\star(\theta)}[R(X), \nabla_2 V(X, \theta)] \, , \\
    \Gamma_{\text{rew}}(p, \theta) &= \cov\nolimits_{X \sim p}[R(X), \nabla_2 V(X, \theta)] \, .
\end{split}
\end{align}
Note that this formula does not involve gradients of~$R$, hence our approach handles any non-differentiable reward.
Second, in the case where $\cF_{\text{ref}}(p) = \text{KL}(\pref \,|| \, p)$, we then have \citep{gutmann2012noise}, for $\ell_{\text{ref}}(\theta) = \cF_{\text{ref}}(\pi^\star(\theta))$,
\begin{align*}
\nabla \ell_{\text{ref}}(\theta) 
&= \E_{X \sim p_{\text{ref}}}[\nabla_2 V(X, \theta)] - \E_{X \sim \pi^\star(\theta)}[\nabla_2 V(X, \theta)].
\end{align*}
This is known as contrastive learning when $p_{\text{ref}}$ is given by data, and suggests taking $\Gamma$ as
\begin{equation}    \label{eq:gamma-2}
    \Gamma_{\text{ref}}(p, \theta)  :=  \E_{X \sim p_{\text{ref}}}[\nabla_2 V(X, \theta)] - \E_{X \sim p}[\nabla_2 V(X, \theta)] \, .
\end{equation}
This extends to linear combinations of $\Gamma_{\text{rew}}$ and~$\Gamma_{\text{ref}}$.

\paragraph{Implicit differentiation.}    
\label{subsubsec:implicit-diff}
When, as in \eqref{eq:inner-sampling},  $\pi^\star(\theta) = \argmin \cG(\cdot, \theta)$, under generic assumptions on $\cG$ the implicit function theorem (see \citealp{krantz2002implicit,blondel2022efficient} and Appendix~\ref{apx:implicit-differentiation}) shows that
$
\nabla \ell(\theta)
= \Gamma(\pi^\star(\theta), \theta)$ with %
\[
\Gamma(p, \theta) = \int \cF'(p)[x] \gamma(p, \theta)[x] \ud x \, .
\]
Here $\cF'(p):\cX \to  \R$ denotes the first variation of $\cF$ at $p\in \cP$ (see \Cref{def:first_var}) and $\gamma(p, \theta)$ is the solution of the linear system
$
\int \nabla_{1, 1} \cG(p, \theta)[x, x'] \gamma(p, \theta)[x'] \ud x' = - \nabla_{1, 2} \cG(p, \theta) [x]
$.
Although this gives us a general way to define gradients of $\pi^\star(\theta)$ with respect to $\theta$, solving this linear system is generally not feasible. One exception is when sampling over a finite state space $\cX$, in which case~$\cP$ is finite-dimensional, and the integrals boil down to matrix-vector products. 

\paragraph{Differential adjoint method.} The adjoint method allows computing gradients through differential equation solvers \citep{pontryagin2018mathematical,li2020scalable}, applying in particular for denoising diffusion. It can be connected to implicit differentiation, by defining~$\cG$ over a measure path instead of a single measure $p$ (see, e.g.,~\citealp{kidger2022neural}). 
To introduce this method, consider the ODE $\ud Y_t = \mu(t, Y_t, \theta) \ud t$ integrated between~$0$ and some $T>0$. This setting encompasses the denoising diffusion ODE~\eqref{eq:diffusion-ode-back}. 
Assume that the outer objective $\cF$ writes as the expectation of some differentiable reward $R$, namely $\cF(p) = \E_{x\sim p}[R(x)]$. Let $ Z_0 \sim p, A_0= \nabla R(Z_0), G_0= 0$, and consider the ODE system
\begin{align*}
   \ud Z_t = - \mu(t, Z_t, \theta) &\ud t \, , \quad 
   \ud A_t = A_t^\top \nabla_2 \mu(T-t, Z_t, \theta) \ud t \, , \\
   \ud G_t &= A_t^\top \nabla_3 \mu(T-t, Z_t,\theta) \ud t \,  .    
\end{align*}
Defining $\Gamma(p, \theta) := G_T$, the adjoint method shows that $\Gamma(\pi^\star(\theta), \theta)$ is an unbiased estimate of $\nabla \ell(\theta)$. We refer to Appendix~\ref{apx:adjoint-method} for details and explanations on how to differentiate through the SDE sampler~\eqref{eq:diffusionback} and to incorporate a KL term in the reward by using Girsanov's theorem.

\subsection{Implicit Diffusion optimization algorithm} \label{subsec:implicit-diff}
To circumvent solving the inner sampling problem in Algorithm \ref{alg:naive-nested}, we propose in Algorithm~\ref{alg:implicit-diff} a \textbf{joint single-loop approach} that keeps track of a single dynamic of probabilities $(p_k)_{k \ge 0}$. At each optimization step, the probability $p_k$ is updated with one sampling step that depends on the current parameter $\theta_k$. As noted in Section \ref{subsec:overview}, there are parallels with approaches in the literature when $\Gamma$ is linear, but we go beyond in making no such assumption.

\begin{figure*}[!ht]
    \centering
  \includegraphics[width=\textwidth]{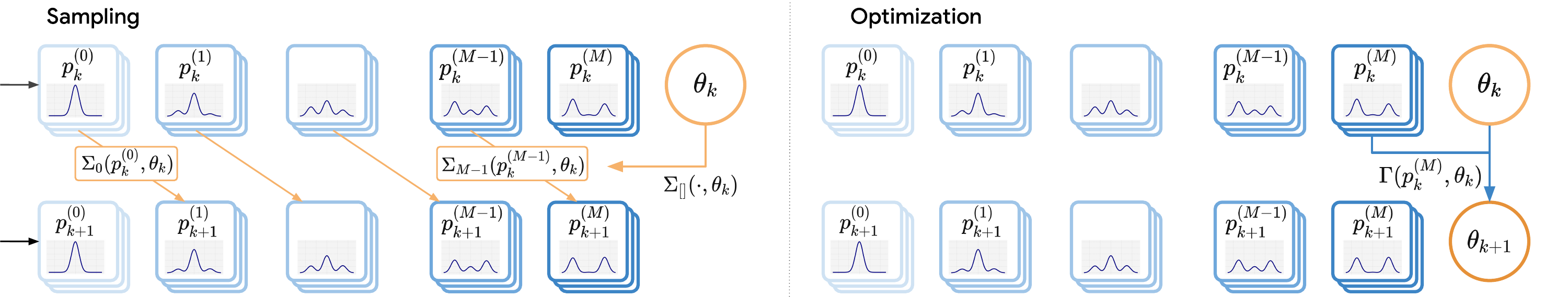}
  \caption{Illustration of the Implicit Diffusion algorithm, in the finite time setting. \underline{Left:} Sampling - one step of the parameterized sampling scheme is applied in parallel to all distributions in the queue. \underline{Right:} Optimization - the last element of the queue is used to compute a gradient for the parameter.}
  \label{fig:maxi-batch}
\end{figure*}

\begin{algorithm}[H]
\caption{Implicit Diff. optimization, infinite time} \label{alg:implicit-diff}
\begin{algorithmic}
\INPUT $\theta_0 \in \R^p$, $p_0 \in \cP$
\FOR{$k \in \{0, \dots, K-1\}$ (joint single loop)} 
    \STATE $p_{k+1} \gets \Sigma_{k}(p_k, \theta_k)$
    \STATE $\theta_{k+1} \gets \theta_k - \eta \Gamma(p_k, \theta_k)$ (or another optimizer)
\ENDFOR
\OUTPUT $\theta_K$
\end{algorithmic}
\end{algorithm}

This point of view is well-suited for stationary processes with infinite-time horizon, but does not directly adapt to differentiation through diffusions with a finite-time horizon (and no stationary property). Indeed, it does not make sense in this case to run sampling for an arbitrary number of steps. Our approach can then be adapted as follows.

\begin{algorithm}[H]
\caption{Implicit Diff. optimization, finite time} \label{alg:implicit-diff-maxi}
\begin{algorithmic}
\INPUT $\theta_0 \in \R^p$, $p_0 \in \cP$
\INPUT $P_M = [p_0^{(0)},\ldots, p_0^{(M)}]$
\FOR{$k \in \{0, \dots, K-1\}$ (joint single loop)}
    \STATE $p_{k+1}^{(0)} \gets p_0$
    \STATE \textbf{parallel} $p_{k+1}^{(m+1)} \gets \Sigma_{m}(p_k^{(m)}, \theta_k)$ for $m \in [M-1]$
    \STATE $\theta_{k+1} \gets \theta_k - \eta \Gamma(p_k^{(M)}, \theta_k)$ (or another optimizer)
\ENDFOR
\OUTPUT $\theta_K$
\end{algorithmic}
\end{algorithm}

\paragraph{Finite time-horizon: queuing trick.}
When $\pi^\star(\theta)$ is obtained by a large, but {\em finite} number $T$ of iterations of the operator $\Sigma_s$, we leverage hardware parallelism to evaluate in parallel several, say $M$, steps of the dynamics of the distribution $p_k$, through a queue of length $M$. We present for simplicity in Figure~\ref{fig:maxi-batch} and in Algorithm \ref{alg:implicit-diff-maxi} the case where $M=T$ and discuss extensions in Appendix~\ref{apx:adjoint-method}.
At each step, the last element of the queue $p_k^{(M)}$ provides a distribution to update $\theta$ through evaluation of $\Gamma$. Note that its dynamics in the previous $M$ steps, from $p_{k-M}^{(0)}$ to $p_k^{(M)}$, used the $M$ previous values of the parameter $\theta_{k-M}, \ldots ,\theta_{k-1}$. In particular, it provides a biased estimate of $\nabla \ell(\theta_k)$.
Importantly, leveraging parallelism, the runtime of our algorithm is $\mathcal{O}(K)$, gaining a factor of $T$ compared to the nested-loop approach, but at a higher memory cost. %

\section{THEORETICAL ANALYSIS}  \label{sec:theory}

Our theoretical guarantees cover Langevin diffusions in the continuous and discrete settings, and a simple case of denoising diffusion. Proofs are given in Appendix \ref{apx:proofs}.

\subsection{Langevin with continuous flow}  \label{subsec:theory-langevin-continuous}

The continuous-time equivalent of Algorithm \ref{alg:implicit-diff} in the case of Langevin dynamics is
\begin{align}   \label{eq:langevin-theory}
\begin{split}
    \ud X_t &= -\nabla_1 V(X_t, \theta_t) \ud t + \sqrt{2} \ud B_t \, , \\
    \ud\theta_t &= - \varepsilon_t \Gamma (p_t, \theta_t) \ud t \, ,
\end{split}
\end{align}
where $p_t$ denotes the distribution of $X_t$ and $\varepsilon_t > 0$ corresponds to the ratio of learning rates between the inner and the outer problems. 
In practice $\Gamma(p_t, \theta_t)$ is approximated on a finite sample, making the dynamics in $\theta_t$ stochastic. We leave the analysis of these stochastic dynamics for future work. 
A possible tool to do so is the \textit{propagation of chaos} \citep{chaintron2022propagation,suzuki2023uniform}, a theory which aims at quantifying the deviation between the dynamics of a system of finitely many interacting particles and the limiting behavior described by a mean-field density.

Recalling the definition \eqref{eq:gibbs} of $\pi^\star(\theta)$, we require the following assumptions.
\begin{assumption}\label{ass:potential_log_sobolev}
    $\pi^\star(\theta_t)$ verifies the Log-Sobolev inequality with constant $\mu > 0$ for all $t\ge 0$, i.e., for all $p\in \cP$,
    $
        \KL(p \, || \, \pi^\star(\theta_t))\le \frac{1}{2\mu} \Vert \nabla \log ( \frac{p}{\pi^\star(\theta_t)})  \Vert^2_{L^2(p)} \, .
    $
\end{assumption}
\begin{assumption}\label{ass:gradient_bounded}
$V$ is continuously differentiable and for $\theta \in \R^p, x \in \R^d$, $\|\nabla_2 V(x, \theta)\| \leq C$, for some $C>0$.
\end{assumption}
Assumption \ref{ass:potential_log_sobolev} generalizes $\mu$-strong convexity of the potentials $(V(\cdot,\theta_t))_{t\ge 0}$, including for instance distributions $\pi$ whose potentials are bounded perturbations of a strongly convex potential \citep{bakry2014analysis,vempala2019rapid}.
Assumptions \ref{ass:potential_log_sobolev} and \ref{ass:gradient_bounded} hold for example when the potential defines a mixture of Gaussians and the parameters~$\theta$ determine the weights of the mixture (see Appendix~\ref{apx:gaussian-mixture} for details).
We also assume that the outer updates are bounded and Lipschitz continuous for the KL divergence.
\begin{assumption}\label{ass:Gamma_Lipschitz}
For $p, q \in \cP$, $\theta \in \R^p$, 
$\|\Gamma(p,\theta)\| \leq C$ and $\|\Gamma(p, \theta) - \Gamma(q, \theta)\| \leq K_\Gamma \sqrt{\KL(p || q)}$, for some $C, K_\Gamma >0$.
\end{assumption}
The next proposition shows that this assumption holds for the examples of interest given in Section \ref{sec:problem-presentation}.
\begin{proposition} \label{prop:verif-ass-gamma-lip}
Consider a bounded function $R: \R^d \to \R$. Then, under Assumption \ref{ass:gradient_bounded}, functions $\Gamma_{\text{rew}}$ and $\Gamma_{\text{ref}}$ defined by \eqref{eq:gamma-1}--\eqref{eq:gamma-2} satisfy Assumption~\ref{ass:Gamma_Lipschitz}.
\end{proposition}
Since we make no strong convexity assumption, we cannot hope to prove convergence to a global minimizer, rather convergence of the average objective gradients. Note that, with strong convexity, it is a rather standard extension to prove convergence to the global minimizer \citep[see, e.g., in a similar context,][]{dagreou2022framework}.
\begin{theorem} \label{thm:langevin-continuous}
Take $\varepsilon_t = \min(1, 1/\sqrt{t})$in \eqref{eq:langevin-theory}. Then, under Assumptions \ref{ass:potential_log_sobolev}, \ref{ass:gradient_bounded}, and \ref{ass:Gamma_Lipschitz},
\[
\frac{1}{T} \int_0^T \|\nabla \ell(\theta_t)\|^2 \ud t \leq \frac{c(\ln T)^2}{T^{1/2}}  \, , \qquad \textnormal{for some } c > 0 \, .
\]
\end{theorem}
The proof starts by noting that updates in $\theta$ would follow the gradient flow for $\ell$ if $p_t = \pi^\star(\theta_t)$. 
The deviation to these ideal dynamics can be controlled by the KL divergence of $p_t$ from $\pi^\star(\theta_t)$, which can itself be bounded since updates in~$X_t$ are gradient steps for the KL (see Section~\ref{subsec:examples}). 
Finally, the decay of the ratio of learning rates $\varepsilon_t$ ensures that $\pi^\star(\theta_t)$ is not moving away from $p_t$ too fast due to updates in~$\theta_t$.
Taking $\varepsilon_t$ small amounts to a \textit{two-timescale approach}, a tool commonly used to tackle non-convex optimization problems in machine learning \citep{heusel2017gans,arbel2022amortized,hong2023two,marion2023}.

\subsection{Langevin with discrete flow}

We now consider the discrete version of \eqref{eq:langevin-theory}, namely
\begin{align}\label{eq:langevin-discrete}
\begin{split}
    X_{k+1} &= X_k - \gamma_k \nabla_1 V(X_k, \theta_k) + \sqrt{2 \gamma_k} \Delta B_{k+1} \, , \\
    \theta_{k+1} &= \theta_k - \gamma_k \varepsilon_k \Gamma (p_k, \theta_k) \, , 
\end{split}
\end{align}
where $p_k$ denotes the distribution of $X_k$.
This setting is more challenging due to the discretization bias \citep[][]{d-tgasslcd-17}. 
We make classical smoothness assumptions to analyze discrete gradient descent \citep[e.g.,][]{cb-clmcmckld-18}:
\begin{assumption}  \label{ass:smoothness}
The functions $\nabla_1 V(\cdot,\theta)$, $\nabla_1 V(x, \cdot)$ and~$\nabla \ell$ are respectively $L_X$-Lipschitz for all $\theta \in \R^p$, $L_\Theta$-Lipschitz for all $x \in \R^d$, and $L$-Lipschitz.
\end{assumption}
We can then show the following convergence result.
\begin{theorem}     \label{thm:langevin-discrete}
Take $\gamma_k = c_1/\sqrt{k}$ and $\varepsilon_k = 1/\sqrt{k}$ in \eqref{eq:langevin-discrete}.
Under Assumptions \ref{ass:potential_log_sobolev}, \ref{ass:gradient_bounded}, \ref{ass:Gamma_Lipschitz}, and~\ref{ass:smoothness},
\[
\frac{1}{K} \sum_{k=1}^{K} \|\nabla \ell(\theta_k)\|^2 \leq \frac{c_2 \ln K}{K^{1/3}} \, , \qquad \textnormal{for some } c_1,c_2 > 0 \, .
\]
\end{theorem}
The proof technique to bound the KL in discrete iterations is inspired by \cite{cb-clmcmckld-18}. Comparing to the continuous case, this step incurs a discretization error proportional to $\gamma_k \to 0$, inducing a slower convergence rate. 
The result is similar to \citet{dagreou2022framework} for finite-dimensional bilevel optimization, albeit our final convergence rate is slower.

\subsection{Denoising diffusion}
\label{theory:diffusion-denoising}

The case of denoising diffusion is more challenging since $\pi^\star(\theta)$ can \textbf{not} be readily characterized as the stationary point of an iterative process. We study a 1D Gaussian case and leave more general analysis for future work. Considering $p_{\text{data}} = \cN(\theta_{\textnormal{data}}, 1)$ and the forward process of Section \ref{subsec:examples}, a straightforward computation shows that the score is given by $\nabla \log p_t(x) = -(x-\theta_{\textnormal{data}} e^{-t})$. A natural parameterization of the score function is therefore $s_{\theta}(x, t) := -(x-\theta e^{-t})$.
With this parameterization, the output of the sampling process~\eqref{eq:diffusionback} is $\pi^\star(\theta) = \cN(\theta (1-e^{-2T}), 1)$.
Remarkably, $\pi^\star(\theta)$ is Gaussian for all $\theta \in \R$, making the analytical study tractable.

Assume that pretraining with samples of $p_{\text{data}}$ yields $\theta = \theta_0$, and we want to finetune the model towards some other $\theta_{\textnormal{target}} \in \R$ by optimizing the reward $R(x)=-(x- \theta_{\textnormal{target}})^2$ over samples of $\pi^\star(\theta)$.
A short computation shows that $\nabla \ell(\theta) = -\E_{x \sim \pi^\star(\theta)} R'(x) (1 - e^{-2T})$, hence one can take  $\Gamma(p, \theta) = -\E_{x \sim p} R'(x) (1 - e^{-2T})$.
In this setting, we study a continuous-time version of Algorithm~\ref{alg:implicit-diff-maxi}, where~$\Sigma$ is the denoising diffusion~\eqref{eq:diffusionback} and~$\Gamma$ is given above, and show convergence of $\theta$ to $\theta_{\textnormal{target}}$. This shows that Algorithm \ref{alg:implicit-diff-maxi} successfully finetunes the parameter to optimize the reward. We refer to Appendix~\ref{apx:theory-diffusion-denoising} for details.

\begin{proposition} (informal)    \label{prop:diffusion-denoising-proof}
Let $(\theta_t)_{t \geq 0}$ be given by the continuous-time equivalent of Algorithm \ref{alg:implicit-diff-maxi}. Then
$\|\theta_{2T} - \theta_{\textnormal{target}}\| = \mathcal{O}(e^{-T})$,
and $\pi^\star(\theta_{2T}) = \cN(\mu_{2T}, 1)$ with $\mu_{2T} = \theta_{\textnormal{target}} + \mathcal{O}(e^{-T})$.
\end{proposition}

\section{EXPERIMENTS}   \label{sec:experiments}

We empirically illustrate the performance of Implicit Diffusion. Details are given in Appendix~\ref{apx:experimental-details}.

\subsection{Reward training of Langevin processes}
\label{sec:experiments-langevin}

We consider the case %
where the potential $V(\cdot, \theta)$ is a logsumexp of quadratics---so that the outcome distributions are mixtures of Gaussians. We optimize the reward $R(x) = 
\mathbf{1}(x_1 >0) \exp(-\|x-\mu\|^2)$, for $\mu \in \R^d$, thereby illustrating the ability of our method to optimize rewards that are not differentiable. 
We run six\linebreak sampling algorithms, including the infinite time-horizon version of \textbf{Implicit Diffusion} (Algorithm~\ref{alg:implicit-diff}), all starting from $p_0 = \cN(0, I_d)$ and for $K=5,000$ steps. 
\begin{itemize}[topsep=0pt,itemsep=2pt,parsep=2pt,leftmargin=15pt]
    \item[${\color{gray} \blacksquare}$] Langevin diffusion \eqref{eq:langevin} with potential $V(\cdot, \theta_0)$ for some fixed $\theta_0 \in \R^p$, no reward.
    \item[${\color{myblue} \bigstar}$] Implicit Diffusion with $\cF(p) = - \E_{X\sim p}[R(X)]$, yields both a sample $\hat p_K$ and parameters~$\theta_{\text{opt}}$.
    \item[${\color{myorange} \blacktriangledown}$] Langevin \eqref{eq:langevin} with potential $V(\cdot, \theta_0) - \lambda R_{\text{smooth}}$, where $R_{\text{smooth}}$ is a smoothed version of $R$.
    \item[${\color{mycoral} \CIRCLE}$] Langevin \eqref{eq:langevin} with potential $V(\cdot, \theta_{\text{opt}})$. This is inference post-training with Implicit Diffusion.
    \item[${\color{mypalegreen}  \blacklozenge}$] ${\color{mydarkgreen}\blacklozenge }$
    Nested loop (Algorithm \ref{alg:naive-nested}) with $T$ inner sampling steps for each gradient step.
    \item[${\color{mythistle} \blacktriangle}$] ${\color{mydarkorchid} \blacktriangle}$ Unrolling through the last step of sampling with $T$ inner sampling steps for each outer step.
\end{itemize}

\begin{figure}[ht]
\begin{center}
\centerline{\includegraphics[width=\columnwidth]{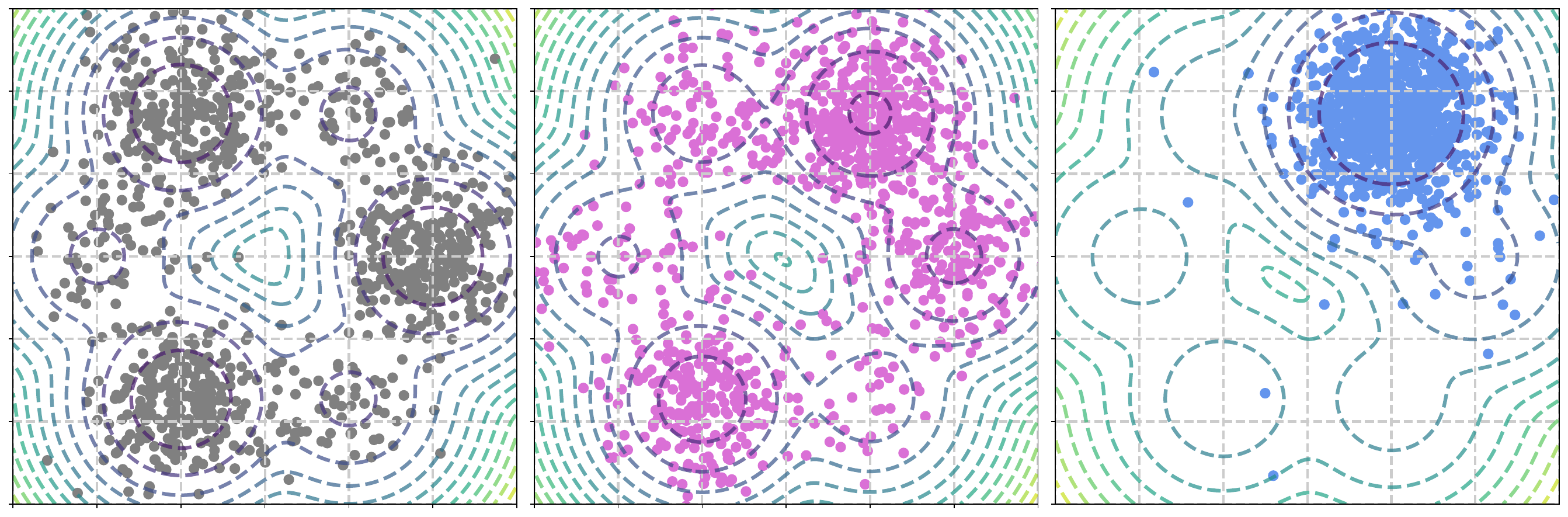}}
\caption{Contour lines and samples for $({\color{gray} \CIRCLE})$: Langevin $\theta_0$ - $({\color{mydarkorchid} \CIRCLE})$ Unrolling with $T=100$ inner sampling steps - $({\color{myblue} \CIRCLE})$
Implicit Diffusion. 
\label{fig:contour-langevin}}
\end{center}
\end{figure}

\begin{figure*}[!ht]
  \begin{center}
\begin{subfigure}{0.74\textwidth}
\centerline{\includegraphics[width=\textwidth]{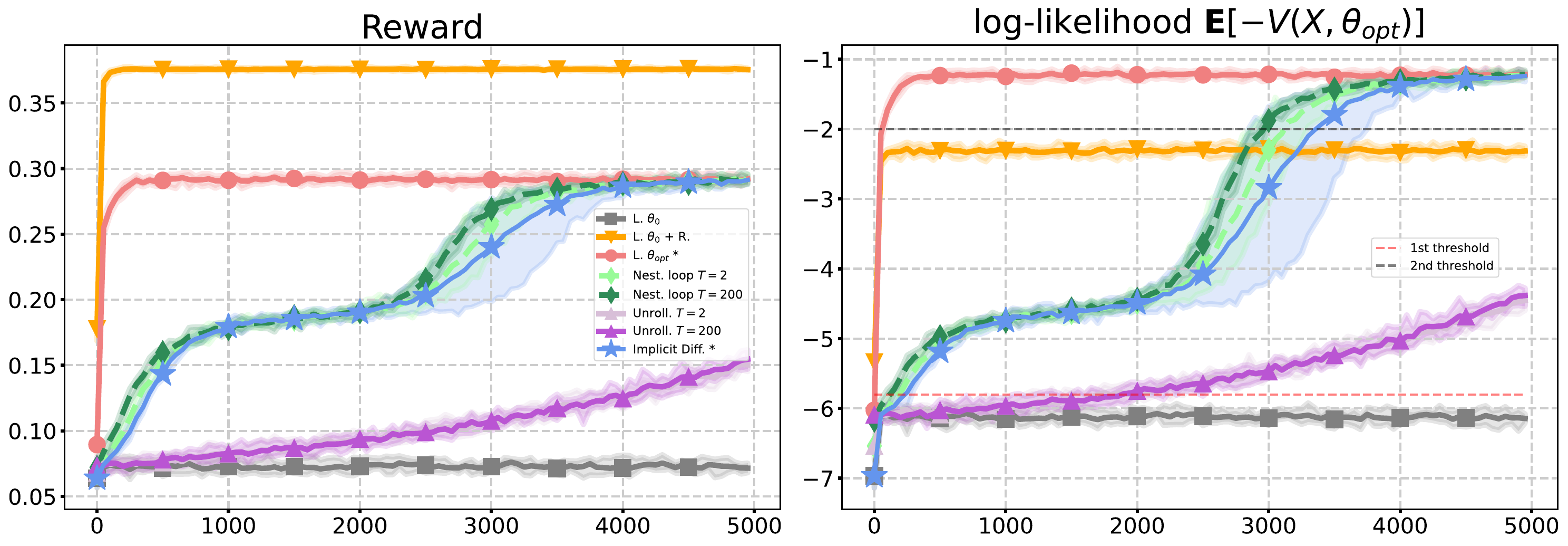}}
\end{subfigure}
\begin{subfigure}{0.25\textwidth}
\centerline{\includegraphics[width=\textwidth]{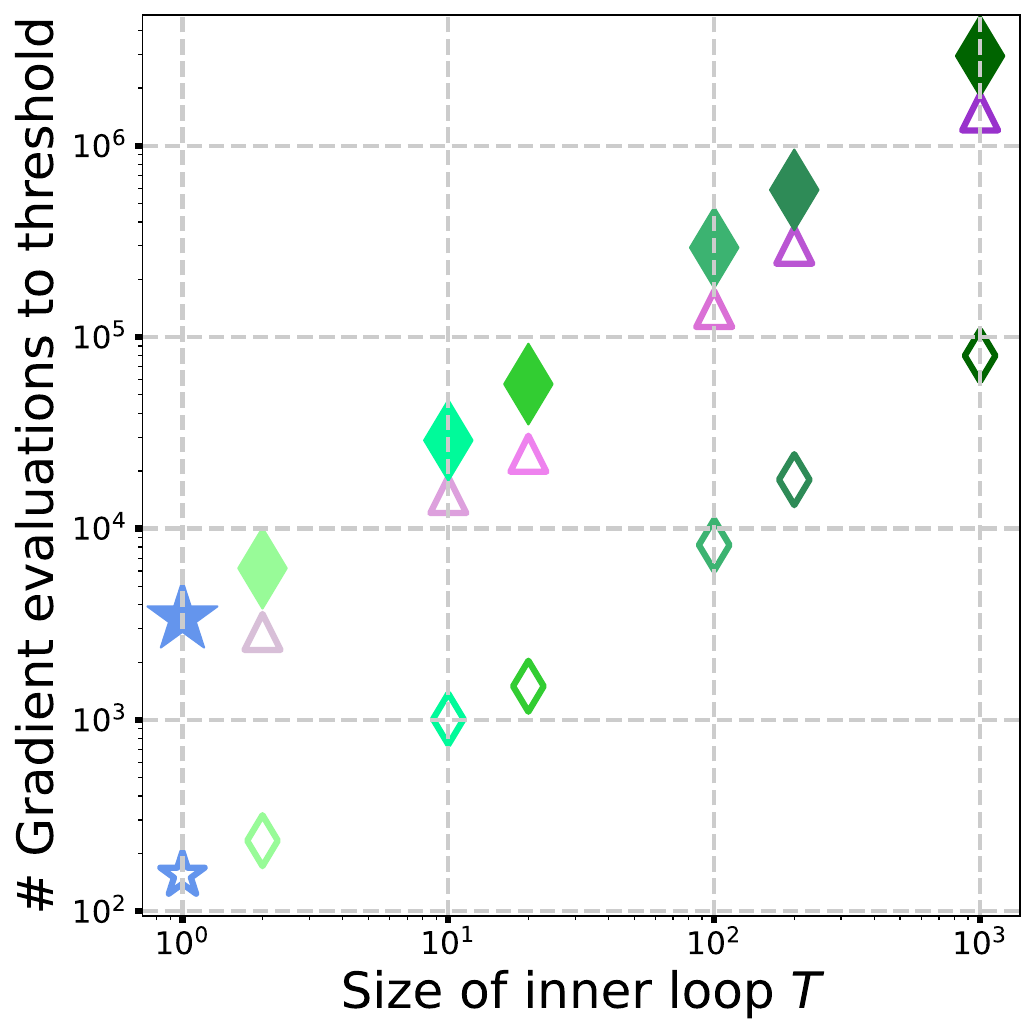}}
\end{subfigure}
\end{center}
  \caption{Metrics for reward training of Langevin processes, 10 runs.  The setting and color code are detailed in  Section~\ref{sec:experiments-langevin}. \textbf{Left:} Reward on the sample distribution, at each outer objective step, averaged on a batch. \textbf{Middle:} Log-likelihood of $\pi^\star(\theta_\text{opt})$ on the sample distribution, at each outer step, averaged on a batch--\underline{higher is better}. \textbf{Right:} Number of gradient evaluations needed to reach a given log-likelihood threshold (y-axis)--\underline{lower is better}, for various sizes of inner loop $T$ (x-axis) and various methods (same symbols and colors as in left plots). Implicit Diffusion is $T=1$. White symbols with colored edges correspond to a log-likelihood threshold of $-5.8$ (red dashed line in the middle plot) and fully-colored symbols to a threshold of $-2.0$ (black dashed line in the middle plot).}
  \label{fig:rewards-langevin}
\end{figure*}
Both qualitatively (Figure~\ref{fig:contour-langevin}) and quantitatively (Figure~\ref{fig:rewards-langevin}), we observe that our approach efficiently optimizes through sampling. We analyze performance both in terms of {\em steps} (number of optimization steps--updates in $\theta$) and {\em gradient evaluations} (number of sampling steps). %
After $K$ optimization steps, our algorithm yields both $\theta_{\text{opt}} := \theta_K$ and a sample $\hat p_K$ approximately from $\pi^\star(\theta_{\text{opt}})$. Then, it is convenient and fast to sample post hoc, with a Langevin process using $\theta_{\text{opt}}$--as observed in Figure~\ref{fig:rewards-langevin}. This is similar in spirit to inference with a finetuned model, post-reward training. We compare our approach with several baselines. First, directly adding a reward term (${\color{myorange} \blacktriangledown}$) %
is less efficient: it tends to overfit on the reward, as the target distribution of this process is out of the family of $\pi^\star(\theta)$'s.
Second, unrolling through the last step of sampling~(${\color{mythistle} \blacktriangle}, {\color{mydarkorchid} \blacktriangle}$) leads to much slower optimization. Finally, using the nested loop approach (${\color{mypalegreen} \blacklozenge }, {\color{mydarkgreen} \blacklozenge }$) makes each optimization step~$T$ times more costly, while barely improving the performance after a given number of steps (due to less biased gradients). 
When comparing in terms of number of gradient evaluations, Implicit Diffusion strongly outperforms the nested loop approach. In other words, for a given computational budget, it is optimal to take $T=1$ and perform more optimization steps, which corresponds to the Implicit Diffusion single-loop approach.
Further comparison plots, as well as a variant of this experiment where we learn a reference distribution (i.e., train from scratch an energy-based model) are included in Appendix~\ref{apx:experimental-details}. We also include an experiment where the potential $V$ parameterizes the means and the covariances of the Gaussians in addition to the weights of the mixture.

\begin{figure}[ht]%
\centering
\includegraphics[width=\columnwidth]{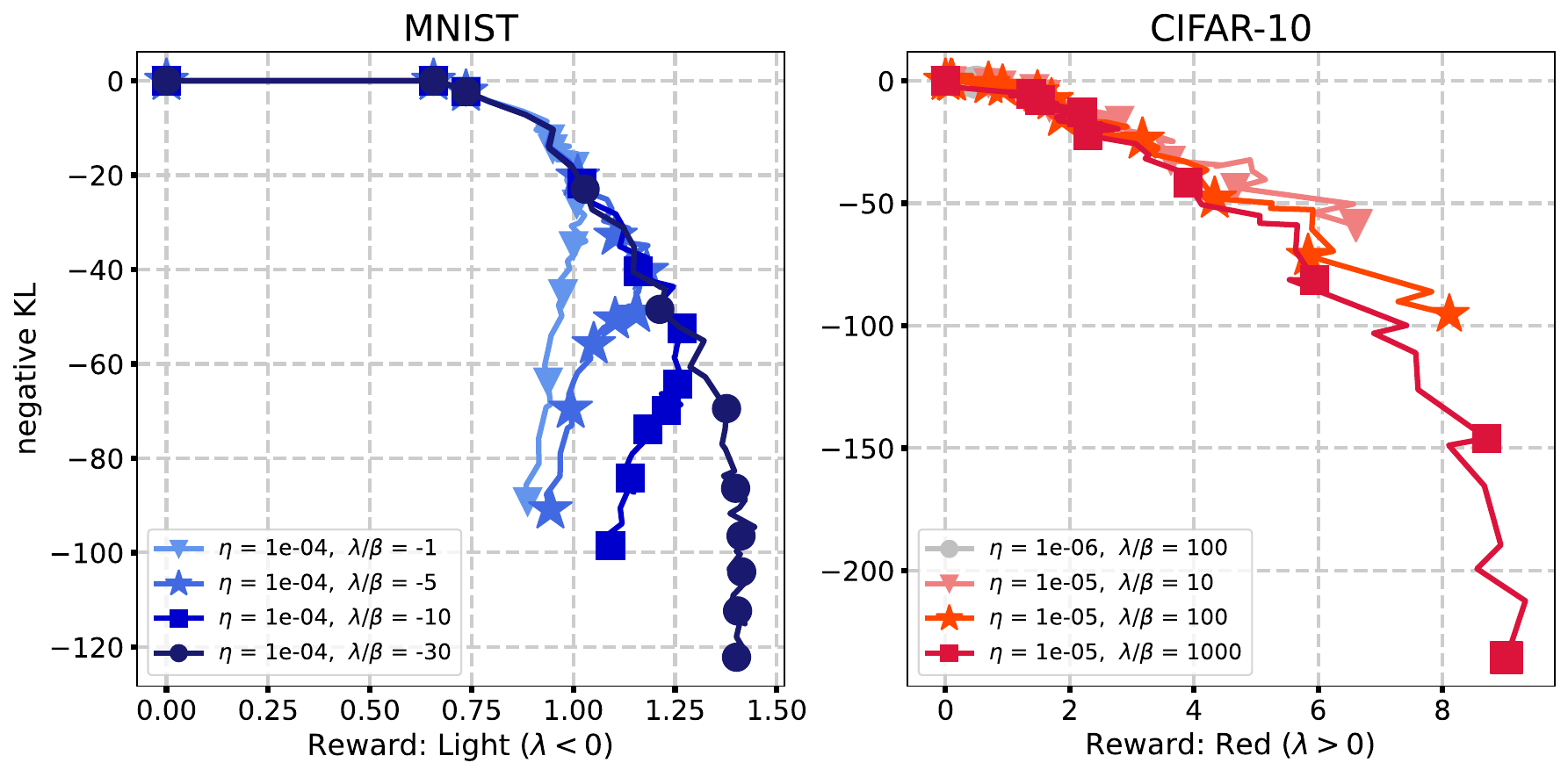}
\caption{Reward training with \textbf{Implicit Diffusion} for various learning rates $\eta$ and reward strengths $\lambda/\beta$. For each dataset, we plot together the reward and the negative KL divergence w.r.t.~$\pi^\star(\theta_0)$.}
\label{fig:metrics-mnist-cifar}
\end{figure}

\subsection{Reward training of denoising diffusion}
We also apply \textbf{Implicit Diffusion} for reward finetuning of denoising diffusion models pretrained on image datasets.  
We denote by $\theta_0$ the weights of a pretrained model, such that $\pi^\star(\theta_0) \approx \pdata$. For various reward functions on the samples $R : \R^d \to \R$, we consider
\[
\cF(p) := - \lambda \E_{x\sim p}[ R(x)] + \beta \KL(p \, || \, \pi^\star(\theta_0))\, ,
\]
common in reward finetuning \citep[see, e.g.,][and references therein]{ziegler2019fine}, for positive and negative values of~$\lambda$. We run \textbf{Implicit Diffusion} using the finite time-horizon variant (Algorithm~\ref{alg:implicit-diff-maxi}), applying the adjoint method on SDEs for gradient estimation. We report selected samples of $\pi^\star(\theta_t)$, as well as reward and KL divergence estimates (see Figures~\ref{fig:intro-examples} and~\ref{fig:metrics-mnist-cifar}--\ref{fig:mnist-neg}). %

\begin{figure}[!ht]
\begin{center}
\centerline{\includegraphics[width=\columnwidth]{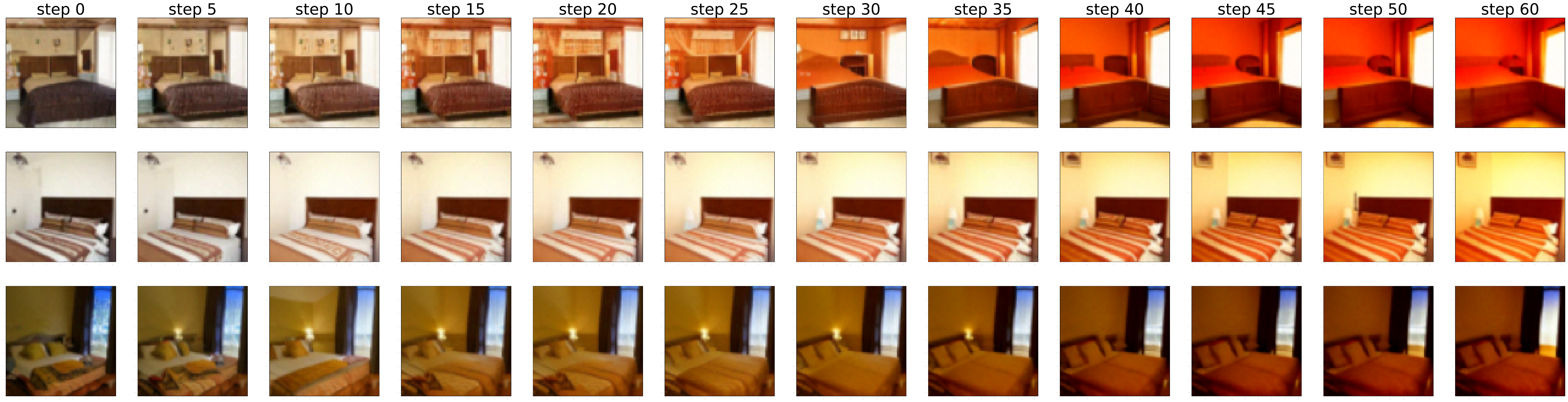}}
\caption{Samples of reward training after pretraining on LSUN ($\lambda / \beta = 10$). The reward incentives for redder images. Images are re-sampled with the same seed every five steps (see Appendix~\ref{apx:exp-denoising}).
\label{fig:lsun-pos}}
\end{center}
\end{figure}

\begin{figure}[!ht]
\begin{center}
\centerline{\includegraphics[width=\columnwidth]{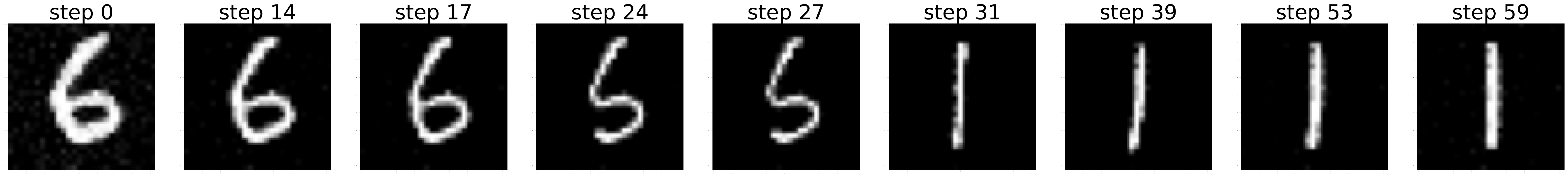}}
\centerline{\includegraphics[width=\columnwidth]{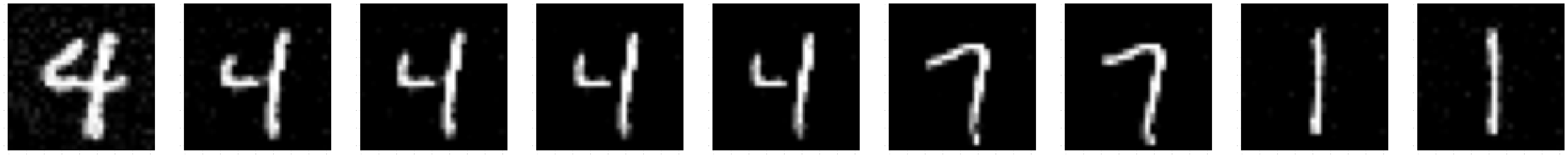}}
\centerline{\includegraphics[width=\columnwidth]{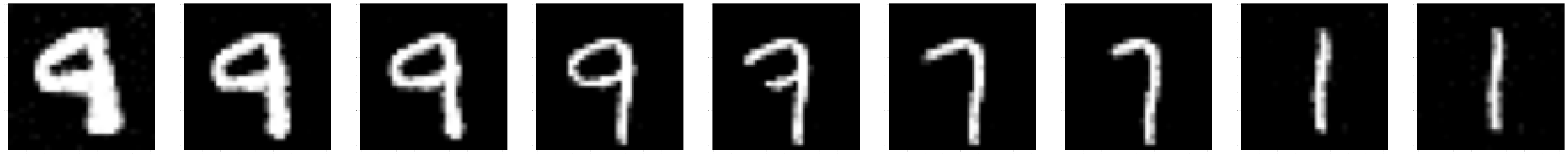}}
\caption{Samples of reward training after pretraining on MNIST. The reward favors \textbf{darker} images ($\lambda / \beta = -30$). Images are re-sampled with the same seed every five steps (see Appendix~\ref{apx:exp-denoising}).
\label{fig:mnist-neg}}
\end{center}
\end{figure}

We report results on models pretrained on the image datasets MNIST \citep{lecun1998}, CIFAR-10 \citep{Krizhevsky2009learningmultiple}, and LSUN (bedrooms) \citep{yu2016lsun}. 
For MNIST, we use a $2.5\si{\mega\nothing}$ parameters model (no label conditioning). Our reward is the average brightness (i.e.~average of all pixel values).
For CIFAR-10 and LSUN, we pretrain a $53.2\si{\mega\nothing}$ parameters model, with label conditioning for CIFAR-10.
Our reward is the average brightness of the red channel minus the average on the other channels. For pretraining, we follow the simple diffusion method \citep{hoogeboom2023simple} and use U-Net models \citep{ronneberger2015u}. We display visual examples in Figures~\ref{fig:intro-examples},~\ref{fig:lsun-pos},~\ref{fig:mnist-neg} and in Appendix~\ref{apx:experimental-details}, where we also report additional metrics. While the finetuned models diverge from the original distribution, they retain overall semantic information (e.g.~brighter digits are thicker, rather than on a gray background). We observe in Figure~\ref{fig:metrics-mnist-cifar} the %
competition between reward and divergence to the pretrained distribution.

Possible limitations of our approach include sensitivity to the choice of hyperparameters (learning rate $\eta$, reward strength $\lambda/\beta$, size of the queue $M$), bias in the gradient estimation, practical applicability to larger scale problems or more complex rewards. We plan to investigate these questions in future research.

\subsection*{Acknowledgments}

The authors would like to thank Fabian Pedregosa for very fruitful discussions on implicit differentiation and bilevel optimization that led to this project, Vincent Roulet for very insightful notes and comments about early drafts of this work as well as help with experiment implementation, Emiel Hoogeboom for extensive help on pretraining diffusion models, and Clément Crépy for help with open-sourcing. PM and AK thank Google for their academic support in the form respectively of a Google PhD Fellowship and a gift in support of her academic research.

\bibliographystyle{abbrvnat}
\bibliography{biblio}

\section*{Checklist}

 \begin{enumerate}

 \item For all models and algorithms presented, check if you include:
 \begin{enumerate}
   \item A clear description of the mathematical setting, assumptions, algorithm, and/or model. \textbf{Yes}
   \item An analysis of the properties and complexity (time, space, sample size) of any algorithm. \textbf{Yes}
   \item (Optional) Anonymized source code, with specification of all dependencies, including external libraries. \textbf{No}
 \end{enumerate}

The setting and algorithms are described in Sections \ref{sec:intro}--\ref{sec:methods}, and further details are given in Section~\ref{apx:algos}. We open-sourced the source code related to the experiments on reward training of Langevin processes.

 \item For any theoretical claim, check if you include:
 \begin{enumerate}
   \item Statements of the full set of assumptions of all theoretical results. \textbf{Yes}
   \item Complete proofs of all theoretical results. \textbf{Yes}
   \item Clear explanations of any assumptions. \textbf{Yes}
 \end{enumerate}

See Section \ref{sec:theory} for precise mathematical statements and assumptions, and Appendix \ref{apx:proofs} for proofs.

 \item For all figures and tables that present empirical results, check if you include:
 \begin{enumerate}
   \item The code, data, and instructions needed to reproduce the main experimental results (either in the supplemental material or as a URL). \textbf{No}
   \item All the training details (e.g., data splits, hyperparameters, how they were chosen). \textbf{Yes}
         \item A clear definition of the specific measure or statistics and error bars (e.g., with respect to the random seed after running experiments multiple times). \textbf{Yes}
         \item A description of the computing infrastructure used. (e.g., type of GPUs, internal cluster, or cloud provider). \textbf{Yes}
 \end{enumerate}

We open-sourced the source code related to the experiments on reward training of Langevin processes. Error bars are provided over independent repetitions for Langevin experiments, as described in Section \ref{sec:experiments} and Appendix \ref{apx:experimental-details}, while they are too costly to compute for the denoising diffusion experiments. The computing infrastructure is described in Section \ref{sec:experiments} and Appendix \ref{apx:experimental-details}.

 \item If you are using existing assets (e.g., code, data, models) or curating/releasing new assets, check if you include:
 \begin{enumerate}
   \item Citations of the creator If your work uses existing assets. \textbf{Yes}
   \item The license information of the assets, if applicable. \textbf{No}
   \item New assets either in the supplemental material or as a URL, if applicable. \textbf{Not Applicable}
   \item Information about consent from data providers/curators. \textbf{Not Applicable}
   \item Discussion of sensible content if applicable, e.g., personally identifiable information or offensive content. \textbf{Not Applicable}
 \end{enumerate}

See Section \ref{sec:experiments} and Appendix \ref{apx:experimental-details} for citations of the datasets and main code packages used in this project.

 \item If you used crowdsourcing or conducted research with human subjects, check if you include:
 \begin{enumerate}
   \item The full text of instructions given to participants and screenshots. \textbf{Not Applicable}
   \item Descriptions of potential participant risks, with links to Institutional Review Board (IRB) approvals if applicable. \textbf{Not Applicable}
   \item The estimated hourly wage paid to participants and the total amount spent on participant compensation. \textbf{Not Applicable}
 \end{enumerate}

 \end{enumerate}

\newpage
\onecolumn

\appendix

\begin{center}
    \Large{\textbf{APPENDIX}}
\end{center}

\paragraph{Organization of the Appendix.} Section \ref{apx:algos} is devoted to explanations of our methodology. Section~\ref{subsec:gradient-estimation-abstraction} explains the gradient estimation setting we consider. Then, in the case of Langevin dynamics (Section \ref{subsec:apx:joint-optimization-algos}), and denoising diffusions (Section \ref{apx:adjoint-method}), we explain how Definition \ref{def:implicit-gradient-estimation} and our Implicit Differentiation algorithms (Algorithms \ref{alg:implicit-diff} and \ref{alg:implicit-diff-maxi}) can be instantiated. Section \ref{apx:implicit-differentiation} gives more details about the implicit differentiation approaches sketched in Section~\ref{sec:grad-sampling}.
Section \ref{apx:proofs} contains the proofs of our theoretical results, while Section \ref{apx:experimental-details} gives details for the experiments of Section \ref{sec:experiments} as well as additional explanations and plots. Finally, Section \ref{apx:additional-related-work} is dedicated to additional related work.

\section{IMPLICIT DIFFUSION ALGORITHMS}  \label{apx:algos}
\subsection{Gradient estimation abstraction: $\Gamma$}  \label{subsec:gradient-estimation-abstraction}
As discussed in Section~\ref{subsec:overview}, we focus on settings where the gradient of the loss $\ell: \R^p \to \R$, defined by
\[
\ell(\theta) := \cF(\pi^\star(\theta))\, ,
\]
can be estimated by using a function $\Gamma$. More precisely, following Definition~\ref{def:implicit-gradient-estimation}, we assume that there exists a function $\Gamma: \cP \times \R^p \to \R^p$ such that $\nabla \ell(\theta) = \Gamma(\pi^\star(\theta), \theta)$. In practice, for almost every setting there is no closed form for $\pi^\star(\theta)$, and even sampling from it can be challenging (e.g.~here, if it is the outcome of infinitely many sampling steps). When we run our algorithms, the dynamic is in practice applied to variables, as discussed in Section \ref{subsec:sampling-optimization}. Using a batch of variables of size $n$, initialized independent with $X_0^{i} \sim p_0$, we have at each step $k$ of joint sampling and optimization a batch of variables forming an empirical measure $\hat p^{(n)}_k$. We consider cases where the operator $\Gamma$ is well-behaved: if $\hat p^{(n)} \approx p$, then $\Gamma(\hat p^{(n)}, \theta) \approx \Gamma(p, \theta)$, and therefore where this finite sample approximation can be used to produce an accurate estimate of~$\nabla \ell(\theta)$.

\subsection{Langevin dynamics}    \label{subsec:apx:joint-optimization-algos}

We explain how to derive the formulas \eqref{eq:gamma-1}--\eqref{eq:gamma-2} for $\Gamma$, and give the sample version of Algorithm \ref{alg:implicit-diff} in these cases.

Recall that the stationary distribution of the dynamics \eqref{eq:langevin} is the Gibbs distribution \eqref{eq:gibbs},
with the normalization factor $Z_\theta = \int \exp(-V(x, \theta)) \ud x$. Assume that the outer objective can be written as the expectation of some (potentially non-differentiable) reward $R$, namely $\cF(p) := -\E_{x\sim p}[R(x)]$. Then our objective is
\[
\ell_{\textnormal{rew}}(\theta) = -\int R(x) \frac{\exp(-V(x, \theta))}{Z_\theta} \ud x \, .
\]
As a consequence,
\begin{align*}
\nabla \ell_{\textnormal{rew}}(\theta) &= \int R(x) \nabla_2 V(x, \theta) \frac{\exp(-V(x, \theta))}{Z_\theta} \ud x \\
&\qquad- \int R(x) \frac{\exp(-V(x, \theta)) \int \nabla_2 V(x', \theta) \exp(-V(x', \theta)) \ud x'}{Z_\theta^2} \ud x \\
&= \E_{X \sim \pi^\star(\theta)}[R(X) \nabla_2 V(X, \theta)] - \E_{X \sim \pi^\star(\theta)}[R(X)] \E_{X \sim \pi^\star(\theta)}[\nabla_2 V(X, \theta)] \\
&= \cov\nolimits_{X \sim \pi^\star(\theta)}[R(X), \nabla_2 V(X, \theta)] \, .
\end{align*}
This computation is sometimes referred to as the REINFORCE trick \citep{williams1992simple}.
This suggests taking 
\[
\Gamma_{\textnormal{rew}}(p, \theta) = \cov\nolimits_{X \sim p}[R(X), \nabla_2 V(X, \theta)] \, .
\]
In this case, the sample version of Algorithm \ref{alg:implicit-diff} is%
\begin{align*}
X_{k+1}^{(i)} &= X_k^{(i)} - \gamma_X \nabla_1 V(X_k^{(i)}, \theta_k) + \sqrt{2\gamma_X} \Delta B_k^{(i)} \, , \quad \text{for all $i \in \{1, \dots, n\}$} \\ %
\theta_{k+1} &= \theta_k - \gamma_\theta \hat \cov[R(X_k^{(i)}), \nabla_2 V(X_k^{(i)}, \theta_k)] \, , %
\end{align*}
where $(\Delta B_k)_{k \geq 0}$ are i.i.d.~standard Gaussian random variables and $\hat \cov$ is the empirical covariance over the sample.

When $\cF(p) := \text{KL}(\pref \,| \, p)$, e.g., when we want to %
regularize towards a reference distribution $p_{\text{ref}}$ with sample access, for $\ell_{\text{ref}}(\theta) = \cF(\pi^\star(\theta))$, we have
\[
\ell_{\text{ref}}(\theta) = \int \log \Big(\frac{\pref[x]}{\pi^\star(\theta)[x]}\Big) \pref[x]\ud x \, ,
\]
thus
\[
\nabla \ell_{\text{ref}}(\theta) = - \int \frac{\partial \pi^\star(\theta)}{\partial \theta}[x] \cdot \frac{1}{\pi^\star(\theta)[x]} \pref[x]\ud x \, .
\]
Leveraging the explicit formula \eqref{eq:gibbs} for $\pi^\star(\theta)$, we obtain
\begin{align*}
\nabla \ell_{\text{ref}}(\theta) 
&= \int \frac{\nabla_2 V(x, \theta) \exp(-V(x, \theta))}{\pi^\star(\theta)[x] Z_\theta} \pref[x] \ud x + \int \frac{\exp(-V(x, \theta)) \nabla_\theta Z_\theta}{\pi^\star(\theta)[x] Z_\theta^2} \pref[x] \ud x \\
&= \int \nabla_2 V(x, \theta) \pref[x] \ud x + \int \frac{\nabla_\theta Z_\theta}{Z_\theta} \pref[x] \ud x \\
&= \E_{X \sim p_{\text{ref}}}[\nabla_2 V(X, \theta)] - \int \nabla_2 V(x, \theta) \frac{\exp(-V(x, \theta))}{Z_\theta} \ud x \\
&= \E_{X \sim p_{\text{ref}}}[\nabla_2 V(X, \theta)] - \E_{X \sim \pi^\star(\theta)}[\nabla_2 V(X, \theta)] \, ,
\end{align*}
where the third equality uses that $\int \pref[x] \ud x = 1$. This suggests taking
\[
\Gamma_{\textnormal{ref}}(p, \theta) = \E_{X \sim p}[\nabla_2 V(X, \theta)] - \E_{X \sim \pi^\star(\theta)}[\nabla_2 V(X, \theta)] \, .
\]
The terms in this gradient can be estimated: for a model $V(\cdot, \theta)$, the gradient function w.r.t $\theta$ can be obtained by automatic differentiation. Samples from $p_{\text{ref}}$ are available by assumption, and samples from $\pi^\star(\theta)$ can be replaced by the $X_k$ in joint optimization as above. We recover the formula for contrastive learning of energy-based model to data from $p_{\text{ref}}$ \citep{gutmann2012noise}. 

This can also be used for finetuning, combining a reward $R$ and a KL term, with $\cF(p) = -\lambda \E_{X \sim p}[R(x)] + \beta\text{KL}(p\, || \, p_{\text{ref}})$. The sample version of Algorithm \ref{alg:implicit-diff} then can be written
\begin{align*}
X_{k+1}^{(i)} &= X_k^{(i)} - \gamma_X \nabla_1 V(X_k^{(i)}, \theta_k) + \sqrt{2\gamma_X} \Delta B_k^{(i)} \quad \text{for all $i \in \{1, \dots, n\}$} \\ 
\theta_{k+1} &= \theta_k - \gamma_\theta \Big[\hat \lambda \cov[R(X_k), \nabla_2 V(X_k, \theta_k)] + \beta\Big(\sum_{j=1}^m \nabla_2 V(\tilde X_k^{(j)}, \theta) - \frac{1}{n} \sum_{i=1}^n \nabla_2 V(X_k^{(i)}, \theta) \Big)\Big]
 , 
\end{align*}
where $(\Delta B_k)_{k \geq 0}$ are i.i.d.~standard Gaussian random variables, $\hat \cov$ is the empirical covariance over the sample, and $\tilde X_k^{(j)} \sim p_{\text{ref}}$.

\subsection{Adjoint method and denoising diffusions}     \label{apx:adjoint-method}
We explain how to use the adjoint method to backpropagate through differential equations, and apply this to derive instantiations of Algorithm \ref{alg:implicit-diff-maxi} for denoising diffusions.

\paragraph{ODE sampling.} We begin by recalling the adjoint method in the ODE case \citep{pontryagin2018mathematical}. Consider the ODE $\ud Y_t = \mu(t, Y_t, \theta) \ud t$ integrated between~$0$ and some $T>0$. For some differentiable function $R: \R^d \to \R$, the derivative of $R(Y_T)$ with respect to~$\theta$ can be computed by the adjoint method. More precisely, it is equal to $G_T$ defined by
\begin{alignat*}{3}
    Z_0 &= Y_T \, , \qquad &&\ud Z_t = - \mu(t, Z_t, \theta) \ud t \, , \\
    A_0 &= \nabla R(Y_T) \, , \qquad &&\ud A_t = A_t^\top \nabla_2 \mu(T-t, Z_t, \theta) \ud t \, , \\
    G_0 &= 0 \, , \qquad &&\ud G_t= A_t^\top \nabla_3 \mu(T-t, Z_t,\theta) \ud t \, .
\end{alignat*}
Note that sometimes the adjoint equations are written with a reversed time index ($t' = T-t$), which is not the formalism we adopt here.

In the setting of denoising diffusion presented in Section \ref{subsec:examples}, we are not interested in computing the derivative of a function of a single realization of the ODE, but of the expectation over $Y_T \sim \pi^\star(\theta)$ of the derivative of $R(Y_T)$ with respect to $\theta$. In other words, we want to compute $\nabla \ell(\theta) = \nabla (\cF \circ \pi^\star)(\theta)$, where $\cF(p) = \E_{x\sim p}[R(x)]$. Rewriting the equations above in this case, we obtain that $G_T$ defined by
\begin{alignat*}{3}
    Z_0 &\sim \pi^\star(\theta) \, , \qquad &&\ud Z_t = - \mu(t, Z_t, \theta) \ud t \, , \\
    A_0 &= \nabla R(Z_0) \, , \qquad &&\ud A_t = A_t^\top \nabla_2 \mu(T-t, Z_t, \theta) \ud t \, , \\
    G_0 &= 0 \, , \qquad &&\ud G_t= A_t^\top \nabla_3 \mu(T-t, Z_t,\theta) \ud t
\end{alignat*}
is an unbiased estimator of $\nabla \ell(\theta)$. Recalling Definition \ref{def:implicit-gradient-estimation}, this means that we can take $\Gamma(p, \theta) := G_T$ defined by
\begin{alignat*}{3}
    Z_0 &\sim p, \quad
    &&\ud Z_t = - \mu(t, Z_t, \theta) \ud t \\
    A_0 &= \nabla R(Z_0), \quad 
     &&\ud A_t = A_t^\top \nabla_2 \mu(T-t, Z_t, \theta) \ud t \\
    G_0 &= 0, \quad
    &&\ud G_t = A_t^\top \nabla_3 \mu(T-t, Z_t,\theta) \ud t \, .
\end{alignat*}
This is exactly the definition of $\Gamma$ given in Section \ref{sec:grad-sampling}. We apply this to the case of denoising diffusions, where is $\mu$ given by \eqref{eq:diffusion-ode-back}. To avoid a notation clash between the number of iterations $T=M$ of the sampling algorithm, and the maximum time $T$ of the ODE in \eqref{eq:diffusion-ode-back}, we rename the latter to $T_\textnormal{horizon}$. A direct instantiation of Algorithm \ref{alg:implicit-diff-maxi} with an Euler solver is the following algorithm.

\begin{algorithm}[H]
\caption{Implicit Diff. optimization, denoising diffusions with ODE sampling} \label{alg:implicit-diff-denoising-diff}
\begin{algorithmic}
\INPUT $\theta_0 \in \R^p$, $p_0 \in \cP$
\INPUT $P_M = [Y_0^{(0)},\ldots, Y_0^{(M)}] \sim \mathcal{N}(0, 1)^{\otimes (m \times d)}$
\FOR{$k \in \{0, \dots, K-1\}$}
    \STATE $Y_{k+1}^{(0)} \sim \mathcal{N}(0, 1)$
    \STATE \textbf{parallel} $Y_{k+1}^{(m+1)} \gets Y_k^{(m)} + \frac{1}{M} \mu(\frac{m T_\textnormal{horizon}}{M}, Y_k^{(m)}, \theta_k)$ for $m \in [M-1]$
    \STATE $Z_k^{(0)} \gets Y_k^{(M)}$
    \STATE $A_k^{(0)} \gets \nabla R (Z_k^{(0)})$
    \STATE $G_k^{(0)} \gets 0$
    \FOR{$t \in \{0, \dots, T-1\}$}
        \STATE $Z_k^{(t+1)} \gets Z_k^{(t)} - \frac{1}{T} \mu(\frac{t T_\textnormal{horizon}}{T}, Z_k^{(t)}, \theta_k)$
        \STATE $A_k^{(t+1)} \gets A_k^{(t)} + \frac{1}{T} (A_k^{(t)})^\top \nabla_2 \mu(\frac{t T_\textnormal{horizon}}{T}, Z_k^{(t)}, \theta_k)$
        \STATE $G_k^{(t+1)} \gets G_k^{(t)} + \frac{1}{T} (G_k^{(t)})^\top \nabla_2 \mu(\frac{t T_\textnormal{horizon}}{T}, Z_k^{(t)}, \theta_k)$
    \ENDFOR
    \STATE $\theta_{k+1} \gets \theta_k - \eta G_k^{(T)}$
\ENDFOR
\OUTPUT $\theta_K$
\end{algorithmic}
\end{algorithm}

Several comments are in order. First, the dynamics of $Y_k^{(M)}$ in the previous $M$ steps, from $Y_{k-M}^{(0)}$  to $Y_{k-1}^{(M-1)}$, uses the $M$ previous values of the parameter $\theta_{k-M}, \ldots ,\theta_{k-1}$.
This means that $Y_k^{(M)}$ does not correspond to the result of sampling with any given parameter $\theta$, since we are at the same time performing the sampling process and updating $\theta$.

Besides, the computation of $\Gamma(p, \theta)$ is the outcome of an iterative process, namely calling an ODE solver. Therefore, it is also possible to use the same queuing trick as for sampling iterations to decrease the cost of this step by leveraging parallelization. For completeness, the variant is given below.

\begin{algorithm}[H]
\caption{Implicit Diff. optimization, denoising diffusions with ODE sampling, variant with a double queue} \label{alg:implicit-diff-denoising-diff-double-queue}
\begin{algorithmic}
\INPUT $\theta_0 \in \R^p$, $p_0 \in \cP$
\INPUT $P_M = [Y_0^{(0)},\ldots, Y_0^{(M)}] \sim \mathcal{N}(0, 1)^{\otimes (m \times d)}$
\FOR{$k \in \{0, \dots, K-1\}$ (joint single loop)}
    \STATE $Y_{k+1}^{(0)} \sim \mathcal{N}(0, 1)$
    \STATE $Z_{k+1}^{(0)} \gets Y_k^{(M)}$
    \STATE $A_{k+1}^{(0)} \gets \nabla R (Z_{k+1}^{(0)})$
    \STATE $G_{k+1}^{(0)} \gets 0$
    \STATE \textbf{parallel} $Y_{k+1}^{(m+1)} \gets Y_k^{(m)} + \frac{1}{M} \mu(\frac{m T_\textnormal{horizon}}{M}, Y_k^{(m)}, \theta_k)$ for $m \in [M-1]$
    \STATE \textbf{parallel}  $Z_{k+1}^{(m+1)} \gets Z_k^{(m)} - \frac{1}{M} \mu(\frac{m T_\textnormal{horizon}}{M}, Z_k^{(m)}, \theta_k)$ for $m \in [M-1]$
    \STATE \textbf{parallel}  $A_{k+1}^{(m+1)} \gets A_k^{(m)} + \frac{1}{M} (A_k^{(m)})^\top \nabla_2 \mu(\frac{m T_\textnormal{horizon}}{M}, Z_k^{(m)}, \theta_k)$ for $m \in [M-1]$
    \STATE \textbf{parallel}  $G_{k+1}^{(m+1)} \gets G_k^{(m)} + \frac{1}{M} (G_k^{(m)})^\top \nabla_2 \mu(\frac{m T_\textnormal{horizon}}{M}, Z_k^{(m)}, \theta_k)$ for $m \in [M-1]$
    \STATE $\theta_{k+1} \gets \theta_k - \eta G_{k}^{(M)}$
\ENDFOR
\OUTPUT $\theta_K$
\end{algorithmic}
\end{algorithm}

Second, each variable $Y_k^{(m)}$ consists of a single sample of $\R^d$. The algorithm straightforwardly extends when each variable $Y_k^{(m)}$ is a batch of samples. Finally, we consider so far the case where the size of the queue $M$ is equal to the number of sampling steps $T$. We give below the variant of Algorithm \ref{alg:implicit-diff-denoising-diff} when $M \neq T$ but $M$ divides $T$. Taking $M$ from $1$ to $T$ balances between a single-loop and a nested-loop algorithm.

\begin{algorithm}[H]
\caption{Implicit Diff. optimization, denoising diffusions with ODE sampling, $M \neq T$, $M$ divides $T$} \label{alg:implicit-diff-denoising-diff-M-neq-T}
\begin{algorithmic}
\INPUT $\theta_0 \in \R^p$, $p_0 \in \cP$
\INPUT $P_M = [Y_0^{(0)},\ldots, Y_0^{(M)}] \sim \mathcal{N}(0, 1)^{\otimes (m \times d)}$
\FOR{$k \in \{0, \dots, K-1\}$}
    \STATE $Y_{k+1}^{(0)} \sim \mathcal{N}(0, 1)$
    \STATE \textbf{parallel} $Y_{k+1/2}^{(m+1)} \gets Y_k^{(m)}$ for $m \in [M-1]$
    \FOR{$t \in \{0, \dots, T/M-1\}$ \textbf{in parallel for $m \in [M-1]$}}
        \STATE $Y_{k+1/2}^{(m+1)} \gets Y_{k+1/2}^{(m)} + \frac{1}{T} \mu((\frac{m}{M} + \frac{t}{T}) T_\textnormal{horizon}, Y_{k+1/2}^{(m)}, \theta_k)$
    \ENDFOR
    \STATE \textbf{parallel} $Y_{k+1}^{(m+1)} \gets Y_{k+1/2}^{(m+1)}$ for $m \in [M-1]$
    \STATE $Z_k^{(0)} \gets Y_k^{(M)}$
    \STATE $A_k^{(0)} \gets \nabla R (Z_k^{(0)})$
    \STATE $G_k^{(0)} \gets 0$
    \FOR{$t \in \{0, \dots, T-1\}$}
        \STATE $Z_k^{(t+1)} \gets Z_k^{(t)} - \frac{1}{T} \mu(\frac{t T_\textnormal{horizon}}{T}, Z_k^{(t)}, \theta_k)$
        \STATE $A_k^{(t+1)} \gets A_k^{(t)} + \frac{1}{T} (A_k^{(t)})^\top \nabla_2 \mu(\frac{t T_\textnormal{horizon}}{T}, Z_k^{(t)}, \theta_k)$
        \STATE $G_k^{(t+1)} \gets G_k^{(t)} + \frac{1}{T} (G_k^{(t)})^\top \nabla_2 \mu(\frac{t T_\textnormal{horizon}}{T}, Z_k^{(t)}, \theta_k)$
    \ENDFOR
    \STATE $\theta_{k+1} \gets \theta_k - \eta G_k^{(T)}$
\ENDFOR
\OUTPUT $\theta_K$
\end{algorithmic}
\end{algorithm}

Algorithm \ref{alg:implicit-diff-denoising-diff-double-queue} extends to this case similarly.

\paragraph{SDE sampling.} The adjoint method is also defined in the SDE case \citep{li2020scalable}.
Consider the SDE
\begin{equation}    \label{eq:general-sde-sampling}
\ud Y_t = \mu(t, Y_t, \theta) \ud t + \sqrt{2} \ud B_t \, ,   
\end{equation}
integrated between $0$ and some $T>0$. This setting encompasses the denoising diffusion SDE~\eqref{eq:diffusionback} with the appropriate choice of $\mu$. For some differentiable function $R: \R^d \to \R$ and for a \textit{given realization} $(Y_t)_{0 \leq t \leq T}$ of the SDE, the derivative of $R(Y_T)$ with respect to~$\theta$ is equal to $G_T$ defined by
\begin{equation}    \label{eq:adjoint-sde}
\begin{alignedat}{3}
    A_0 &= \nabla R(Y_T) \, , \qquad &&\ud A_t = A_t^\top \nabla_2 \mu(T-t, Y_{T-t}, \theta) \ud t \, , \\
    G_0 &= 0 \, , \qquad &&\ud G_t= A_t^\top \nabla_3 \mu(T-t, Y_{T-t},\theta) \ud t \, .
\end{alignedat}
\end{equation}
This is a similar equation as in the ODE case. The main difference is that it is not possible to recover $Y_{T-t}$ only from the terminal value of the path $Y_T$, but that we need to keep track of the randomness from the Brownian motion $B_t$. Efficient ways to do so are presented in \cite{li2020scalable,kidger2021efficient}. In a nutshell, they consist in only keeping in memory the seed used to generate the Brownian motion, and recomputing the path from the seed. 

Using the SDE sampler allows us to incorporate a KL term in the reward. Indeed, consider the SDE \eqref{eq:general-sde-sampling} for two different parameters $\theta_1$ and $\theta_2$, with associated variables $Y_t^1$ and $Y_t^2$. Then, by Girsanov's theorem (see \citealp[][Chapter III.8]{protter2005stochastic}, and \citealp{tzen2019theoretical} for use in a similar context), the KL divergence between the paths $Y_t^1$ and $Y_t^2$ is 
\begin{equation}    \label{eq:kl-girsanov}
\KL((Y_t^1)_{t \geq 0} \, || \, (Y_t^2)_{t \geq 0}) = \int_0^T \E_{y \sim q_t^1}\|\mu(t, y, \theta_1) - \mu(t, y, \theta_2)\|^2 \ud t \, ,    
\end{equation}
where $q_t^1$ denotes the distribution of $Y_t^1$. This term can be (stochastically) estimated at the same time as the SDE \eqref{eq:general-sde-sampling} is simulated, by appending a new coordinate to $Y_t$ (and to $\mu$) that integrates \eqref{eq:kl-girsanov} over time. Then, adding the KL in the reward is as simple as adding a linear term in $\tilde{R}: \R^{d+1} \to \R$, that is, $\tilde{R}(x) = R(x[:-1]) + x[-1]$ (using Numpy notation, where `$-1$' denotes the last index). The same idea is used in \cite{dvijotham2023algorithms} to incorporate a KL term in reward finetuning of denoising diffusion models. Finally, note that, if we had at our disposal a reward $R$ that indicates if an image is ``close'' to $p_{\textnormal{data}}$ (for instance implemented by a neural network), we could use our algorithm to train from scratch a denoising diffusion.

\subsection{Implicit differentiation}   \label{apx:implicit-differentiation}

\paragraph{Finite dimension.}
Take $g: \R^m \times \R^p \to \R$ a continuously differentiable function. Then $x^\star(\theta) = \argmin g(\cdot, \theta)$ %
implies a stationary point condition $\nabla_1 g(x^\star(\theta_0), \theta_0) = 0$. In this case, it is possible to define and analyze the function $x^\star: \R^p \to \R^m$ and its variations. Note that this generalizes to the case where $x^\star(\theta)$ can be written as the root of a parameterized system.

More precisely, the \textbf{implicit function theorem}
\citep[see, e.g.,][and references therein]{griewank_2008,krantz2002implicit} can be applied. Under differentiability assumptions on $g$, for $(x_0, \theta_0)$ such that $\nabla_1 g(x_0,
\theta_0) = 0$ with a continuously differentiable $\nabla_1 g$, and if the Hessian
$\nabla_{1, 1} g$ evaluated at $(x_0, \theta_0)$ is a square invertible matrix, then there exists a function $x^\star(\cdot)$ over a neighborhood of
$\theta_0$ satisfying $x^\star(\theta_0) = x_0$. Furthermore, for all $\theta$ in this
neighborhood, we have that $\nabla_1 g(x^\star(\theta), \theta) = 0$ and
its Jacobian $\partial x^\star(\theta)$ exists. It is then possible to differentiate with respect to $\theta$ both sides of the equation $\nabla_1 g(x^\star(\theta_0), \theta_0) = 0$, which yields a linear equation satisfied by this Jacobian
\[
\nabla_{1, 1} g(x^\star(\theta_0), \theta_0) \partial x^\star(\theta) + \nabla_{1, 2} g(x^\star(\theta_0), \theta_0) = 0\, .
\]
This formula can be used for automatic implicit differentiation, when both the evaluation of the derivatives in this equation and the inversion of the linear system can be done automatically \cite{blondel2022efficient}.

\paragraph{Extension to space of probabilities.}
When $\cG: \cP \times \R^p \to \R$ and $\pi^\star(\theta) = \argmin \cG(\cdot, \theta)$ as in \eqref{eq:inner-sampling}, under assumptions on differentiability and uniqueness of the solution on $\cG$, this can also be extended to a distribution setting. We write here the infinite-dimensional equivalent of the above equations, involving derivatives or variations over the space of probabilities, and refer to \cite{ambrosio2005gradient} for more details.

First, we have that
\[
\nabla \ell(\theta) = \nabla_\theta \big(\cF(\pi^\star(\theta)\big) = \int \cF'(p)[x] \nabla_\theta \pi^\star(\theta)[x] \ud x \, , 
\]
where $\cF'(p):\cX \to  \R$ denotes the first variation of $\cF$ at $p\in \cP$ (see \Cref{def:first_var}). This yields $
\nabla \ell(\theta)
= \Gamma(\pi^\star(\theta), \theta)$ with %
\[
\Gamma(p, \theta) = \int \cF'(p)[x] \gamma(p, \theta)[x] \ud x \, .
\]
where $\gamma(p, \theta)$ is the solution of the linear system
\[
\int \nabla_{1, 1} \cG(p, \theta)[x, x'] \gamma(p, \theta)[x'] \ud x' = - \nabla_{1, 2} \cG(p, \theta) [x]\, ,
\]
Although this gives us a general way to define gradients of $\pi^\star(\theta)$ with respect to $\theta$, solving this linear system is generally not feasible. One exception is when sampling over a finite state space $\cX$, in which case~$\cP$ is finite-dimensional, and the integrals boil down to matrix-vector products. 

\section{THEORETICAL ANALYSIS}  \label{apx:proofs}

\subsection{Langevin with continuous flow} 

\subsubsection{Additional definitions}\label{sec:details_math}
\textbf{Notations.} %
We denote by 
$\cP_2(\R^d)$ the set of probability measures on $\R^d$ with bounded second moments. 
Given a Lebesgue measurable map $ T: X\to X$ and $\mu\in \cP_2(X)$, $ T_{\#}\mu$ is the pushforward measure of $\mu$ by $T$.
 For any $\mu \in \cP_2(\X)$, $L^2(\mu)$ is the space of functions $f : \X \to \R$ such that $\int \|f\|^2 d\mu < \infty$. 
 We denote by $\Vert \cdot \Vert_{L^2(\mu)}$ and $\ps{\cdot,\cdot}_{L^2(\mu)}$ respectively the norm and the inner product of the Hilbert space $L^2(\mu)$. 
We consider, for $\mu,\nu \in \cP_2(\X)$, the 2-Wasserstein distance $W_2 (\mu, \nu) = \inf_{s \in \mathcal{S}(\mu,\nu)} \int \|x-y\|^2 ds(x,y)$, where $\mathcal{S}(\mu,\nu)$ is the set of couplings between $\mu$ and $\nu$. The metric space
$(\cP_2(\X), W_2)$ is called the Wasserstein space.

Let $\cF:\cP(\R^d) \to \R^+$ a functional. 
\begin{definition}\label{def:first_var}
Fix $\nu\in\cP(\R^d)$. If it exists, the \textit{first variation of $\cF$ at $\nu$} is the  function $ \cF'(\nu):\X \rightarrow \R$ s. t. for any $\mu \in \mathcal{P}(\X)$, with $\xi = \mu-\nu$:
\begin{equation*}
\lim_{\epsilon \rightarrow 0}\frac{1}{\epsilon}(\cF(\nu+\epsilon  \xi) -\cF(\nu))%
=\int_{\R^d} \cF'(\nu)(x)d \xi(x),
\end{equation*} 
and is defined uniquely up to an additive constant.
\end{definition}
We will extensively apply the following formula:
\begin{equation}    \label{eq:chain-rule}
    \frac{\ud \cF(p_t)}{\ud t}= \int \cF'(p_t)\frac{\partial p_t}{\partial t} = \int \cF'(p_t)[x]\frac{\partial p_t[x]}{\partial t}\ud x.
\end{equation}

We will also rely regularly on the definition of a Wasserstein gradient flow, since Langevin dynamics correspond to a Wasserstein gradient flow of the Kullback-Leibler (KL) divergence \cite{jordan1998variational}. A Wasserstein gradient flow of $\cF$ \cite{ambrosio2005gradient} can be described by the following continuity equation: 
\begin{equation}\label{eq:wgf}
    \frac{\partial \mu_t}{\partial t} = \div(\mu_t \nabla_{W_2}\cF(\mu_t)),\quad\nabla_{W_2}\cF(\mu_t)=\nabla \cF'(\mu_t),
\end{equation}
where $\cF'$ denotes the first variation. Equation \eqref{eq:wgf} holds in the sense of distributions (i.e. the equation above holds when integrated against a smooth function with compact support), see \cite[][Chapter 8]{ambrosio2005gradient}. In particular, if $\cF=\KL(\cdot|\pi)$ for $\pi \in \cP_2(\R^d)$, then $\nabla_{W_2}\cF(\mu) = \nabla \log(\nicefrac{\mu}{\pi})$. In this case, the corresponding continuity equation is known as the Fokker-Planck equation, and in particular it is known that the law $p_t$ of Langevin dynamics:
\begin{equation*}
    \ud X_t = \nabla \log(\pi(X_t)) \ud t+\sqrt{2}\ud B_t
\end{equation*}
satisfies the Fokker-Planck equation \citep[Chapter 3]{pavliotis2016stochastic}.

\subsubsection{Gaussian mixtures satisfy the Assumptions}     \label{apx:gaussian-mixture}

We begin by a more formal statement of the result alluded to in Section \ref{subsec:theory-langevin-continuous}. 
\begin{proposition}
\label{pro:gaussian-mixture-2}
Let
\begin{equation*}   %
V(x, \theta) := - \log \Big( \sum_{i=1}^p H(\theta_i) \exp(-\|x-z_i\|^2)\Big),    
\end{equation*}
for some fixed $z_1, \dots, z_p \in \R^d$ and where  
$$H(x) := \eta + (1 - \eta) \cdot \frac{1}{1 + e^{-x}}$$
is a shifted version of the logistic function for some $\eta \in (0, 1)$.
Then Assumptions \ref{ass:potential_log_sobolev} and \ref{ass:gradient_bounded} hold. 
\end{proposition}

\begin{proof}
    Assumption \ref{ass:potential_log_sobolev} holds since a mixture of Gaussians is Log-Sobolev with a bounded constant \citep[Corollary~1]{chen2021dimension}. Note that the constant deteriorates as the modes of the mixture get further apart.

Furthermore, Assumption \ref{ass:gradient_bounded} holds since, for all $\theta \in \R^p$ and $x \in \R^d$,
\[
\|\nabla_2 V(x, \theta)\|_1 = \frac{\sum_{i=1}^p H'(\theta_i) \exp(-\|x-z_i\|^2)}{\sum_{i=1}^p H(\theta_i) \exp(-\|x-z_i\|^2)} \leq \frac{\sum_{i=1}^p \exp(-\|x-z_i\|^2)}{\sum_{i=1}^p \eta \exp(-\|x-z_i\|^2)} = \frac{1}{\eta} \, .
\]
\end{proof}

\subsubsection{Proof of Proposition \ref{prop:verif-ass-gamma-lip}}
In the case of the functions $\Gamma$ defined by \eqref{eq:gamma-1}--\eqref{eq:gamma-2}, we see that $\Gamma$ is bounded under Assumption \ref{ass:gradient_bounded} and when the reward $R$ is bounded. The Lipschitz continuity can be obtained as follows. Consider for instance the case of \eqref{eq:gamma-2} where $\Gamma(p, \theta)$ is given by
\[
\Gamma_{\textnormal{ref}}(p, \theta) = \E_{X \sim p_{\text{ref}}}[\nabla_2 V(X, \theta)] - \E_{X \sim p}[\nabla_2 V(X, \theta)] \, .
\]
Then
\begin{align*}
\|\Gamma_{\textnormal{ref}}(p, \theta) - \Gamma_{\textnormal{ref}}(q, \theta)\|
&= \|\E_{X \sim q}[\nabla_2 V(X, \theta)] - \E_{X \sim p}[\nabla_2 V(X, \theta)]\| \\
&\leq C \TV(p, q) \\
&\leq \frac{C}{\sqrt{2}} \sqrt{\KL(p || q)} \, ,
\end{align*}
where the first inequality comes from the fact that the total variation distance is an integral probability metric generated by the set of bounded functions, and the second inequality is Pinsker's inequality \citep[][Lemma 2.5]{tsybakov2009introduction}. 
The first case of Section \ref{subsec:apx:joint-optimization-algos} unfolds similarly.

\subsubsection{Proof of Theorem \ref{thm:langevin-continuous}}
The dynamics \eqref{eq:langevin-theory} can be rewritten equivalently on $\cP_2(\R^d)$ and $\R^p$ as
\begin{align}
    \frac{\partial p_t}{\partial t} &= \div(p_t \nabla_{W_2} \cG(p_t, \theta_t)) \label{eq:langevin-inner-proof} \\
    \ud \theta_t &= - \varepsilon_t \Gamma (p_t, \theta_t) \ud t \, ,  \label{eq:langevin-outer-proof}
\end{align}
where $\cG(p, \theta) = \KL(p \, || \, \pi_\theta^*)$, see \Cref{sec:details_math}.
The Wasserstein gradient in \eqref{eq:langevin-inner-proof} is taken with respect to the first variable of $\cG$. 

\paragraph{Evolution of the loss.} Recall that $\Gamma$ satisfies by Definition \ref{def:implicit-gradient-estimation} that $\nabla \ell(\theta) =\Gamma (\pi^\star(\theta), \theta)$. Thus we have, by \eqref{eq:langevin-outer-proof},
\begin{align*}
    \frac{\ud \ell}{\ud t}(t) 
    &= \Big\langle \nabla \ell(\theta_t), \frac{\ud\theta_t}{\ud t} \Big\rangle \\
    &= - \varepsilon_t \langle \nabla \ell(\theta_t), \Gamma (p_t, \theta_t) \rangle \\
    &= - \varepsilon_t \langle \nabla \ell(\theta_t), \Gamma (\pi^\star(\theta_t), \theta_t) \rangle + \varepsilon_t \langle \nabla \ell(\theta(t)), \Gamma (\pi^\star(\theta_t), \theta_t) - \Gamma (p_t, \theta_t) \rangle \\
    &\leq - \varepsilon_t \|\nabla \ell(\theta_t)\|^2 + \varepsilon_t \|\nabla \ell(\theta_t)\| \|\Gamma (\pi^\star(\theta_t), \theta_t) - \Gamma (p_t, \theta_t)\| \, .
\end{align*} 
Then, by \Cref{ass:Gamma_Lipschitz},
\begin{equation*}
\frac{\ud \ell}{\ud t}(t) \leq - \varepsilon_t \|\nabla \ell(\theta_t)\|^2 + \varepsilon_t K_\Gamma \|\nabla \ell(\theta_t)\| \sqrt{\KL(p_t \, || \, \pi^\star(\theta_t) )} \, .
\end{equation*}
Using $ab \leq \frac{1}{2}(a^2 + b^2)$, we get
\begin{equation} \label{eq:evolution-loss-langevin}
\frac{\ud \ell}{\ud t}(t) \leq - \frac{1}{2} \varepsilon_t \|\nabla \ell(\theta_t)\|^2 + \frac{1}{2} \varepsilon_t K_\Gamma^2 \KL(p_t \, || \, \pi^\star(\theta_t)) \, ,
\end{equation}

\paragraph{Bounding the KL divergence of $p_t$ from $\pi^*(\theta_t)$.} Recall that
\[
\KL(p_t\,||\,\pi^\star(\theta_t)) = \int \log\left(\frac{p_t}{\pi^\star(\theta_t)}\right) p_t \, .
\]
Thus, by the chain rule formula \eqref{eq:chain-rule},
\begin{equation*}
    \frac{\ud\KL(p_t\,||\,\pi^\star(\theta_t))}{\ud t}=\int \log\left(\frac{p_t}{\pi^\star(\theta_t)}\right) \frac{\partial p_t}{\partial t} - \int \frac{p_t}{\pi^\star(\theta_t)}\frac{\partial \pi^\star(\theta_t)}{\partial t}:= a - b \, .
\end{equation*}
From an integration by parts, using \eqref{eq:langevin-inner-proof} and by \Cref{ass:potential_log_sobolev}, we have 
\begin{multline*}
    a= \int \log\left(\frac{p_t}{\pi^\star(\theta_t)}\right) \frac{\partial p_t}{\partial t} =\int \log\left(\frac{p_t}{\pi^\star(\theta_t)}\right) \div(p_t \nabla \log\left(\frac{p_t}{\pi^\star(\theta_t)} \right))\\
   = \left\langle \nabla \log\left(\frac{p_t}{\pi^\star(\theta_t)} \right), -\nabla \log\left(\frac{p_t}{\pi^\star(\theta_t)} \right) \right \rangle_{L^2(p_t)} = - \left \Vert \nabla \log\left(\frac{p_t}{\pi^\star(\theta_t)} \right)  \right \Vert^2_{L^2(p_t)} \\
   \le - 2\mu \KL(p_t \, || \, \pi^\star(\theta_t)) \, .
\end{multline*}
Moving on to $b$, we have
\begin{equation*}
    b = \int \frac{p_t}{\pi^\star(\theta_t)}\frac{\partial \pi^\star(\theta_t)}{\partial t} = \int p_t \frac{\partial \log(\pi^\star(\theta_t))}{\partial t} \, .
\end{equation*}
By the chain rule and \eqref{eq:langevin-outer-proof}, we have for $x \in \cX$
\[
    \frac{\partial \pi^\star(\theta_t)}{\partial t}[x] 
    = \left \langle \frac{\partial \pi^\star(\theta_t)}{\partial \theta}[x],\frac{d\theta_t}{\ud t}\right \rangle
    = \left \langle \frac{\partial \pi^\star(\theta_t)}{\partial \theta}[x],-\varepsilon_t \Gamma(p_t, \theta_t)\right \rangle \, .
\]
Using $\pi^\star(\theta)\propto e^{-V(\theta,\cdot)}$ (with similar computations as for $\nabla \ell_{\text{ref}}$ in Section \ref{subsec:apx:joint-optimization-algos}), we have
\[
\frac{\partial \pi^\star(\theta_t)}{\partial t}[x] =  \left  \langle -\nabla_2 V(x, \theta_t)\pi^\star(\theta_t)[x] + \E_{X \sim \pi^\star(\theta_t)}(\nabla_2 V(X, \theta_t))\pi^\star(\theta_t)[x],-\varepsilon_t \Gamma(p_t, \theta_t)\right \rangle, 
\]
and
\[
\frac{\partial \log(\pi^\star(\theta_t)) }{\partial t} = \varepsilon_t\ps{\nabla_2 V(\cdot, \theta_t) - \E_{X \sim \pi^\star(\theta_t)}(\nabla_2 V(X, \theta_t)), \Gamma(p_t, \theta_t)}.
\]
This yields 
\begin{align*}
    |b| &=  \bigg| \varepsilon_t \int \ps{\nabla_2 V(x, \theta_t)- \E_{X \sim \pi^\star(\theta_t)}(\nabla_2 V(X, \theta_t)), \Gamma(p_t, \theta_t)} p_t[x]\ud x \bigg| \\
    &=  \varepsilon_t \bigg|\left \langle \int \nabla_2 V(x, \theta_t) dp_t[x] - \E_{X \sim \pi^\star(\theta_t)}(\nabla_2 V(X, \theta_t)),\Gamma(p_t, \theta_t)\right \rangle\bigg| \\
    &\le \varepsilon_t \| \Gamma(p_t, \theta_t)\|_{\R^p} \big(\|\nabla_2 V(\cdot, \theta_t)\|_{L^2(p_t)} + \|\nabla_2 V(\cdot, \theta_t)\|_{L^2(\pi^\star(\theta_t))}\big) \\
    &\leq 2C^2 \varepsilon_t \, ,
\end{align*}
where the last step uses Assumptions \ref{ass:gradient_bounded} and \ref{ass:Gamma_Lipschitz}. Putting everything together, we obtain
\[
\frac{d\KL(p_t \, || \, \pi^\star(\theta_t))}{\ud t} \leq - 2\mu \KL(p_t \, || \, \pi^\star(\theta_t)) + 2 C^2 \varepsilon_t.
\]
Using Grönwall's inequality \citep{pachpatte1997inequalities} to integrate the inequality, the KL divergence can be bounded by
\[
\KL(p_t \, || \, \pi^\star(\theta_t)) \leq \KL(p_0 \, || \, \pi^\star(\theta_0)) e^{-2\mu t} + 2 C^2 \int_0^t \varepsilon_s e^{2 \mu (s-t)} \ud s \, .
\]
Coming back to \eqref{eq:evolution-loss-langevin}, we get
\[
\frac{\ud \ell}{\ud t}(t) \leq - \frac{1}{2} \varepsilon_t \|\nabla \ell(\theta_t)\|^2 + \frac{1}{2} \varepsilon_t K_\Gamma^2 \Big(\KL(p_0 \, || \, \pi^\star(\theta_0)) e^{-2\mu t} + 2 C^2 \int_0^t \varepsilon_s e^{2 \mu (s-t)} \ud s\Big) \, .
\]
Integrating between $0$ and $T$, we have
\begin{align*}
\ell(T) - \ell(0) &\leq  -\frac{1}{2} \int_0^T \varepsilon_t \|\nabla \ell(\theta_t)\|^2 \ud t + \frac{K_\Gamma^2 \KL(p_0 \, || \, \pi^\star(\theta_0))}{2} \int_0^T \varepsilon_t e^{-2\mu t} \ud t \\
&\qquad+ K_\Gamma^2 C^2 \int_0^T \int_0^t \varepsilon_t \varepsilon_s e^{2 \mu (s-t)} \ud s \ud t \, .
\end{align*}
Since $\varepsilon_t$ is decreasing, we can bound $\varepsilon_t$ by $\varepsilon_T$ in the first integral and rearrange terms to obtain
\begin{align}    \label{eq:upper-bound-grads}
\begin{split}
\frac{1}{T} \int_0^T \|\nabla \ell(\theta_t)\|^2 \ud t &\leq \frac{2}{T \varepsilon_T} (\ell(0) - \inf \ell) + \frac{K_\Gamma^2 \KL(p_0 \, || \, \pi^\star(\theta_0))}{T \varepsilon_T} \int_0^T \varepsilon_t e^{-2\mu t} \ud t  \\
&\qquad+ \frac{2 K_\Gamma^2 C^2}{T \varepsilon_T} \int_0^T \int_0^t \varepsilon_t \varepsilon_s e^{2 \mu (s-t)} \ud s \ud t  \, .     
\end{split}
\end{align}
Recall that, by assumption of the Theorem, $\varepsilon_t = \min(1, \frac{1}{\sqrt{t}})$. Thus $T \varepsilon_T = \sqrt{T}$, and the first term is bounded by a constant times $T^{-1/2}$. It is also the case of the second term since $\int_0^T \varepsilon_t e^{-2\mu t} \ud t$ is converging. 
Let us now estimate the magnitude of the last term. Let $T_0 \geq 2$ (depending only on $\mu$) such that $\frac{\ln(T_0)}{2 \mu} \leq \frac{T_0}{2}$. For $t \geq T_0$, let $\alpha(t) := t - \frac{\ln t}{2 \mu}$. We have, for $t \geq T_0$,
\begin{align*}
\int_0^t \varepsilon_s e^{2 \mu (s-t)}ds &= \int_0^{\alpha(t)} \varepsilon_s e^{2 \mu (s-t)}ds + \int_{\alpha(t)}^t \varepsilon_s e^{2 \mu (s-t)}ds \\
&\leq \varepsilon_0 e^{-2 \mu t} \int_0^{\alpha(t)} e^{2 \mu s}ds + (t - \alpha(t)) \varepsilon_{\alpha(t)} \\
&\leq \frac{\varepsilon_0}{2 \mu}e^{2 \mu(\alpha(t) - t)} + \frac{\varepsilon_{\alpha(t)} \ln t}{2 \mu} \\
&\leq \frac{\varepsilon_0}{2 \mu t} + \frac{\varepsilon_{t/2} \ln t}{2 \mu} \, ,
\end{align*}
where in the last inequality we used that $\alpha(t) \geq t/2$ and $\varepsilon_t$ is decreasing. 
For $t < T_0$, we can simply bound the integral $\int_0^t \varepsilon_s e^{\mu (s-t)}ds$ by $\varepsilon_0 T_0$. We obtain
\[
\int_0^T \int_0^t \varepsilon_t \varepsilon_s e^{2 \mu (s-t)} \ud s \leq \int_0^{T_0} \varepsilon_t \varepsilon_0 T_0 \ud t + \int_{T_0}^T \frac{\varepsilon_t \varepsilon_0}{2 \mu t} + \frac{\varepsilon_t \varepsilon_{t/2} \ln t}{2 \mu} \ud t \, .
\]
Recall that $\varepsilon_t = \min(1, \frac{1}{\sqrt{t}})$, and that $T_0 \geq 2$. Thus
\begin{align*}
\int_0^T \int_0^t \varepsilon_t \varepsilon_s e^{2 \mu (s-t)} \ud s \leq \int_0^{T_0} \varepsilon_0 T_0 \ud t + \frac{\varepsilon_0}{2 \mu} \int_{T_0}^T \frac{\varepsilon_t}{t} \ud t + \frac{\ln T}{2 \mu} \int_{2}^T \varepsilon_t \varepsilon_{t/2} \ud t \, .    
\end{align*}
The first two integrals are converging when $T \to \infty$ and the last integral is $\mathcal{O}(\ln T)$. Plugging this into~\eqref{eq:upper-bound-grads}, we finally obtain the existence of a constant $c > 0$ such that
\[
\frac{1}{T} \int_0^T \|\nabla \ell(\theta_t)\|^2 \ud t \leq \frac{c (\ln T)^2}{T^{1/2}} \, .
\]

\subsection{Langevin with discrete flow--proof of Theorem \ref{thm:langevin-discrete}}
We take 
\begin{equation}    \label{eq:def-gamma-k-epsilon-k}
\gamma_k = \frac{1}{k^{1/3}} \min \Big(\frac{1}{L_X}, \frac{1}{\sqrt{L_\Theta}}, \frac{1}{\mu}, \frac{\mu}{4 L_X^2} \Big) \quad \textnormal{and} \quad \varepsilon_k = \frac{1}{k^{1/3}}.    
\end{equation}

\paragraph{Bounding the KL divergence of $p_{k+1}$ from $\pi^\star(\theta_{k+1})$.} Recall that $p_k$ is the law of $X_k$. 
We leverage similar ideas to the proof of \cite[Theorem~3]{cb-clmcmckld-18}, exploiting the Log Sobolev inequality to bound the KL along one Langevin Monte Carlo iteration (an approach that was further streamlined in \cite[Lemma~3]{vempala2019rapid}). 
The starting point is to notice that one Langevin Monte Carlo iteration can be equivalently written as a continuous-time process over a small time interval $[0, \gamma_k]$. More precisely, let 
\[
\rho_0:=p_k\,, \quad x_0 \sim \rho_0\,,
\]
and $x_t$ satisfying the SDE
\[
\ud x_t =  -  \nabla_1V(x_0,\theta_k)\ud t+\sqrt{2}\ud B_t \,.
\]
Then, following the proof of \cite[Theorem~3]{cb-clmcmckld-18}, $p_{k+1}$ has the same distribution as the output at time $\gamma := \gamma_k$ of the continuity equation
\begin{equation*}
    \frac{\partial \rho_t}{\partial t}[x] = \div(\rho_t[x](\E_{\rho_{0|t}}[\nabla_1 V(x_k,\theta_k)|x_t = x] +\nabla \log \rho_t))
\end{equation*}
where $\rho_{0|t}$ is the conditional distribution of $x_0$ given $x_t$. Similarly, $\theta_{k+1}$ is equal to the output at time $\gamma$ of 
\begin{equation}    \label{eq:evolution-theta-discrete}
    \vartheta_{t} := \theta_k - t \varepsilon_k \Gamma(\mu_k,\theta_k).
\end{equation}
We have:
\begin{equation*}
    \frac{d\KL(\rho_t\,||\,\pi^\star(\vartheta_t))}{\ud t}=\int \log\left(\frac{\rho_t}{\pi^\star(\vartheta_t)}\right) \frac{\partial \rho_t}{\partial t} + \int \frac{\rho_t}{\pi^\star(\vartheta_t)}\frac{\partial \pi^\star(\vartheta_t)}{\partial t}:= a+b.
\end{equation*}
We first bound $b$ similarly to the proof of Theorem \ref{thm:langevin-continuous}, under Assumptions \ref{ass:gradient_bounded} and \ref{ass:Gamma_Lipschitz}:
\begin{align*}
    b =  \varepsilon_k \int \ps{\nabla_2 V(x, \vartheta_t)- \E_{X \sim \pi^\star(\vartheta_t)}(\nabla_2 V(X, \vartheta_t)), \Gamma(\mu_k, \theta_k)} p_t[x]\ud x \le 2 \varepsilon_k C^2.
\end{align*}
Then we write $a$ as
\begin{align*}
    a &= \int \log(\frac{\rho_t[x]}{\pi^\star(\vartheta_t)[x]})\div(\rho_t [x](\E_{\rho_{0|t}}[\nabla_1 V(x_0,\theta_k)|x_t = x] +\nabla \log \rho_t[x])) \ud x \\
     &= -\int \rho_t(x) \bigg \langle \nabla \log(\frac{\rho_t[x]}{\pi^\star(\vartheta_t)[x]}), \E_{\rho_{0|t}}[\nabla_1 V(x_0,\theta_k)|x_t = x] + \nabla \log \rho_t[x] \bigg \rangle \ud x \\
    & = -\int \rho_t(x) \bigg \langle \nabla \log(\frac{\rho_t[x]}{\pi^\star(\vartheta_t)[x]}), \E_{\rho_{0|t}}[\nabla_1 V(x_0,\theta_k)|x_t = x] -\nabla_1 V(x,\theta_k) +\nabla_1 V(x,\theta_k) \\ 
    & \qquad \qquad \qquad \qquad \qquad \qquad \qquad \qquad \qquad \qquad 
    - \nabla_1 V(x,\vartheta_t) +\nabla \log \left(\frac{\rho_t[x]}{\pi^\star(\vartheta_t)[x]}\right)\bigg \rangle \ud x\\
    & = -\int \rho_t \Big\|\nabla \log(\frac{\rho_t}{\pi^\star(\vartheta_t)})\Big\|^2 \\
    &\qquad+ \int \rho_t[x] \bigg \langle \nabla \log \left(\frac{\rho_t[x]}{\pi^\star(\vartheta_t)[x]}\right), \nabla_1 V(x,\theta_k) - \E_{\rho_{0|t}}[\nabla_1 V(x_0,\theta_k)|x_t = x]\bigg \rangle \ud x \\
    &\qquad+ \int \rho_t[x] \bigg \langle \nabla \log \left(\frac{\rho_t[x]}{\pi^\star(\vartheta_t)[x]}\right), \nabla_1 V(x,\vartheta_t) - \nabla_1 V(x,\theta_k)\bigg \rangle \ud x \\
    &=: a_1+a_2 + a_3.
\end{align*}
Denote the Fisher divergence by
\[
\FD(\rho_t\,||\,\pi^\star(\vartheta_t)) := \int \rho_t \Big\|\nabla \log(\frac{\rho_t}{\pi^\star(\vartheta_t)})\Big\|^2 \, .
\]
The first term $a_1$ is equal to $-\FD(\rho_t\,||\,\pi^\star(\vartheta_t))$. To bound the second term $a_2$, denote $\rho_{0t}$ the joint distribution of $(x_0, x_t)$. Then
\begin{align*}
    a_2 &= \int \rho_{0t}[x_0, x_t] \ps{ \nabla \log \left(\frac{\rho_t[x_t]}{\pi^\star(\vartheta_t)[x_t]}\right), \nabla_1 V(x_t,\theta_k) - \nabla_1 V(x_0,\theta_k)}  \ud x_0 \ud x_t
\end{align*}
Using $\langle a, b\rangle \leq \|a\|^2 + \frac{1}{4}\|b\|^2$ and recalling that $x \mapsto \nabla_1 V(x,\theta)$ is $L_X$-Lipschitz for all $\theta \in \R^p$ by Assumption \ref{ass:smoothness},
\begin{align*}
    a_2 &\leq \E_{(x_0, x_t) \sim \rho_{0t}} \|\nabla_1 V(x_t,\theta_k) - \nabla_1 V(x_0,\theta_k)\|^2 + \frac{1}{4} \E_{(x_0, x_t) \sim \rho_{0t}} \Big\|\nabla \log \left(\frac{\rho_t[x_t]}{\pi^\star(\vartheta_t)[x_t]}\right)\Big\|^2 \\
    &\leq L_X^2 \E_{(x_0, x_t) \sim \rho_{0t}} \|x_t - x_0\|^2 + \frac{1}{4} \FD(\rho_t\,||\,\pi^\star(\vartheta_t)) \, .
\end{align*}
Proceeding similarly for $a_3$, we obtain
\begin{align*}
a_3 \leq \E_{x \sim \rho_{t}} \|\nabla_1 V(x,\vartheta_t) - \nabla_1 V(x,\theta_k)\|^2 + \frac{1}{4} \FD(\rho_t||\pi^\star(\vartheta_t)) \, .
\end{align*}
Since $\theta \mapsto \nabla_1 V(x, \theta)$ is $L_\Theta$-Lipschitz for all $x \in \R^d$ by Assumption \ref{ass:smoothness}, we get
\[
a_3 \leq L_\Theta \|\vartheta_t - \theta_k\|^2 + \frac{1}{4} \FD(\rho_t\,||\,\pi^\star(\theta_t)) \, .
\]
Moreover, by \eqref{eq:evolution-theta-discrete} and under \Cref{ass:Gamma_Lipschitz}, we have $\|\vartheta_t - \theta_k\|^2 = t^2\varepsilon_k^2\|\Gamma(\mu_k,\theta_k)\|^2 \le t^2\varepsilon_k^2 C^2$, which yields
\[
a_3 \leq L_\Theta t^2\varepsilon_k^2 C^2 + \frac{1}{4} \FD(\rho_t\,||\,\pi^\star(\vartheta_t)) \, .
\]
Putting everything together,
\begin{align*}
    \frac{d\KL(\rho_t\,||\,\pi^\star(\vartheta_t))}{\ud t}&= a+b \\
    &\leq - \frac{1}{2} \FD(\rho_t\,||\,\pi^\star(\vartheta_k)) + L_X^2 \E_{(x_0, x_t) \sim \rho_{0t}} \|x_t - x_0\|^2 + L_\Theta t^2\varepsilon_k^2 C^2 + 2 \varepsilon_k C^2 \\
    &\leq - \mu \KL(\rho_t\,||\,\pi^\star(\vartheta_t)) + L_X^2 \E_{(x_0, x_t) \sim \rho_{0t}} \|x_t - x_0\|^2 + L_\Theta t^2\varepsilon_k^2 C^2 + 2 \varepsilon_k C^2
\end{align*}
where the last inequality uses \Cref{ass:potential_log_sobolev} and  where the two last terms in the r.h.s. can be seen as additional bias terms compared to the analysis of \cite[Theorem~3]{cb-clmcmckld-18} and \cite[Lemma 3]{vempala2019rapid}. Let us now bound $\E_{(x_0, x_t) \sim \rho_{0t}} \|x_t - x_0\|^2$. Recall that $x_t \stackrel{d}{=} x_0 - t \nabla_1 V(x_0, \theta_k) + \sqrt{2 t} z_0$, where $z_0 \sim \mathcal{N}(0, I)$ is independent of $x_0$. Then
\begin{align*}
    \E_{(x_0, x_t) \sim \rho_{0t}} \|x_t - x_0\|^2 &= \E_{(x_0, x_t) \sim \rho_{0t}} \|-t \nabla_1 V(x_0, \theta_k) + \sqrt{2 t} z_0\|^2 \\
    &= t^2 \E_{x_0 \sim \rho_{0}} \|\nabla_1 V(x_0, \theta_k)\|^2 + 2 t d.
\end{align*}
Finally, since $x \mapsto \nabla_1 V(x, \theta)$ is $L_X$-Lipschitz for all $x \in \R^d$ by Assumption \ref{ass:smoothness}, and under \Cref{ass:potential_log_sobolev}, we get by \cite[Lemma 12]{vempala2019rapid} that
\[
\E_{x_0 \sim \rho_{0}} \|\nabla_1 V(x_0, \theta_k)\|^2 \leq \frac{4 L_X^2}{\mu} \KL(\rho_0 \, || \, \pi^\star(\theta_k)) + 2 d L_X \, .
\]
All in all,
\begin{align*}
\frac{d\KL(\rho_t \, || \, \pi^\star(\vartheta_t))}{\ud t} &\leq - \mu \KL(\rho_t \, || \, \pi^\star(\vartheta_t)) + \frac{4 L_X^4 t^2}{\mu} \KL(\rho_0 \, || \, \pi^\star(\theta_k)) \\
&\qquad + 2 d L_X^3 t^2 + 2 L_X^2 t d + L_\Theta t^2\varepsilon_k^2 C^2 + 2 \varepsilon_k C^2 \, .
\end{align*}
Recall that we want to integrate $t$ between $0$ and $\gamma$. For $t \leq \gamma$, we have
\begin{align*}
\frac{d\KL(\rho_t \, || \, \pi^\star(\vartheta_t))}{\ud t}
&\leq - \mu \KL(\rho_t \, || \, \pi^\star(\vartheta_t)) + \frac{4 L_X^4 \gamma^2}{\mu} \KL(\rho_0 \, || \, \pi^\star(\theta_k)) \\
&\qquad+ 2 d L_X^3 \gamma^2 + 2 L_X^2 \gamma d + L_\Theta \gamma^2\varepsilon_k^2 C^2 + 2 \varepsilon_k C^2 \\ 
&\leq - \mu \KL(\rho_t \, || \, \pi^\star(\vartheta_t)) + \frac{4 L_X^4 \gamma^2}{\mu} \KL(\rho_0 \, ||  \,\pi^\star(\theta_k)) + 4 L_X^2 \gamma d + 3 \varepsilon_k C^2 
\end{align*}
since $L_X \gamma \leq 1$, $L_\Theta \gamma^2 \leq 1$ and $\varepsilon_k \leq 1$ by \eqref{eq:def-gamma-k-epsilon-k}.
Denote by $C_1$ the second term and $C_2$ the sum of the last two terms. Then, by Grönwall's inequality \citep{pachpatte1997inequalities},
\[
\KL(\rho_t \, || \, \pi^\star(\vartheta_t)) \leq \frac{(C_1 + C_2) e^{-\mu t} (e^{\mu t} - 1)}{\mu} + \KL(\rho_0 \, || \, \pi^\star(\theta_k)) e^{-\mu t} \, .
\]
Since $\mu t \leq \mu \gamma \leq 1$ by \eqref{eq:def-gamma-k-epsilon-k}, we have $e^{\mu t} \leq 1 + 2 \mu \gamma$, and
\begin{align*}
\KL(\rho_t \, ||\, \pi^\star(\vartheta_t)) &\leq 2 (C_1 + C_2) \gamma e^{-\mu t} + \KL(\rho_0 \,||\, \pi^\star(\theta_k)) e^{-\mu t} \\
&= 2 C_2 \gamma e^{-\mu t} + \Big(1 + \frac{8 L_X^4 \gamma^3}{\mu}\Big) \KL(\rho_0 \,||\, \pi^\star(\theta_k)) e^{-\mu t} \, .    
\end{align*}
Since $\gamma \leq \frac{\mu}{4 L_X^2}$ by \eqref{eq:def-gamma-k-epsilon-k}, we have $\frac{8 L_X^4 \gamma^3}{\mu} \leq \frac{\mu \gamma}{2}$, and
\[
\KL(\rho_t \, || \, \pi^\star(\vartheta_t)) \leq 2 C_2 \gamma e^{-\mu t} + \Big(1 + \frac{\mu \gamma}{2}\Big) \KL(\rho_0 \, || \, \pi^\star(\theta_k)) e^{-\mu t} \, . 
\]
We therefore obtain, by evaluating at $t=\gamma$ and renaming $p_{k+1} = \rho_\gamma$, $\pi^\star(\theta_{k+1}) = \pi^\star(\vartheta_\gamma)$, $p_k = \rho_0$, and $\gamma_k = \gamma$,
\[
\KL(p_{k+1} \, || \, \pi^\star(\theta_{k+1})) \leq 2 C_2 \gamma_k e^{-\mu \gamma_k} + \Big(1 + \frac{\mu \gamma_k}{2}\Big) \KL(p_k \, || \, \pi^\star(\theta_k)) e^{-\mu \gamma_k} \, . 
\]

\paragraph{Bounding the KL over the whole dynamics.} Let $C_3 := (1 + \frac{\mu \gamma_k}{2}) e^{-\mu \gamma_k}$. We have $C_3 < 1$, and by summing and telescoping,
\[
\KL(p_k \, || \, \pi^\star(\theta_k)) \leq \KL(p_0 \, || \, \pi^\star(\theta_0)) C_3^k + \frac{2 C_2 C_3 \gamma_k e^{-\mu \gamma_k}}{1 - C_3} \, . 
\]
We have
\[
\frac{C_3 e^{-\mu \gamma_k}}{1 - C_3} = \frac{C_3}{e^{\mu \gamma_k} - (1 + \frac{\mu \gamma_k}{2})} \leq \frac{2}{\mu \gamma_k} \, ,
\]
by using $e^x \geq 1 + x$ and $C_3 \leq 1$.  Thus
\[
\KL(p_k \, || \, \pi^\star(\theta_k)) \leq \KL(p_0 \,||\, \pi^\star(\theta_0)) C_3^k + \frac{4 C_2}{\mu} \, .
\]
Replacing $C_2$ by its value,
\begin{equation*} 
\KL(p_k \,||\, \pi^\star(\theta_k)) \leq \KL(p_0 \,||\, \pi^\star(\theta_0)) C_3^k + \frac{16 L_X^2 d \gamma_k}{\mu} + \frac{12 C^2 \varepsilon_k}{\mu} \, .    
\end{equation*}
Since $e^x \geq 1 + x$, $C_3 \leq e^{- \mu \gamma_k / 2}$, and
\begin{equation}    \label{eq:bound-kl-discrete}
\KL(p_k \,||\, \pi^\star(\theta_k)) \leq \KL(p_0 \,||\, \pi^\star(\theta_0)) e^{- \frac{\mu \gamma_k k}{2}} + \frac{16 L_X^2 d \gamma_k}{\mu} + \frac{12 C^2 \varepsilon_k}{\mu} \, .    
\end{equation}
We obtain three terms in our bound that have different origins. The first term corresponds to an exponential decay of the KL divergence at initialization. The second term is linked to the discretization error, and is proportional to $\gamma_k$. The third term is due to the fact that $\pi^\star(\theta_k)$ is moving due to the outer problem updates. It is proportional to the ratio of learning rates~$\varepsilon_k$.

\paragraph{Evolution of the loss.} By Assumption \ref{ass:smoothness}, the loss $\ell$ is $L$-smooth, and recall that $\theta_{k+1} = \theta_k - \gamma_k \varepsilon_k \Gamma(p_{k},\theta_k)$. We have
\begin{align*}
    \ell(\theta_{k+1}) &= \ell(\theta_k - \gamma_k \varepsilon_k \Gamma(p_{k},\theta_k)) \\
    &\leq \ell(\theta_k) - \gamma_k \varepsilon_k \langle \nabla \ell(\theta_k), \Gamma(p_{k},\theta_k) \rangle + \frac{L \gamma_k^2 \varepsilon_k^2}{2} \|\Gamma(p_{k},\theta_k)\|^2 \\
    &\leq \ell(\theta_k) - \gamma_k \varepsilon_k \langle \nabla \ell(\theta_k), \Gamma(p_{k},\theta_k) \rangle + \frac{L C^2 \gamma_k^2 \varepsilon_k^2}{2}
\end{align*}
by \Cref{ass:Gamma_Lipschitz}. Furthermore, 
\[
\Gamma(\mu_{k},\theta_k) = \Gamma(\pi^\star(\theta_k),\theta_k) + \Gamma(p_{k},\theta_k) - \Gamma(\pi^\star(\theta_k),\theta_k) = \nabla \ell(\theta_k) + \Gamma(p_{k},\theta_k) - \Gamma(\pi^\star(\theta_k),\theta_k) \, .
\]
Thus
\begin{align*}
\ell(\theta_{k+1}) &\leq \ell(\theta_k) - \gamma_k \varepsilon_k \|\nabla \ell(\theta_k)\|^2 + \gamma_k \varepsilon_k \|\nabla \ell(\theta_k)\| \|\Gamma(p_{k},\theta_k) - \Gamma(\pi^\star(\theta_k),\theta_k)\| + \frac{L C^2 \gamma_k^2 \varepsilon_k^2}{2} \\
&\leq \ell(\theta_k) - \gamma_k \varepsilon_k \|\nabla \ell(\theta_k)\|^2 + \gamma_k \varepsilon_k \|\nabla \ell(\theta_k)\| \sqrt{\KL(p_k || \pi^\star(\theta_k))} + \frac{L C^2 \gamma_k^2 \varepsilon_k^2}{2}
\end{align*}
by \Cref{ass:Gamma_Lipschitz}. Using $ab \leq \frac{1}{2}(a^2 + b^2)$, we obtain
\[
\ell(\theta_{k+1}) \leq \ell(\theta_k) - \frac{\gamma_k \varepsilon_k}{2} \|\nabla \ell(\theta_k)\|^2 + \frac{\gamma_k \varepsilon_k}{2}  \KL(p_k || \pi^\star(\theta_k)) + \frac{L C^2 \gamma_k^2 \varepsilon_k^2}{2} \, .
\]
\paragraph{Conclusion.} Summing and telescoping,
\[
\ell(\theta_{K+1}) - \ell(\theta_1) \leq - \frac{1}{2} \sum_{k=1}^{K} \gamma_k \varepsilon_k \|\nabla \ell(\theta_k)\|^2 + \frac{1}{2} \sum_{k=1}^{K} \gamma_k \varepsilon_k \KL(p_k || \pi^\star(\theta_k)) + \frac{L C^2}{2} \sum_{k=1}^{K} \gamma_k^2 \varepsilon_k^2 \, .
\]
Lower bounding $\gamma_k$ by $\gamma_K$ and $\varepsilon_k$ by $\varepsilon_K$ in the first sum, then reorganizing terms, we obtain
\[
\frac{1}{K} \sum_{k=1}^{K} \|\nabla \ell(\theta_k)\|^2 \leq \frac{2 (\ell(\theta_1) - \inf \ell)}{K \gamma_K \varepsilon_K} + \frac{1}{K \gamma_K \varepsilon_K} \sum_{k=1}^{K} \gamma_k \varepsilon_k \KL(p_k || \pi^\star(\theta_k)) + \frac{L C^2}{K \gamma_K \varepsilon_K} \sum_{k=1}^{K} \gamma_k^2 \varepsilon_k^2 \, .
\]
Bounding the KL divergence by \eqref{eq:bound-kl-discrete}, we get
\begin{align*}
\frac{1}{K} \sum_{k=1}^{K} \|\nabla \ell(\theta_k)\|^2 &\leq \frac{2 (\ell(\theta_1) - \inf \ell)}{K \gamma_K \varepsilon_K} + \frac{1}{K \gamma_K \varepsilon_K} \sum_{k=1}^{K} \KL(p_0 || \pi^\star(\theta_0)) \gamma_k \varepsilon_k e^{- \frac{\mu \gamma_k k}{2}}  \\
&+ \frac{1}{K \gamma_K \varepsilon_K} \sum_{k=1}^{K} \frac{16 L_X^2 d \gamma_k^2 \varepsilon_k}{\mu} + \frac{1}{K \gamma_K \varepsilon_K} \sum_{k=1}^{K} \frac{12 C^2 \gamma_k \varepsilon_k^2}{\mu} + \frac{L C^2 \gamma}{K \gamma_K \varepsilon_K} \sum_{k=1}^{K} \gamma_k^2 \varepsilon_k^2 \, .    
\end{align*}
By definition \eqref{eq:def-gamma-k-epsilon-k} of $\gamma_k$ and $\varepsilon_k$, we see that the first and last sums are converging, and the middle sums are $\mathcal{O}(\ln K)$. Therefore, we obtain
\[
\frac{1}{K} \sum_{k=1}^{K} \|\nabla \ell(\theta_k)\|^2 \leq \frac{c \ln K}{K^{1/3}}
\]
for some $c > 0$ depending only on the constants of the problem.

\subsection{Denoising diffusion}    \label{apx:theory-diffusion-denoising}

Our goal is first to find a continuous-time equivalent of Algorithm \ref{alg:implicit-diff-maxi}. To this aim, we first introduce a slightly different version of the algorithm where the queue of $M$ versions of $p_k$ is not present at initialization, but constructed during the first~$M$ steps of the algorithm. As a consequence, $\theta$ changes for the first time after $M$ steps of the algorithm, when the queue is completed and the $M$-th element of the queue has been processed through all the sampling steps.

\begin{algorithm}[H]
\caption{Implicit Diff. optimization, finite time (no warm start)} \label{alg:implicit-diff-maxi-no-warm-start}
\begin{algorithmic}
\INPUT $\theta_0 \in \R^p$, $p_0 \in \cP$
\STATE $p_0^{(0)} \gets p_0$ 
\FOR{$k \in \{0, \dots, K-1\}$ (joint single loop)}
    \STATE $p_{k+1}^{(0)} \gets p_0$
    \STATE \textbf{parallel} $p_{k+1}^{(m+1)} \gets \Sigma_{m}(p_k^{(m)}, \theta_k)$ for $m \in [\min(k, M-1)]$
    \IF{$k \geq M$}
        \STATE $\theta_{k+1} \gets \theta_k - \eta_k \Gamma(p_k^{(M)}, \theta_k)$ (or another optimizer)
    \ENDIF
\ENDFOR
\OUTPUT $\theta_K$
\end{algorithmic}
\end{algorithm}

In order to obtain the continuous-time equivalent of Algorithm \ref{alg:implicit-diff-maxi-no-warm-start}, it is convenient to change indices defining $p_k^{(m)}$. Note that in the update of $p_k^{(m)}$ in the algorithm, the quantity $m-k$ is constant. Therefore, denoting $j = m-k$, the algorithm above is exactly equivalent to
\begin{align*}
    p_j^{(0)} &= p_0 \\
    p_j^{(m+1)} &= \Sigma_m(p_j^{(m)}, \theta_{j+m}), \quad m \in [M-1] \\
    \theta_{k+1} &= \theta_k \quad \textnormal{ if } k < M \\
    \theta_{k+1} &= \theta_k - \eta_k \Gamma(p_{k-M}^{(M)}, \theta_k) \quad \textnormal{ else.}
\end{align*}
It is then possible to translate this algorithm in a continuous setting in the case of denoising diffusions. We obtain
\begin{align}   \label{eq:implit-diff-finite-denoising}
    \begin{split}
    Y^0_t &\sim \cN(0, 1) \\
    \ud Y^\tau_t &= \{Y^\tau_t + 2 s_{\theta_{t + \tau}}(Y^\tau_t, T-\tau)\} \ud \tau + \sqrt{2} \ud B_\tau \\
    \ud \theta_t &= 0 \quad \textnormal{ for } t < T \\
    \ud \theta_t &= - \eta_k \Gamma(Y_{t-T}^T, \theta_t) \ud t \quad \textnormal{ for } t \geq T,
    \end{split}
\end{align}
Let us choose the score function $s_\theta$ as in Section \ref{theory:diffusion-denoising}. The backward equation \eqref{eq:diffusionback} then writes
\begin{equation}    \label{eq:1d-diffusion-back}
    \ud Y_t = (Y_t + 2 s_\theta(Y_t, T-t) )\ud t + \sqrt{2}\ud B_t = -(Y_t - 2 \theta e^{-(T-t)} )\ud t + \sqrt{2}\ud B_t\, .    
\end{equation}
We now turn our attention to the computation of $\Gamma$. Recall that $\Gamma$ should satisfy $\Gamma(\pi^\star(\theta), \theta) = \nabla \ell(\theta)$, where $\pi^\star(\theta)$ is the distribution of $Y_T$ and $\ell(\theta) = \E_{Y \sim \pi^\star(\theta)}(L(Y))$ for some loss function $L:\R \to \R$. To this aim, for a given realization of the SDE \eqref{eq:1d-diffusion-back}, let us compute the derivative of $L(Y_T)$ with respect to $\theta$ using the adjoint method \eqref{eq:adjoint-sde}. We have $\nabla_2 \mu(T-t, Y_{T-t}, \theta) = -1$, hence $\frac{\ud A_t}{\ud t} = -1$, and $A_t = L'(Y_T) e^{-t}$. Furthermore, $\nabla_3 \mu(T-t, Y_{T-t}, \theta) = 2e^{-t}$. Thus
\[
    \frac{\ud L(Y_T)}{\ud \theta} = \int_0^T A_t \nabla_3 \mu(T-t, Y_{T-t}, \theta) \ud t = \int_0^T 2 L'(Y_T) e^{-2t} \ud t = L'(Y_T) (1 - e^{-2T}) \, .
\]
As a consequence,
\[
    \nabla \ell(\theta) = \E_{Y \sim q_T^*(\theta)}\Big(\frac{\ud L(Y)}{\ud \theta}\Big) = \E_{Y \sim q_T^*(\theta)}(L'(Y) (1 - e^{-2T})) \, .
\]
This prompts us to define
\[
    \Gamma(p, \theta) = \E_{X \sim p}(L'(X) (1 - e^{-2T})) = \E_{X \sim p}((X - \theta_{\textnormal{target}}) (1 - e^{-2T})) = (\E_{X \sim p}X - \theta_{\textnormal{target}}) (1 - e^{-2T})
\]
for $L$ defined by $L(x) = (x - \theta_{\textnormal{target}})^2$ as in Section \ref{theory:diffusion-denoising}.
Note that, in this case, $\Gamma$ depends only on $p$ and not on $\theta$.

Replacing $s_\theta$ and $\Gamma$ by their values in \eqref{eq:implit-diff-finite-denoising}, we obtain the coupled differential equations in $Y$ and $\theta$
\begin{align}  \label{eq:implit-diff-finite-denoising-instantiated}
\begin{split}
    Y^0_t &\sim \cN(0, 1) \\
    \ud Y^\tau_t &= \{-Y^\tau_t + 2 \theta_{t + \tau} e^{-(T-\tau)}\} \ud \tau + \sqrt{2} \ud B_\tau \\
    \ud \theta_t &= 0 \quad \textnormal{ for } t < T \\
    \ud \theta_t &= - \eta (\E Y_{t-T}^T - \theta_{\textnormal{target}}) (1 - e^{-2T}) \ud t \quad  \textnormal{ for } t \geq T,
\end{split}
\end{align}
We can now formalize Proposition \ref{prop:diffusion-denoising-proof}, with the following statement.
\begin{proposition}
    Consider the dynamics \eqref{eq:implit-diff-finite-denoising-instantiated}. Then
    \[
    \|\theta_{2T} - \theta_{\textnormal{target}}\| = \mathcal{O}(e^{-T}) \, , 
    \quad \textnormal{and} \quad 
    \pi^\star(\theta_{2T}) = \cN(\mu, 1) \textnormal{ with } \mu = \theta_{\textnormal{target}} + \mathcal{O}(e^{-T}) \, .
    \]
\end{proposition}
\begin{proof}
Let us first compute the expectation of $Y^\tau_T$. To this aim, denote $Z^\tau_t = e^\tau Y^\tau_t$. Then we have
\[
    \ud Z^\tau_t = e^\tau (\ud Y^\tau_t + Y^\tau_t \ud t) = 2 \theta_{t + \tau} e^{2\tau - T} \ud \tau + \sqrt{2} \ud B_\tau \, .
\]
Since $\E(Z^0_t) = \E(Y^0_t) = 0$, we obtain that $\E(Z^T_t) = 2 \int_{0}^{T} \theta_{t + \tau} e^{2(\tau-T)}\ud\tau$, and
\[
    \E(Y^T_t) = e^{-T} \E(Z^T_t) = 2 \int_0^T \theta_{t + \tau} e^{2(\tau-T)}\ud\tau \, .
\]
Therefore, we obtain the following evolution equation for $\theta$ when $t \geq T$:
\[
\dot{\theta}_t = - \eta \Big( 2 \int_0^T \theta_{t-T+\tau} e^{2(\tau-T)}\ud \tau - \theta_{\textnormal{target}} \Big) (1 - e^{-2T}) \, .
\]
By the change of variable $\tau \gets T - \tau$ in the integral, we have, for $t \geq T$,
\[
\dot{\theta}_t = - \eta \Big( 2 \int_0^T \theta_{t-\tau} e^{-2\tau}\ud \tau - \theta_{\textnormal{target}} \Big) (1 - e^{-2T}) \, .
\]
Let us introduce the auxiliary variable $\psi_t = \int_0^T \theta_{t-\tau} e^{-2\tau}\ud \tau$. We have
$\dot{\theta}_t = - \eta(2 \psi_t - \theta_{\textnormal{target}})$, and
\[
\dot{\psi}_t = \int_0^T \dot{\theta}_{t-\tau} e^{-2\tau}\ud \tau = - [\theta_{t-\tau} e^{-2\tau}]_0^T - 2 \int_0^T \theta_{t-\tau} e^{-2\tau} d\tau = \theta_t - \theta_{t - T} e^{-2T} - 2 \psi_t \, .
\]
Recall that $\theta_t$ is constant equal to $\theta_0$ for $t \in [0, T]$. Therefore, for $t \in [T, 2T]$, $\xi := (\theta, \psi)$ satisfies the first order linear ODE with constant coefficients
\[
\dot{\xi} = A \xi + b, \quad 
A = \begin{pmatrix}
0 & -2\eta \\
1 & -2
\end{pmatrix}, \quad
b = \begin{pmatrix}
\eta \theta_{\textnormal{target}} (1-e^{-2T})  \\
- \theta_0 e^{-2T}
\end{pmatrix} \, .
\]
For $\eta > 0$, $A$ is invertible, and 
\[
A^{-1} = \frac{1}{2 \eta} \begin{pmatrix}
-2 & 2\eta \\
-1 & 0
\end{pmatrix} \, .
\]
Hence the linear ODE has solution, for $t \in [T, 2T]$,
\[
\xi_t = -(I - e^{(t-T)A}) A^{-1} b + e^{(t-T)A} \xi_T = - A^{-1} b + e^{(t-T)A} (A^{-1} b + \xi_T) \, .
\]
Thus
\[
\xi_{2T} = - A^{-1} b + e^{TA} (A^{-1} b + \xi_T) = - A^{-1} b + e^{TA} \Bigg(A^{-1} b + \begin{pmatrix}
\theta_0  \\
\frac{\theta_0 (1 - e^{-2T})}{2} 
\end{pmatrix}\Bigg) \, .
\]
A straightforward computation shows that
\[
[A^{-1}b]_0 = - \theta_{\textnormal{target}} (1 - e^{-2T}) + \theta_0 e^{-2T} \, ,
\]
and that $A$ has eigenvalues with negative real part. Putting everything together, we obtain
\[
\theta_{2T} = [\xi_{2T}]_0 = \theta_{\textnormal{target}} + \mathcal{O}(e^{-T}) \, .
\]
Finally, recall that we have $\pi^\star(\theta) = \cN(\theta (1-e^{-2T}), 1)$. Thus $\pi^\star(\theta_{2T}) = \cN(\mu_{2T}, 1)$ with $\mu_{2T} = \theta_{\textnormal{target}} + \mathcal{O}(e^{-T})$.
\end{proof}

\section{EXPERIMENTAL DETAILS}  \label{apx:experimental-details}
We provide here details about the experiments that we have presented in Section~\ref{sec:experiments}. 
\subsection{Langevin processes}
We provide in this section details about our experiments on Langevin processes. In Section~\ref{apx:langevin-reward}, we do so for the experiment described in Section~\ref{sec:experiments-langevin}, where the reward is an explicit function $R(\cdot)$. In Section~\ref{apx:langevin-scratch}, we show additional experiments showing that Implicit Diffusion can also be used to ``train from scratch'' a model: in this case the reward is a negative KL term that is evaluated from samples of a target distribution. We present results in several parametrization settings.  

\subsubsection{Reward training}
\label{apx:langevin-reward}

We consider a parameterized family of potentials for $x \in \R^2$ and $\theta \in \R^6$ defined by
\[
V(x, \theta) = -\log\Big(\sum_{i=1}^6 \sigma(\theta)_i \exp(-\|x - \mu_i\|^2) \Big)\, ,
\]
where the $\mu_i \in  \R^2$ are the six vertices of a regular hexagon and $\sigma$ is the softmax function mapping $\R^6$ to the unit simplex. In this setting, for any $\theta \in \R^6$,
\[
\pi^\star(\theta) = \frac{1}{Z}\sum_{i=1}^6 \sigma(\theta)_i \exp(-\|x - \mu_i\|^2) \, ,
\]
where $Z$ is an absolute renormalization constant that is independent of $\theta$. This fact simplifies drawing contour lines, but we do not use this prior knowledge in our algorithms, and only use calls to functions $\nabla_1 V(\cdot,\theta)$ and $\nabla_2 V(\cdot,\theta)$ for various $\theta \in \R^6$.

\begin{figure}[ht]
\vskip 0.1in
\begin{center}
\centerline{\includegraphics[width=\textwidth]{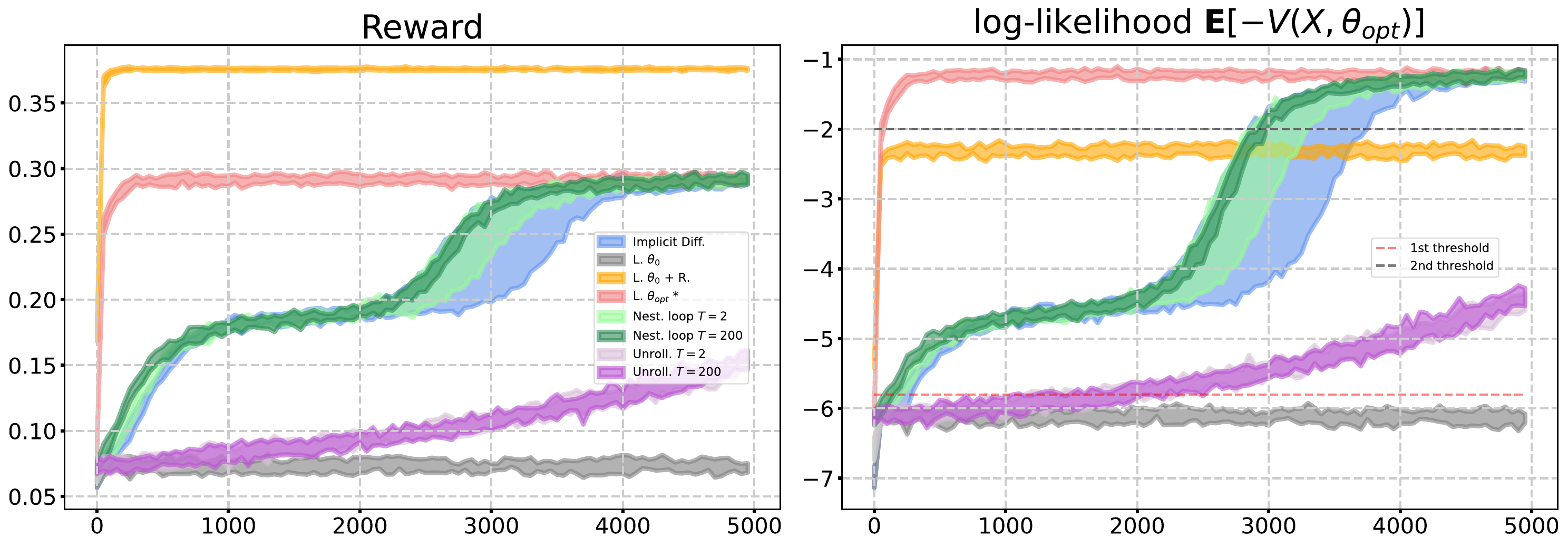}}
\caption{Confidence intervals for metrics for reward training of Langevin processes. \textbf{Left}: Evolution of the reward. \textbf{Right}: Evolution of the log-likelihood. \label{fig:values-ci-only}}
\end{center}
\end{figure}

We run six sampling algorithms, all initialized with $p_0 = \cN(0, I_2)$. For all of them we generate a batch of variables $X^{(i)}$ of size $1,000$, all initialized independently with $X_0^{(i)} \sim \cN(0, I_2)$. 
The sampling and optimization steps are realized in parallel over the batch. 
The samples are represented after $K=5,000$ steps of each algorithm in Figure~\ref{fig:contour-langevin}, and used to compute the values of reward and likelihood reported in Figure~\ref{fig:rewards-langevin}. We also display in Figure~\ref{fig:contour-langevin-history} the dynamics of the probabilities throughout these algorithms.

\begin{figure}[H]
\vskip 0.1in
\begin{center}
\centerline{\includegraphics[width=\textwidth]{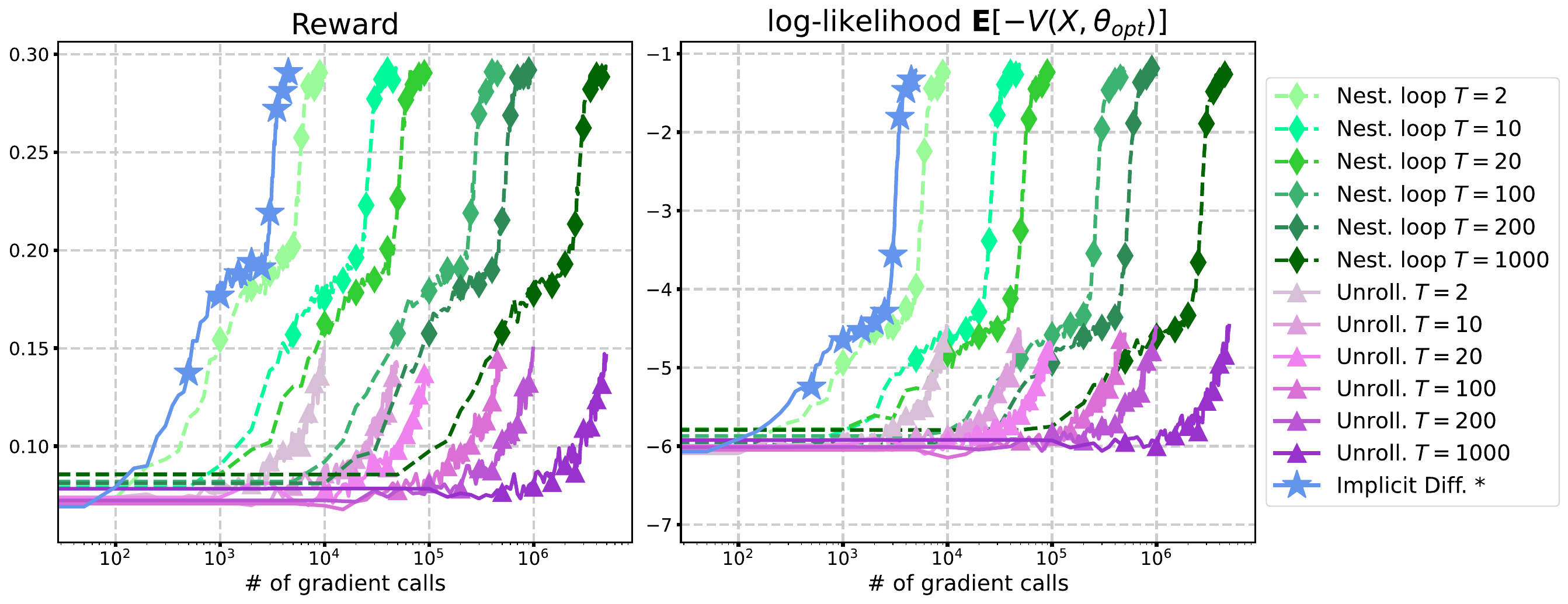}}
\caption{Comparison between Implicit Diffusion, nested loop algorithm and unrolling algorithm, for various number of inner steps $T$. The x-axis is the total number of gradient evaluations (roughly equal to the number of optimization steps multiplied by the number of inner loop steps $T$). \textbf{Left}: Evolution of the reward. \textbf{Right}: Evolution of the log-likelihood. \label{fig:values-shifted}}
\end{center}
\end{figure}

\newpage

We provide here additional details of and motivation for these algorithms, denoted by the colored markers that represent them in these figures.

\afterpage{\clearpage}
\begin{figure}[p]
\vskip 0.1in
\begin{center}
\centerline{\includegraphics[width=\textwidth]{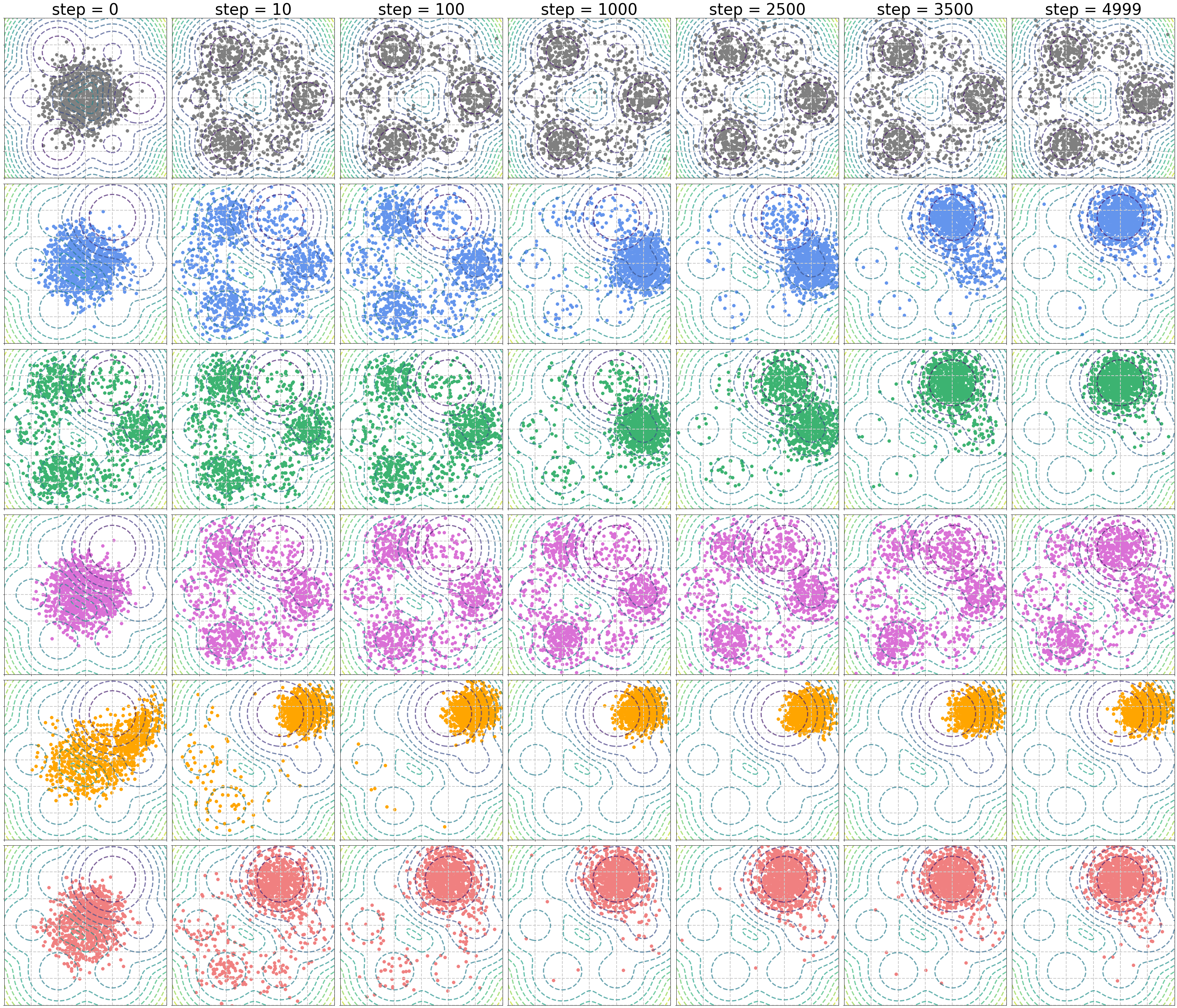}}
\caption{Dynamics of samples for four sampling algorithms after different time steps (for instance, the first column is after one step). \textbf{First row:}  Langevin $\theta_0$ $({\color{gray} \CIRCLE})$ with $\pi^\star(\theta_0)$ contour lines. \textbf{Second:} Implicit Diffusion $({\color{myblue} \CIRCLE})$ with $\pi^\star(\theta_\text{opt})$ contour lines. \textbf{Third:} Nested loop algorithm with $T=100$~$({\color{mydarkgreen} \CIRCLE})$. \textbf{Fourth:} Unrolling through the last step of sampling with $T=100$ $({\color{mydarkorchid} \CIRCLE})$. \textbf{Fifth:} Langevin $\theta_0$ + smoothed Reward $({\color{myorange} \CIRCLE})$. \textbf{Sixth:} Langevin $\theta_{\text{opt}}$ $({\color{mycoral} \CIRCLE})$.
\label{fig:contour-langevin-history}}
\end{center}
\end{figure}

\begin{itemize}
    \item[-] Langevin $\theta_0$ $({\color{gray} \blacksquare})$: This is the discrete-time process (a Langevin Monte Carlo process) approximating a Langevin diffusion with potential $V(\cdot, \theta_0)$ for fixed $\theta_0: =(1, 0, 1, 0, 1, 0)$. There is no reward here; the time-continuous Langevin process converges to $\pi^\star(\theta_0)$, which has some symmetries. It can be thought of as a pretrained model, and the Langevin sampling algorithm as an inference-time generative algorithm.
    \item[-] \textbf{Implicit Diffusion} $({\color{myblue} \bigstar})$: We run the infinite-time horizon version of our method (Algorithm~\ref{alg:implicit-diff}), aiming to minimize 
    $\ell(\theta) := \cF(\pi^\star(\theta))$ for $\cF(p) = - \E_{X\sim p}[R(X)]$ with $R(x) = \mathbf{1}(x_1 >0) \exp(-\|x-\mu\|^2)$
    where $\mu = (1, 0.95)$. This algorithm yields both a sample $\hat p_K$ and parameters $\theta_{\text{opt}}$ after $K$ steps, and can be thought of as jointly sampling and reward finetuning.
    \item[-] Nested loop $({\color{mydarkgreen} \blacklozenge})$: We run Algorithm \ref{alg:naive-nested} with $T$ inner sampling steps for each gradient step. For $T=1$, this is exactly Implicit Diffusion. For $T \gg 1$, it means we compute nearly perfectly $\pi^\star(\theta_t)$ at each optimization step.
    \item[-] Unrolling through the last step of sampling $({\color{mydarkorchid} \blacktriangle})$: For each optimization step, we perform $T$ sampling steps, then differentiate through the last step of sampling by automatic differentiation. This is akin to a stop gradient method. The learning rate here is chosen as $2 \gamma_\theta / \gamma_X$ to improve its performance, for a fair comparison. Recent studies show that differentiating through the last sampling step is an efficient and theoretically-grounded method in bilevel optimization \citep{bolte2023onestep}. It has been applied successfully to denoising diffusions \citep{clark2024directly}.
    \item[-] Langevin $\theta_0$ + R $({\color{myorange} \blacktriangledown})$: This is the discretization of a Langevin diffusion with reward-guided potential $V(\cdot, \theta_0) - \lambda R_{\text{smooth}}$, where $R_{\text{smooth}}$ is a smoothed version of $R$ (replacing the indicator by a sigmoid). Using this approach is different from finetuning: it proposes to modify the sampling algorithm, and does not yield new parameters~$\theta$. This is akin to guidance of generative models \citep{dhariwal2021diffusion}. Note that this requires a differentiable reward $R_{\text{smooth}}$, contrarily to our approach that handles non-differentiable rewards. 
    \item[-] Langevin $\theta_{\text{opt}}$ - post \textbf{Implicit Diffusion} $({\color{mycoral} \CIRCLE})$: This is a discrete-time process approximating a Langevin diffusion with potential $V(\cdot, \theta_{\text{opt}})$, where $\theta_\text{opt}$ is the outcome of reward training by our algorithm. This can be thought of as doing inference with the new model parameters, post reward training with Implicit Diffusion.
\end{itemize}

As mentioned in Section~\ref{sec:experiments-langevin}, this setting illustrates the advantage of our method, which allows the efficient optimization of a function over a constrained set of distribution, without overfitting outside this class. We display in Figure~\ref{fig:contour-langevin-history} snapshots throughout some selected steps of these six algorithms (in the same order and with the same colors as indicated above). We observe that the dynamics of Implicit Diffusion are slower than those of Langevin processes (sampling), which can be observed also in the metrics reported in Figure~\ref{fig:rewards-langevin}. The reward and log-likelihood change slowly, plateauing several times: when $\theta_k$ in this algorithm is initially close to $\theta_0$, the distribution gets closer to $\pi^\star(\theta_0)$ (steps 0-100). It then evolves towards another distribution (steps 1000-2500), after $\theta$ has been affected by accurate gradient updates, before converging to $\pi^\star(\theta_\text{opt})$. The two-timescale dynamics is by design: the sampling dynamics are much faster, aiming to quickly lead to an accurate evaluation of gradients with respect to $\theta_k$. This corresponds to our theoretical setting where $\varepsilon_k \ll 1$. To complement the comparisons between our algorithm and other baselines included in Section~\ref{sec:experiments-langevin}, we also provide in Figure~\ref{fig:values-shifted} a comparison between Implicit Diffusion, the nested loop and unrolling approaches, in terms of reward and log-likelihood optimization, \textbf{as a function of the number of gradient evaluations} (i.e., number of sampling steps), rather than number of optimization steps. Again, it is apparent that the algorithmic cost of doing several steps ($T>1$) of inner loop is much higher than the small improvement obtained by a better estimate of the gradients. Finally, the confidence intervals in Figure~\ref{fig:rewards-langevin} are computed by performing $10$ independent repetitions of the experiment, and reporting the largest and lowest metrics across the $10$ repetitions, at each time step. For readability, Figure~\ref{fig:values-ci-only} shows the same plot with confidence intervals only (without plotting the average value).

\newpage

\subsubsection{Training from scratch}
\label{apx:langevin-scratch}
We present in Figure \ref{fig:training-from-scratch} a variant of this experiment where we start from a model generating a standard Gaussian, and our goal is to learn to generate a mixture of several Gaussians. For this, comparing with the setup presented above, we add a $7$th potential well at the origin, and choose at initialization $\theta_0 = (-7, -7, -7, -7, -7, -7, 11)$. This means that the distribution at initialization is extremely close to being a standard Gaussian, as can be seen in the top-right plot of Figure \ref{fig:training-from-scratch-top}. The target is $\theta^* = (1.5, 0, 1.5, 0, 1.5, 0, 0)$. We use Implicit Diffusion where the reward is the KL between the target distribution and the current one. This KL admits explicit gradients (see Section \ref{sec:grad-sampling}) which can be evaluated with samples of the target distribution. We train for $K=40,000$ steps with a batch of size $1,000$.
Learning rates are $\gamma_X = 2.5 \cdot 10^{-2}$ and $\gamma_\theta = 7 \cdot 10^{-3}$.

\paragraph{Training from scratch with a full Gaussian mixture model parameterization.}
We present in Figure \ref{fig:training-from-scratch-full-gmm} a variant of this experiment where we now parameterize the means and covariances of the Gaussians in addition to the weights of the mixture. More precisely, we consider the potential
\[
V(x, \theta) = -\log\Big(\sum_{i=1}^2 \frac{\sigma(w)_i}{2\pi \sqrt{\det(\Sigma_i)}} \exp(-(x - \mu_i)^\top \Sigma_i^{-1} (x - \mu_i) )  \Big)\, ,
\]
where $\sigma$ is still the softmax function mapping $\R^2$ to the unit simplex, and $\theta = \{w, (\mu_i, \Sigma_i)_{1\leq i \leq 2}\} \in \R^2 \times (\R^2, \R^{2 \times 2})^2$. At initialization, the parameters are
\[
\theta_0 = \big\{(0, 0), ((4, 0), I_2), ((-2, 2 \sqrt{3}), I_2) \big\} \, ,
\]
while our target is
\[
\theta^* = \bigg\{(1.5, 0), 
\bigg((-4, 0), \begin{pmatrix}
0.75 & -0.5 \\
-0.5 & 1.5
\end{pmatrix}\bigg), 
\bigg((2, -2 \sqrt{3}), \begin{pmatrix}
0.75 & 0.5 \\
0.5 & 1.25
\end{pmatrix} \bigg) \bigg\} \, .
\]
As in the previous experiment, we use Implicit Diffusion where the reward is the KL between the target distribution and the current one. This KL admits explicit gradients (see Section \ref{sec:grad-sampling}) which can be evaluated with samples of the target distribution. We train for $K=40,000$ steps with a batch of size $1,000$. Learning rates are $\gamma_X = 5 \cdot 10^{-2}$ and $\gamma_\theta = 5 \cdot 10^{-4}$. We observe on the Figure that Implicit Diffusion is able to learn the target distribution.

\afterpage{\clearpage}
\begin{figure}[p]
    \begin{subfigure}[b]{0.99\textwidth}
    \begin{center}
    \centerline{\includegraphics[width=0.5\linewidth]{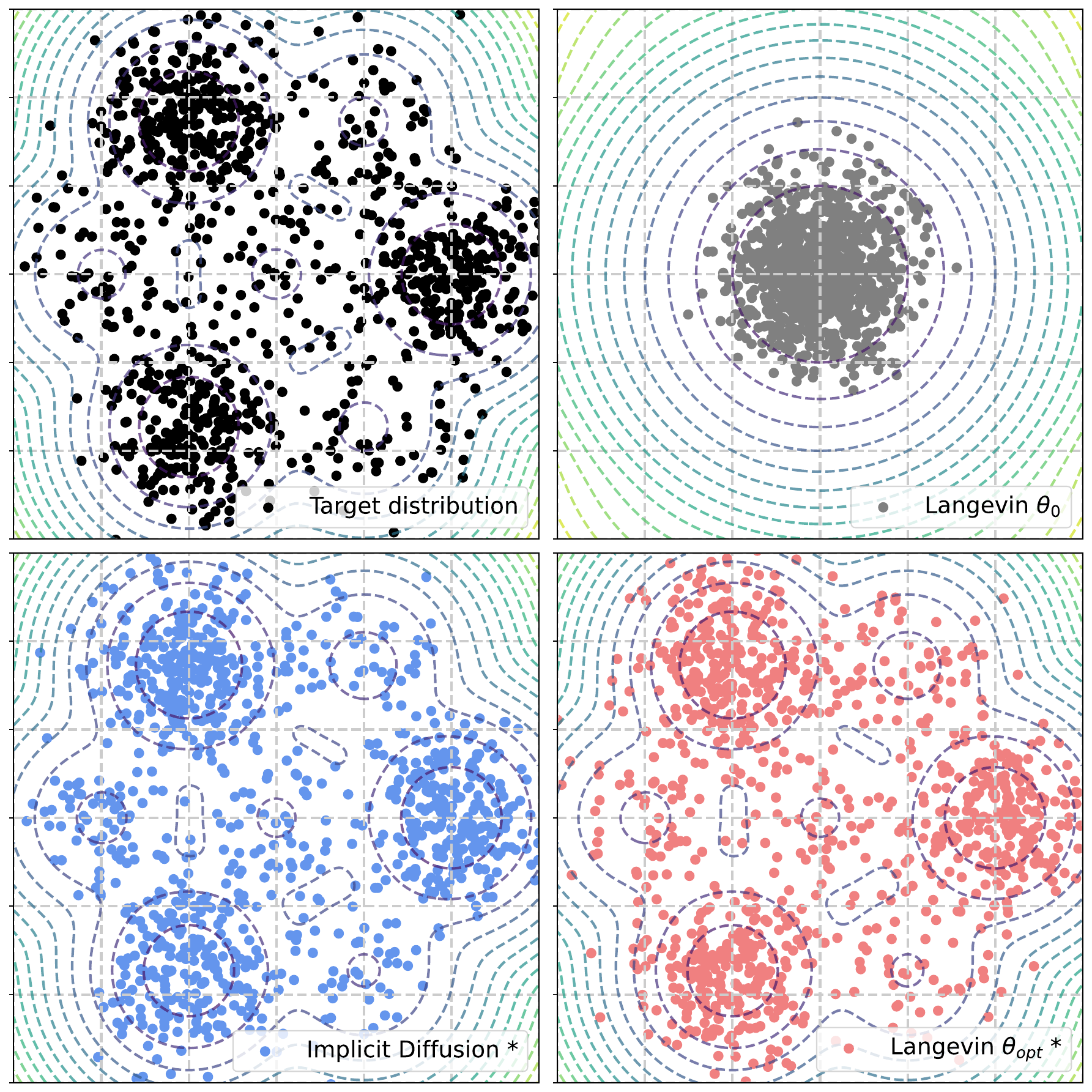}}
    \caption{Contour lines and samples from  sampling algorithms. We start from a model generating a standard Gaussian (top-right figure), and our goal is to learn to generate a mixture of several Gaussians (top-left figure). We use Implicit Diffusion where the reward is the KL between the target distribution and the current one. We observe that Implicit Diffusion is able to learn the target distribution (bottom-left figure). After running Implicit Diffusion, it is easy to generate new samples that are close to the target distribution (bottom-right figure).
    }
    \label{fig:training-from-scratch-top}
    \end{center}
    \end{subfigure}
    \begin{subfigure}[b]{0.99\textwidth}
    \begin{center}
    \centerline{\includegraphics[width=0.7\linewidth]{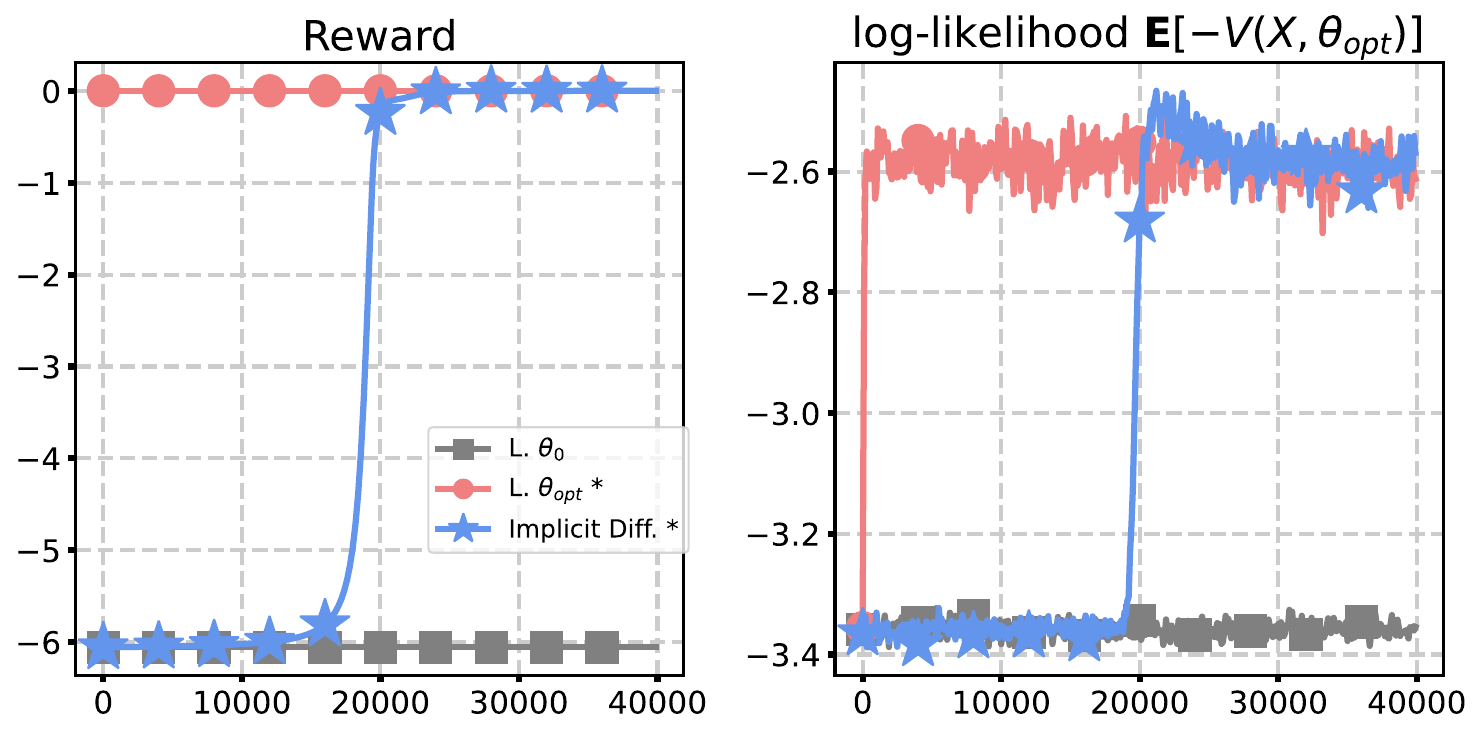}}
    \caption{Evolution of the reward and of the log-likelihood of the samples for the initial model, for the Implicit Diffusion algorithm, and for the trained model after Implicit Diffusion. The reward is the (opposite of the) KL between the target distribution and the current one, so a reward equal to zero means we learnt to reproduce the target distribution.}
    \end{center}
\end{subfigure}
\caption{Optimizing through sampling with \textbf{Implicit Diffusion} to train from scratch an energy-based model (Langevin diffusion).}
\label{fig:training-from-scratch}
\end{figure}

\afterpage{\clearpage}
\begin{figure}[p]
    \begin{subfigure}[b]{0.99\textwidth}
    \begin{center}
    \centerline{\includegraphics[width=0.5\linewidth]{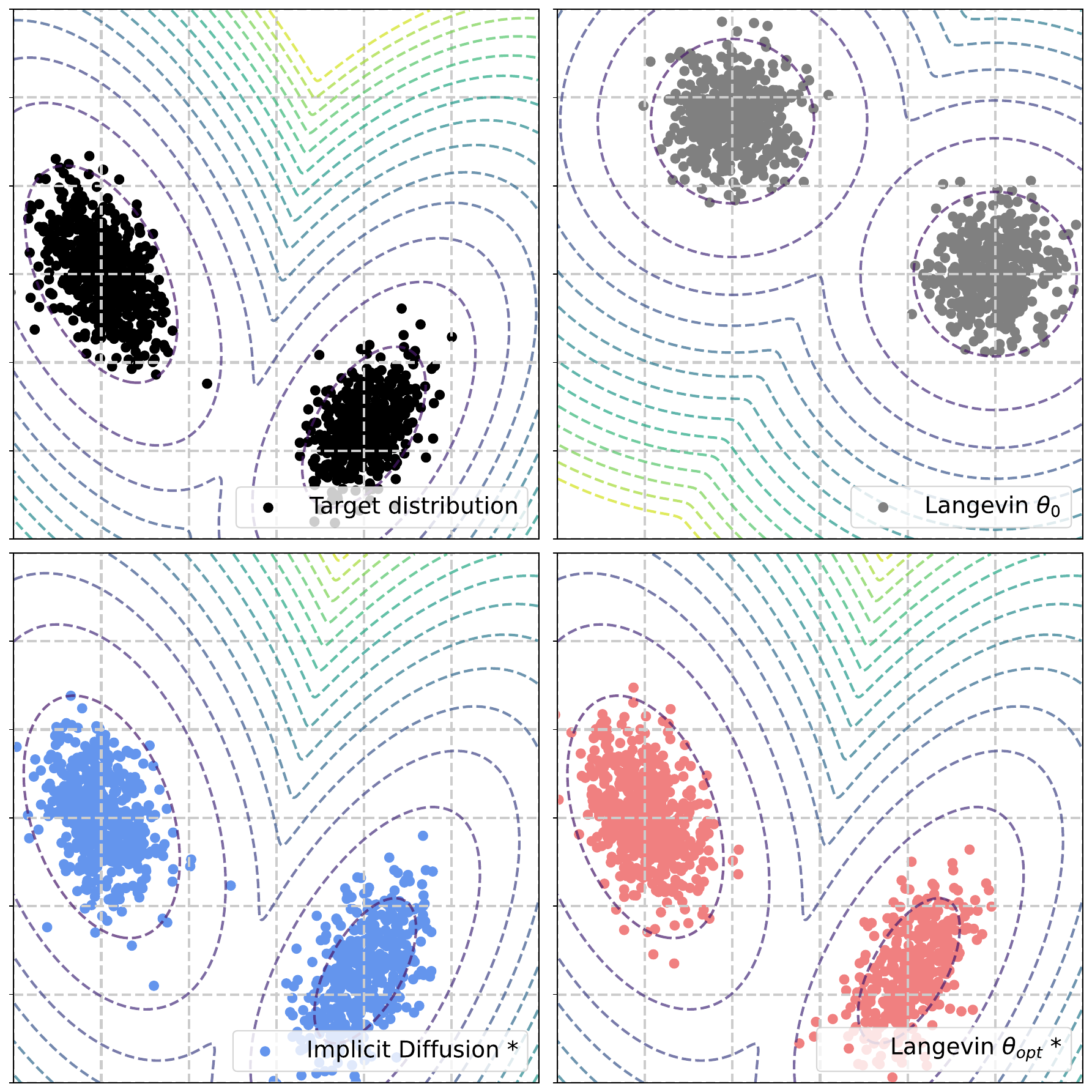}}
    \caption{Contour lines and samples from  sampling algorithms. We start from a model generating a mixture of two Gaussians (top-right figure), and our goal is to learn to generate another mixture with different means, covariances and weights (top-left figure). We use Implicit Diffusion where the reward is the KL between the target distribution and the current one. We observe that Implicit Diffusion is able to learn the target distribution (bottom-left figure). After running Implicit Diffusion, it is easy to generate new samples that are close to the target distribution (bottom-right figure).
    }
    \end{center}
    \end{subfigure}
    \begin{subfigure}[b]{0.99\textwidth}
    \begin{center}
    \centerline{\includegraphics[width=0.7\linewidth]{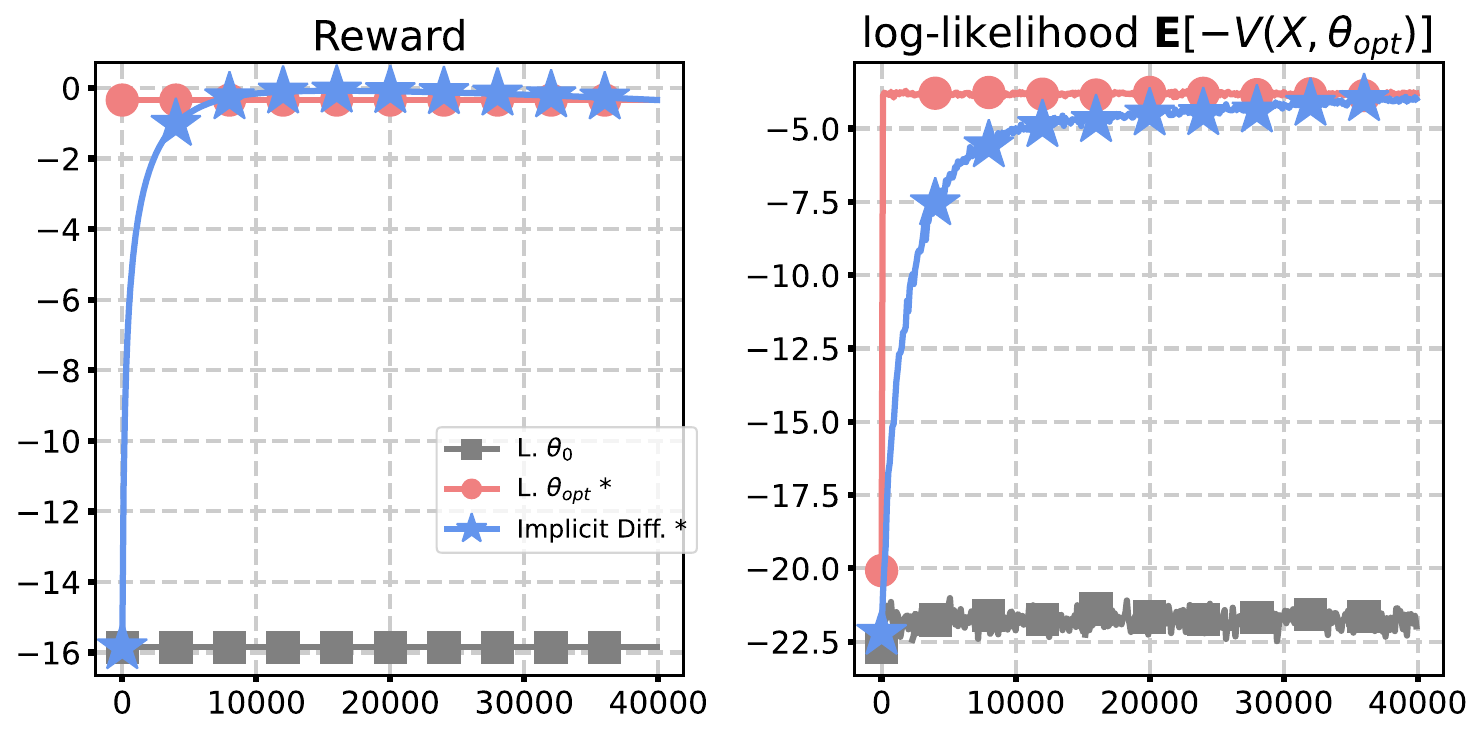}}
    \caption{Evolution of the reward and of the log-likelihood of the samples for the initial model, for the Implicit Diffusion algorithm, and for the trained model after Implicit Diffusion. The reward is the (opposite of the) KL between the target distribution and the current one, so a reward equal to zero means we learnt to reproduce the target distribution.}
    \end{center}
\end{subfigure}
\caption{Optimizing through sampling with \textbf{Implicit Diffusion} to train from scratch an energy-based model (Langevin diffusion). The potential $V$ is a logsumexp of $2$ quadratic forms, where the weights, centers, and the matrix of the forms are learnable parameters, so that the outcome distribution can be any mixture of $2$ Gaussians.}
\label{fig:training-from-scratch-full-gmm}
\end{figure}

\newpage

\subsection{Denoising diffusion models} \label{apx:exp-denoising}

We first provide additional experimental configurations that are common between both datasets before explaining details specific to each one.

\paragraph{Common details.} The KL term in the reward is computed using Girsanov's theorem as explained in Appendix~\ref{apx:adjoint-method}. We use the Adam optimizer \citep{kingmaAdamMethodStochastic2017}, with various values for the learning rate (see e.g.~Figure~\ref{fig:metrics-mnist-cifar}). The code was implemented in JAX \citep{jax2018github}. As mentioned in the main text, we use a U-Net model \citep{ronneberger2015u}.

\paragraph{MNIST.} We use an Ornstein-Uhlenbeck noise schedule, meaning that the forward diffusion is $\ud X_t = -X_t \ud t + \sqrt{2} \ud B_t$ (as presented in Section \ref{subsec:examples}). We pretrain for $18\si{\kilo\nothing}$ steps in $7$ minutes on 4 TPUv2. For reward training, we train on a TPUv2 for $4$ hours with a queue of size $M=4$, $T=64$ steps, and a batch size of~$32$.
Further hyperparameters for pretraining and reward training are given respectively in Tables \ref{tab:hyperparams-mnist-pretraining} and \ref{tab:hyperparams-mnist-reward}.

\begin{table}[ht]
    \centering
    \begin{tabular}{cc}
    \toprule
    {\bf Name} & {\bf Value} \\
    \midrule
    Noise schedule & Ornstein-Uhlenbeck \\
    Optimizer & Adam with standard hyperparameters \\
    EMA decay & $0.995$ \\
    Learning rate & $10^{-3}$ \\
    Batch size & $32$ \\
    \bottomrule
    \end{tabular}
    \caption{Hyperparameters for pretraining of denoising diffusion models on MNIST.}
    \label{tab:hyperparams-mnist-pretraining}
\end{table}

\begin{table}[ht]
    \centering
    \begin{tabular}{cc}
    \toprule
    {\bf Name} & {\bf Value} \\
    \midrule
    Number of sampling steps & $256$ \\
    Sampler & Euler \\
    Noise schedule & Ornstein-Uhlenbeck \\
    Optimizer & Adam with standard hyperparameters \\
    \bottomrule
    \end{tabular}
    \caption{Hyperparameters for reward training of denoising diffusion models pretrained on MNIST.}
    \label{tab:hyperparams-mnist-reward}
\end{table}

\paragraph{CIFAR-10 and LSUN.} For CIFAR-10, we pretrain for $500\si{\kilo\nothing}$ steps in $30$ hours on $16$ TPUv2, reaching an FID score \citep{heusel2017gans} of~$2.5$. For reward training, we train on a TPUv3 for~$9$ hours with a queue of size $M=4$ and $T=64$ steps, and a batch size of $32$. For LSUN, we pretrain for $13$ hours, reaching a FID score of~$2.26$. For reward training, we train on a TPUv3 for~$9$ hours with a queue of size $M=2$, $T=64$ steps, and a batch size of $16$. Further hyperparameters for pretraining and reward training are given respectively in Tables \ref{tab:hyperparams-cifar-pretraining} and \ref{tab:hyperparams-cifar-reward}. 

\begin{table}[H]
    \centering
    \begin{tabular}{cc}
    \toprule
    {\bf Name} & {\bf Value} \\
    \midrule
    Number of sampling steps & $1,024$ \\
    Sampler & DDPM \\
    Noise schedule & Cosine \citep{nichol2021improved} \\
    Optimizer & Adam with $\beta_1=0.9$, $\beta_2=0.99$, $\varepsilon=10^{-12}$ \\
    EMA decay & $0.9999$ \\
    Learning rate & $2 \cdot 10^{-4}$ \\
    Batch size & $2048$ \\
    Number of samples for FID evaluation & $50\si{\kilo\nothing}$ \\
    \bottomrule
    \end{tabular}
    \caption{Hyperparameters for pretraining of denoising diffusion models on CIFAR-10 and LSUN.}
    \label{tab:hyperparams-cifar-pretraining}
\end{table}

\begin{table}[H]
    \centering
    \begin{tabular}{cc}
    \toprule
    {\bf Name} & {\bf Value} \\
    \midrule
    Number of sampling steps & $1,024$ \\
    Sampler & Euler \\
    Noise schedule & Cosine \citep{nichol2021improved} \\
    Optimizer & Adam with standard hyperparameters \\
    \bottomrule
    \end{tabular}
    \caption{Hyperparameters for reward training of denoising diffusion models pretrained on CIFAR-10 and LSUN.}
    \label{tab:hyperparams-cifar-reward}
\end{table}

\paragraph{Additional figures for MNIST.}
We report in Figure \ref{fig:metrics-mnist-pos} metrics on the rewards and KL divergence with respect to the original distribution, in the case where $\lambda > 0$. As in Figure~\ref{fig:metrics-mnist-cifar}, we observe the competition between the reward and the divergence with respect to the distribution after pretraining. 
The results are sensitive to the choice of hyperparameters. In particular, when the ratio $\lambda/\beta$ is too small, we do not observe an improvement of the reward. Note that we did not perform extensive hyperparameter finetuning for these plots, so it is likely that better results could be obtained with more hyperparameter finetuning.
We also display in Figures \ref{fig:app-mnist-pos} some additional selected examples of samples generated by our denoising diffusion model with parameters $\theta_k$, at several steps $k$ of our algorithm. Note that the random number generator system of JAX allows us, for illustration purposes, to sample for different parameters from the same seed. We take advantage of this feature to visualize the evolution of a given realization of the stochastic denoising process depending on $\theta$.

Recall that we consider
\[
\cF(p) :=- \lambda \E_{x \sim p}[R(x)] + \beta \KL(p \, || \, \pi^\star(\theta_0))\, ,
\]
where $R(x)$ is the average value of all the pixels in $x$. The figures in the main text and in the appendix present samples for negative and positive $\lambda$, rewarding respectively darker and brighter images. We emphasize that these samples are generated for evaluation purposes. To generate the samples, at various steps of the optimization procedure, we run the full denoising process for the current value of the parameters. In particular, these samples are different from the ones used to perform the joint sampling and parameter updates in Implicit Diffusion.

We have purposefully chosen, for illustration purposes, samples for experiments with the highest magnitude of $\lambda / \beta$, i.e. those that favor reward optimization over proximity to the original distribution. As noted in Section~\ref{sec:experiments}, we observe qualitatively that reward training, while shifting some aspects of the distribution (here the average brightness), and necessarily diverging from the original pretrained model, manages to do so while retaining some important global characteristics of the dataset--even though the pretraining dataset is never observed during reward training. Since we chose to display samples from experiments with the most extreme incentives towards the reward, we observe that the similarity with the pretraining dataset can be forced to break down after a certain number of reward training steps. We also observe some mode collapse; we comment further on this point below.

\begin{figure}[H]
\begin{center}
\centerline{\includegraphics[width=\textwidth]{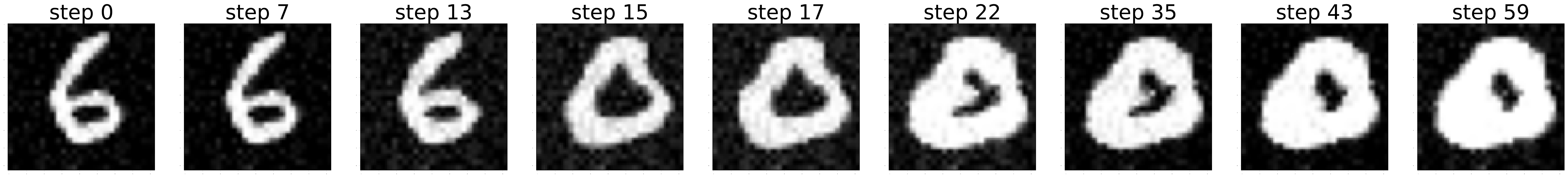}}
\centerline{\includegraphics[width=\textwidth]{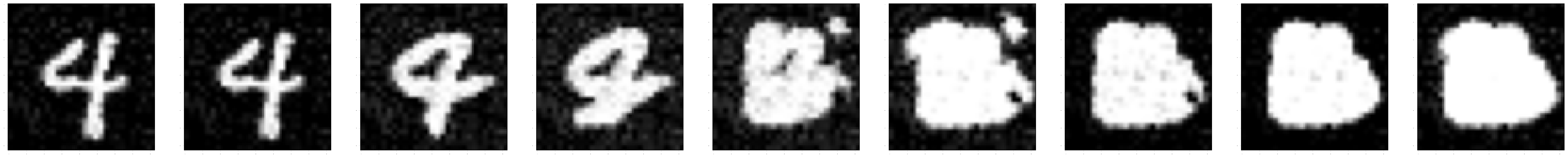}}
\centerline{\includegraphics[width=\textwidth]{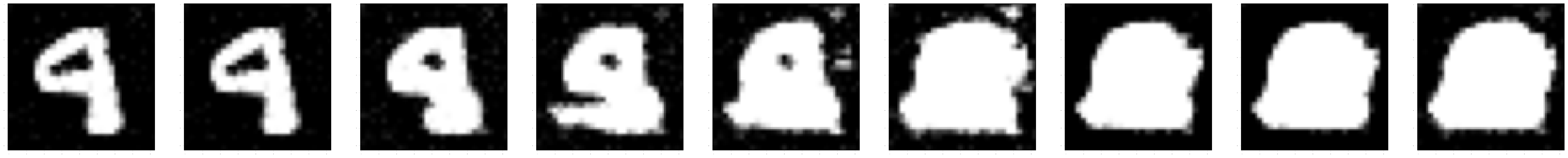}}
\caption{Reward training for a model pretrained on MNIST. The reward favors \textbf{brighter} images ($\lambda>0$, $\beta>0$). Selected examples are shown coming from a single experiment with $\lambda / \beta = 30$. All digits are re-sampled at the same selected steps of the Implicit Diffusion algorithm.
\label{fig:app-mnist-pos}}
\end{center}
\vskip -0.2in
\end{figure}

\begin{figure}[H]
\begin{center}
\centerline{\includegraphics[width=0.67\textwidth]{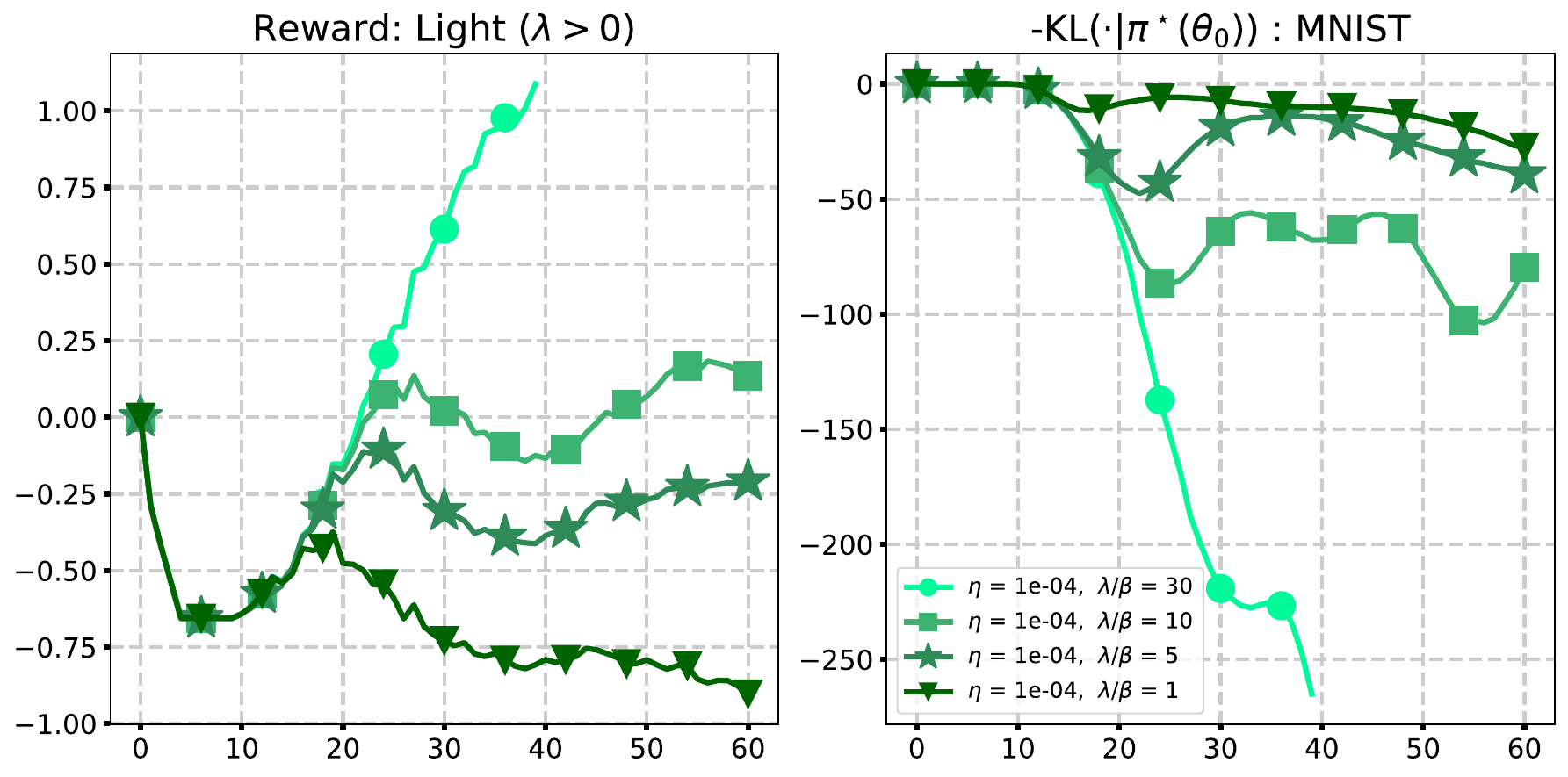}}
\caption{Score function reward training with \textbf{Implicit Diffusion} pretrained on MNIST for various $\lambda > 0$ (brighter). \textbf{Left:} Reward, average brightness of image. \textbf{Right:} Divergence w.r.t.~the original pretrained distribution. 
\label{fig:metrics-mnist-pos}}
\end{center}
\end{figure}

\paragraph{Additional figures for CIFAR-10.} We recall that we consider, for a model with weights $\theta_0$ pretrained on CIFAR-10, the objective function
\[
\cF(p) :=- \lambda \E_{x \sim p}[R(x)] + \beta \KL(p \, || \, \pi^\star(\theta_0))\, ,
\]
where $R(x)$ is the average over the red channel, minus the average of the other channels. We show in Figure~\ref{fig:app-cifar-pos} the result of the denoising process for some fixed samples and various steps of the reward training, for the experiment with the most extreme incentive towards the reward. %

We observe as for MNIST some mode collapse, although less pronounced here. Since the pretrained model has been trained with label conditioning for CIFAR-10, it is possible that this phenomenon could be a byproduct of this pretraining feature.

\begin{figure}[ht]
\vskip 0.1in
\begin{center}
\centerline{\includegraphics[width=\textwidth]{new-figures/image-mnist-appendix-car.pdf}}
\centerline{\includegraphics[width=\textwidth]{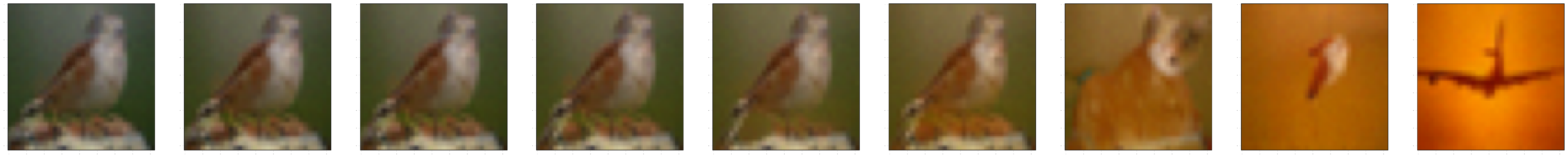}}
\centerline{\includegraphics[width=\textwidth]{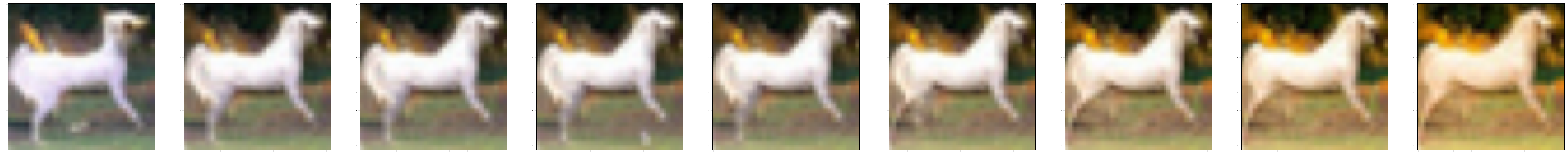}}
\caption{Reward training for a model pretrained on CIFAR-10. The reward favors \textbf{redder} images ($\lambda>0$, $\beta>0$). Selected examples are shown coming from a single experiment with $\lambda / \beta = 1,000$. All images are re-sampled at the same selected steps of the Implicit Diffusion algorithm, as explained in Appendix \ref{apx:exp-denoising}.
\label{fig:app-cifar-pos}}
\end{center}
\vskip -0.35in
\end{figure}

\paragraph{Additional figures for LSUN.} As for other datasets, we report in Figure \ref{fig:metrics-lsun} metrics on the rewards and KL divergence with respect to the original distribution.

\begin{figure}[ht]
\vskip 0.1in
\begin{center}
\centerline{\includegraphics[width=0.8\textwidth]{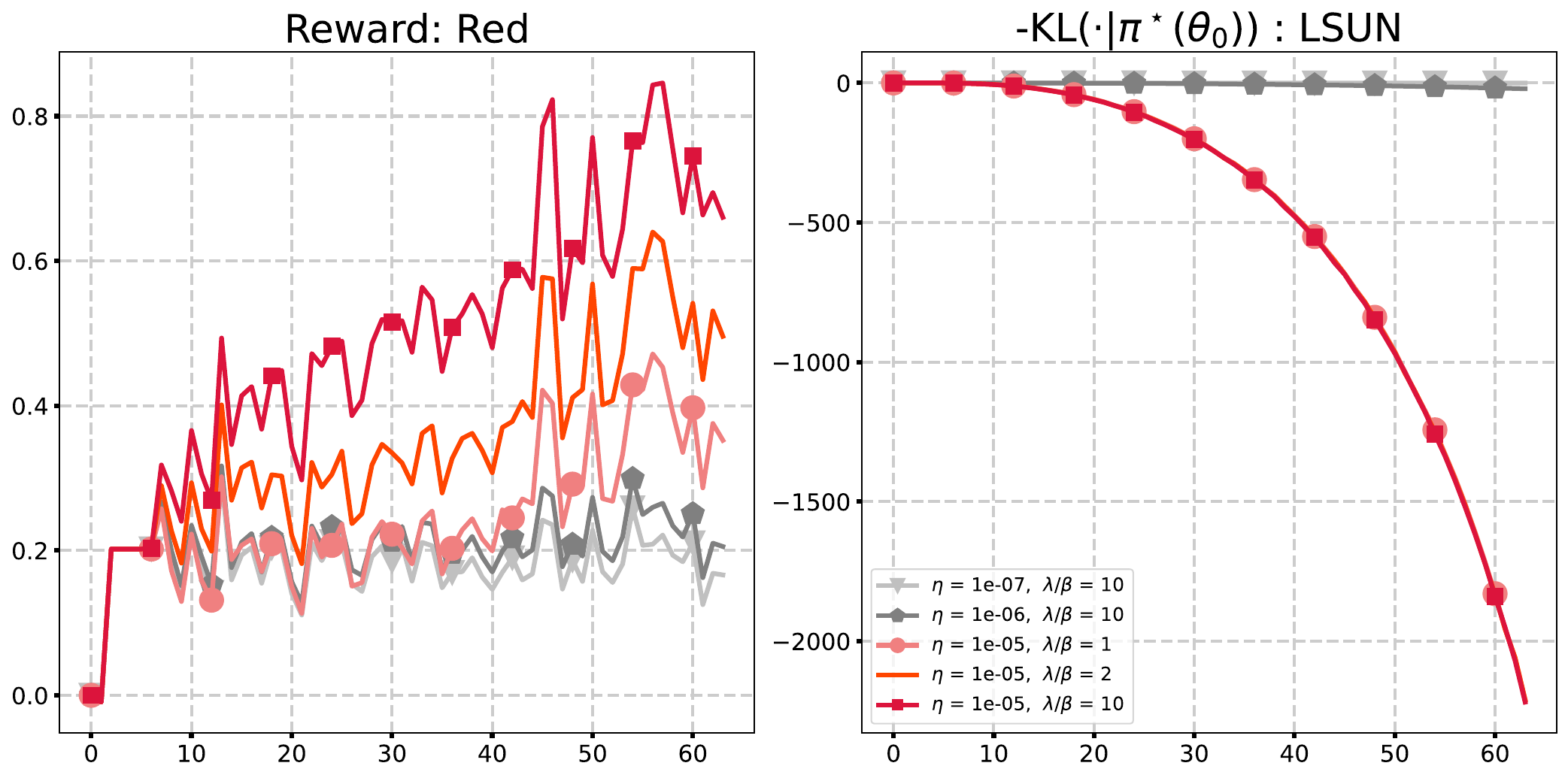}}
\caption{Score function reward training with \textbf{Implicit Diffusion} pretrained on LSUN for various $\lambda > 0$ (brighter). \textbf{Left:} Reward, average brightness of image. \textbf{Right:} Divergence w.r.t.~the original pretrained distribution. 
\label{fig:metrics-lsun}}
\end{center}
\vskip -0.35in
\end{figure}

\paragraph{Extensions.} We emphasize that our methodology covers a wide range of sampling processes and rewards. For instance, it could be applied to diffusions in discrete spaces for language modeling, or for more complex rewards such as aesthetic scores.

\section{ADDITIONAL RELATED WORK}   \label{apx:additional-related-work}

\paragraph{Reward finetuning of denoising diffusion models.}
A large body of work has recently tackled the task of finetuning denoising diffusion models, with various point of views. 
\cite{wu2023human} update weight parameters in a supervised fashion by building a high-reward dataset, then using score matching.
Other papers use reinforcement learning approaches to finetune the parameters of the model \citep{dvijotham2023algorithms,fan2023dpok,black2024training}.
Closer to our approach are works that propose finetuning of denoising diffusion models by backpropagating through sampling \citep{watson2022learning,dong2023raft,wallace2023end,clark2024directly}. However, they sample only once \citep{lee2023aligning}, or use a nested loop approach (described in Section \ref{subsec:overview}) and resort to implementation techniques such as gradient checkpointing or gradient rematerialization to limit the memory burden. We instead depart from this point of view and propose a \textbf{single-loop} approach. Furthermore, our approach is much more general than denoising diffusion models and includes any iterative sampling algorithm such as Langevin sampling. 

We emphasize that the finetuning approach differs from guidance of diffusion models (see, e.g., \citealp{dhariwal2021diffusion,graikos2022diffusion,hertz2023prompttoprompt,kwon2023diffusion,wu2023uncovering,zhang2023unsupervised}). In the latter case, the sampling scheme is modified to bias sampling towards maximizing the reward. On the contrary, finetuning directly modifies the weights of the model without changing the sampling scheme. Both approaches are complementary, and it can happen that in practice people prefer to modify the weights of the model rather than the sampling scheme: e.g., to distribute weights that take into account the reward and that can be used with any standard sampling scheme, without asking downstream users to modify their sampling method or requiring them to share the reward mechanism.

\paragraph{Single-loop approaches for bilevel optimization.} Our single-loop approach for differentiating through sampling processes is inspired by recently-proposed single-loop approaches for bilevel optimization problems \citep{guo2021novel,yang2021provably,chen2022single,dagreou2022framework,hong2023two}. Closest to our setting is \cite{dagreou2022framework}, where strong convexity assumptions are made on the inner problem while gradients for the outer problem are assumed to be Lipschitz and bounded. They also show convergence of the average of the objective gradients, akin to our Theorem \ref{thm:langevin-discrete}. However, contrarily to their analysis, we study the case where the inner problem is a sampling problem (or infinite-dimensional optimization problem). Our methodology also extends to the non-stationary case, e.g. encompassing denoising diffusion models. 

\paragraph{Study of optimization through Langevin dynamics in the linear case.}
In the case where the operator $\Gamma$ can be written as an expectation w.r.t. $p_t$ then the dynamics of $\theta$ in \eqref{eq:langevin-theory} can be seen as a McKean-Vlasov process. \cite{kuntz2023particle} and \cite{sharrock2024tuning} propose efficient algorithms to approximate this process using the convergence of interacting particle systems to McKean-Vlasov process when the number of particles is large. In the same setting, where $\Gamma$ can be written as an expectation w.r.t. $p_t$, discretization of such dynamics have been extensively studied \citep{atchade2017perturbed,debortoli2021efficient,xiao2014proximal,rosasco2020convergence,nitanda2014stochastic,tadic2017asymptotic}. In that setting,  one can leverage convergence results of the Langevin algorithm under mild assumption such as \cite{eberle2016reflection} to prove the convergence of a sequence $(\theta_k)_{k \in \mathbb{N}}$ to a local minimizer such that $\nabla \ell(\theta^\star) = 0$, see \cite[Appendix B]{debortoli2021efficient} for instance. Finally, \cite{wang2024continuous} and \cite{wang2022forward} propose and analyze a single-loop algorithm to differentiate through solutions of SDEs. Their algorithm uses forward-mode differentiation, which does not scale well to large-scale machine learning problems.

\section*{AUTHOR CONTRIBUTION STATEMENT}
PM worked on designing the methodology, implemented the codebase for experiments, proved theoretical guarantees for proposed method, contributed importantly to writing the paper.
AK contributed to designing the methodology, worked on proving theoretical guarantees, made some contributions to the paper.
PB, MB, VDB, AD, FL, CP (by alphabetical order) contributed to discussions in designing the methodology, provided references, made remarks and suggestions on the manuscript and provided some help with the codebase implementation.
QB proposed the initial idea, proposed the general methodology and worked on designing it, contributed to the codebase implementation, ran experiments, and contributed importantly to writing the paper.

\end{document}